%% file: main.tex
\documentclass{article}

\usepackage{microtype}
\usepackage{graphicx}
\usepackage{wrapfig}
\graphicspath{{figures/}}
\usepackage{caption}
\usepackage{booktabs} 
\usepackage[utf8]{inputenc}
\usepackage[hyperfootnotes=false]{hyperref}
\usepackage{xcolor}
\usepackage{tikz}
\usepackage{fancyvrb}
\usepackage{fvextra}
\usepackage{multirow}
\usepackage{pifont}
\usepackage{float}
\usepackage[skins,breakable]{tcolorbox}
\usepackage{capt-of}

\usepackage[preprint]{neurips_2026}

\usepackage{amsmath}
\usepackage{amssymb}
\usepackage{mathtools}
\usepackage{amsthm}
\usepackage{algorithm}
\usepackage{algorithmic}

\usepackage[capitalize,noabbrev]{cleveref}
\usepackage{placeins}
\usetikzlibrary{arrows.meta}

\theoremstyle{plain}
\newtheorem{theorem}{Theorem}[section]
\newtheorem{proposition}[theorem]{Proposition}
\newtheorem{lemma}[theorem]{Lemma}
\newtheorem{corollary}[theorem]{Corollary}
\theoremstyle{definition}

\theoremstyle{remark}
\newtheorem{remark}[theorem]{Remark}

\DeclareRobustCommand{\capo}{\textsc{Capo}}
\DeclareRobustCommand{\dcapo}{\textsc{DCAPO}}
\DeclareRobustCommand{\capoea}{\capo{}(EA)}
\DeclareRobustCommand{\caposgd}{\capo{}}
\newcommand{\agentgrpoloose}{Agent-GRPO (loose $\lambda$)}
\newcommand{\agentgrpostrict}{Agent-GRPO (strict $\lambda$)}
\newcommand{\E}{\mathbb{E}}

\newcommand{\sametextbf}[1]{\textbf{#1}}

\tcbset{
  promptblockstyle/.style={
    breakable,
    enhanced,
    boxrule=0pt,
    frame hidden,
    colback=white,
    arc=0pt,
    outer arc=0pt,
    boxsep=0pt,
    left=0pt, right=0pt, top=0pt, bottom=0pt,
    before skip=1.2em,
    after skip=1.2em,
    pad before break=2pt,
    pad after break=2pt,
    overlay unbroken and first={%
      \draw[line width=1.0pt] (frame.north west) -- (frame.north east);},
    overlay unbroken and last={%
      \draw[line width=1.0pt] (frame.south west) -- (frame.south east);},
  }
}

\title{\capo{}: Constraint-Aware Prompt Optimization for LLM Agents}

\author{%
  Victor Ye Dong\textsuperscript{1},
  Reid Pryzant\textsuperscript{2},
  Yi Liu\textsuperscript{1},
  Jian Jiao\textsuperscript{1}%
  \\[0.4em]
  \textsuperscript{1}Microsoft
  \qquad
  \textsuperscript{2}Independent Researcher%
  \\[0.3em]
  \texttt{victordong@microsoft.com}%
}
\date{\today}

\begin{document}
\raggedbottom
\maketitle

\begin{abstract}
Large language models (LLMs) are increasingly deployed as agents that rely on
system prompts to use tools and complete tasks. Such deployments impose distinct
operational requirements, including appropriate tool use, concise prompts and
solution paths, and compliance with safety and formatting policies. For many
practitioners, however, assembling domain-specific supervised data to post-train
models to meet these requirements is infeasible. We introduce \capo{}
(Constraint-Aware Prompt Optimization), a primal--dual method that combines
pool-based rewrites with adaptive constraint weighting to optimize system
prompts under explicit operational constraints. Across agentic benchmarks,
\capo{} more reliably reaches empirically feasible operating points while
improving task performance. \capo{} also generalizes beyond agentic settings,
achieving strong results on assistant-style evaluations with output-format and
safety/privacy constraints.
We further introduce \dcapo{} (Dynamically Trained \capo{}), which trains a
feedback- and dual-conditioned rewriter with pool-based GRPO while keeping the task agent frozen. Across task agents of
different sizes, \dcapo{} produces a feasible prompt in every evaluated domain
and matches or improves the task accuracy achieved by the evaluated baselines.
A surrogate analysis characterizes how finite-pool and discrete-rewrite errors
enter the inexact primal--dual procedure.
\end{abstract}

\section{Introduction}
\label{sec:intro}

An LLM agent can solve its assigned task and still be unsuitable for
deployment. It may invoke too many tools, escalate requests unnecessarily,
produce an overlong system prompt, or violate safety and formatting policies.
These requirements define separate operating thresholds, not one
undifferentiated notion of quality~\citep{cuiORBenchOverRefusalBenchmark2025,
Barres_Dong_Ray_Si_Narasimhan_2025}. A useful optimizer must therefore improve
task performance \emph{within} the region in which every operational budget is
satisfied.

Most automatic prompt optimizers do not let users enforce separate thresholds.
A fixed
weighted score assigns each requirement a coefficient before search, although
the binding constraint can change across models, domains, and optimization
rounds. Pareto ranking avoids choosing coefficients, but a Pareto-optimal prompt
can still violate every deployment threshold. Post-training can help the task
model internalize some requirements, but doing so requires access to model
weights and suitable training data. By contrast, prompt optimization applies
to frozen and API-only models without changing their weights. It does not
replace hard safeguards for irreversible or high-stakes actions.

We formulate system-prompt optimization as an explicit threshold-constrained
problem and solve it with feedback-driven primal--dual search. Each constraint has
a nonnegative multiplier that is updated from its signed empirical residual.
Violations increase the corresponding weight in later ranking and rewriting;
slack decreases it. A
prompt pool preserves strong candidates and allows successful rewrites to
become parents in later rounds. Figure~\ref{fig:capo_framework} illustrates how agent
behavior drives rewriting, dual updates, and pool retention in one round.

The framework has two forms. \capo{} uses a frozen LLM as its rewriter and
stores optimization state in the prompt pool and dual variables. \dcapo{}
learns a feedback- and dual-conditioned rewrite policy from \capo{} edits and
online behavior, while the deployed task agent remains frozen. Our
contributions are:
\begin{itemize}
  \item \textbf{Constraint-aware prompt optimization.} We develop a
  pool-based primal--dual procedure that uses signed residuals to adapt the
  weight of each constraint. We also connect its discrete rewrites to an
  inexact surrogate bound
  (§\ref{sec:problem}--§\ref{sec:theory}).
  \item \textbf{Consistent empirical feasibility.} A \capo{} variant reaches
  all thresholds fixed before search for all six combinations of three
  \textsc{tau2-bench} domains and two task models; every fixed-score or Pareto
  baseline does so for at most one combination. Chatbot and privacy-delegation
  studies show that the formulation extends beyond
  tool-using agents (§\ref{sec:experiments}).
  \item \textbf{A learned rewriter.} \dcapo{} trains the rewriter with the proposed algorithm pool-based GRPO. Across three agent
  domains, it produces feasible prompts, and controlled ablations show that
  online behavioral feedback supplies information that edit imitation alone
  does not (§\ref{sec:trainable_rewriter}).
\end{itemize}

\begin{figure*}[t]
\centering
\includegraphics[width=\textwidth]{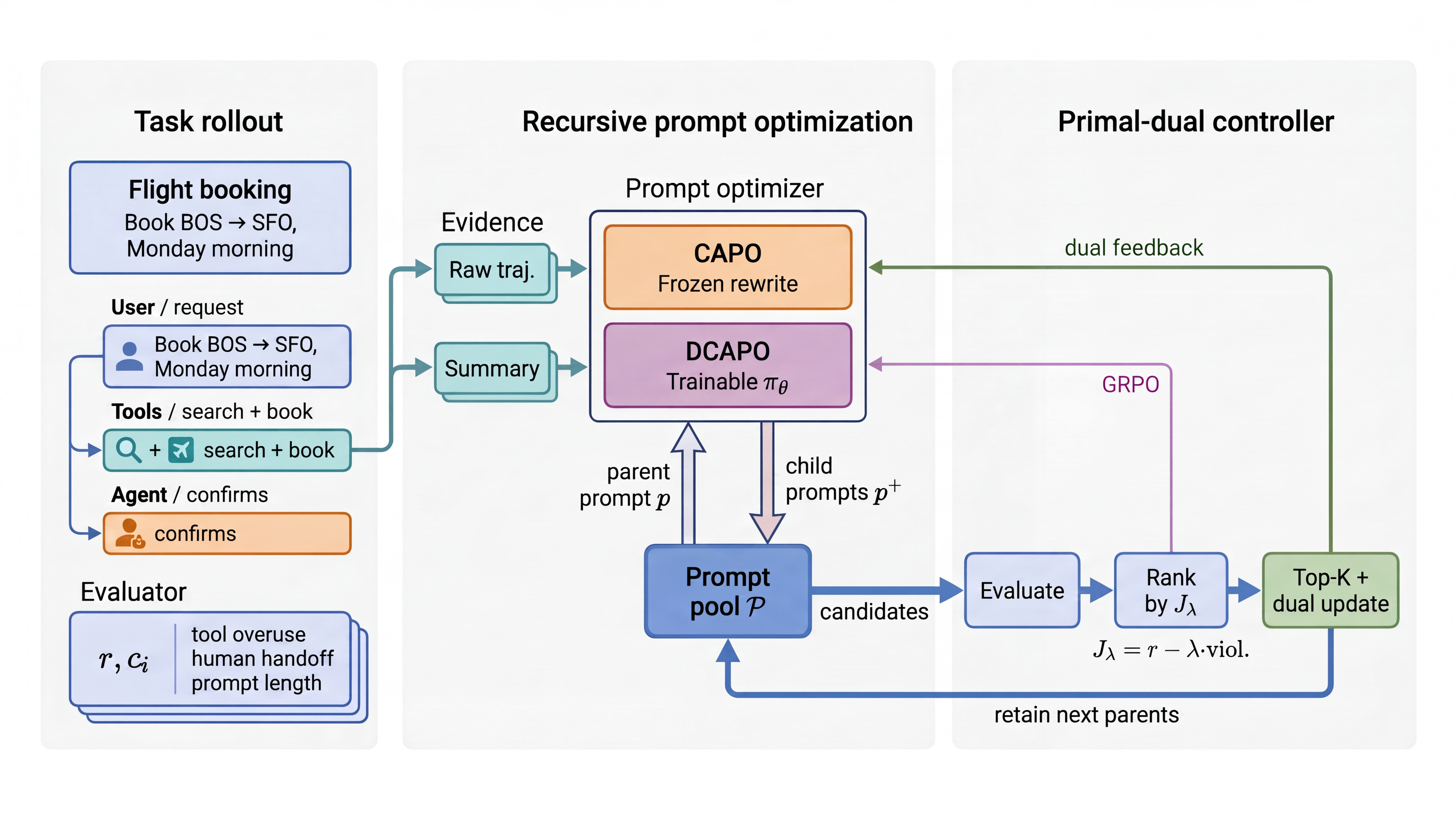}
\caption{\textbf{One round of constraint-aware prompt evolution.}
The frozen task agent executes a retained prompt, and its behavior supplies
reward, constraint costs, and rewrite evidence. \capo{} uses a frozen LLM as
the rewriter, whereas \dcapo{} uses a trainable rewriter. Both methods evaluate
children, update the
dual variables from signed threshold residuals, and retain the next prompt
pool. Only \dcapo{} updates the rewriter through GRPO.}
\label{fig:capo_framework}
\end{figure*}

\section{Related Work}
\label{sec:related}

\paragraph{Automatic prompt optimization.}
APO optimizes prompts through textual gradients and OPRO through an optimizer LLM
\citep{pryzantAutomaticPromptOptimization2023,yangOPROLargeLanguage2024}, while
InstructZero combines soft prompts with Bayesian optimization
\citep{chenInstructZeroEfficientInstruction2023}. RLPrompt learns discrete
prompts from scalar rewards \citep{dengRLPromptOptimizingDiscrete2022}, whereas
StablePrompt trains a prompt-generation policy with APPO
\citep{kwonEtAlStablePrompt2024}. GEPA, MOPO,
and EvoPrompt use reflection, Pareto selection, or evolutionary operators
\citep{agrawalGEPAReflectivePrompt2025,resendiz2025mopo,
guoConnectingLargeLanguage2023}. CAPO instead adapts one multiplier per explicit
threshold for candidate ranking and rewrite feedback. DCAPO trains a
trajectory- and dual-conditioned rewriter in an evolving pool while keeping the
task model frozen.

\paragraph{Context and skill evolution.}
ACE evolves reusable playbooks for frozen models
\citep{zhangAgenticContextEngineering2025}; EvoSkill and SkillGrad revise
structured skills from execution failures
\citep{alzubi2026evoskillautomatedskilldiscovery,
wang2026skillgradoptimizingagentskills}. INSPO, SkillRL, SAGE, and ReSkill
jointly evolve instructions or skills with an RL-trained task policy
\citep{zhou2026agenticpolicyoptimizationinstructionpolicy,
xia2026skillrlevolvingagentsrecursive,
wang2026reinforcementlearningselfimprovingagent,
he2026reskillreconcilingskillcreation}. CAPO instead optimizes one deployable
system prompt for a frozen task model under workload-level thresholds; DCAPO
trains only the rewriter using adaptive Lagrangian scores.

\paragraph{Constrained generation and learning.}
Constrained decoding enforces lexical, logical, or syntactic restrictions
during generation~\citep{hokamp-liu-2017-lexically,post-vilar-2018-fast,
lu-etal-2021-neurologic,scholak2021picardparsingincrementallyconstrained}, and
controllable generation steers output attributes~\citep{qianControllableNaturalLanguage2022,
chai2022fastimprovingcontrollabilitytext,huang2025jamcontrollableresponsibletext}.
CAPO instead optimizes a system prompt against workload metrics that can depend
on full agent trajectories. Its dual update follows constrained RL and CMDP methods
\citep{tesslerRewardConstrainedPolicy2018,dingNaturalPolicyGradient2020,
dingProvablyEfficientSafe2021,gattamiReinforcementLearningConstrained2021}:
evaluators supply costs, and signed residuals update constraint multipliers.

\section{Threshold-Constrained Prompt Optimization}
\label{sec:problem}

\subsection{Problem Formulation}

Let $p\in\mathcal{P}_{\mathrm{text}}$ be a system prompt and $\mathcal{M}$ a
frozen task model. On input $x$, the model produces
$y=\mathcal{M}(p,x)$. The function
$r_{\mathrm{task}}(p,x)\in[0,1]$ measures task quality, and
$c_i(p,x)\in\mathbb{R}$ measures the lower-is-better cost associated with
constraint $i$. Each cost has an explicit budget $\tau_i$. Real-valued costs
can represent both violation indicators and normalized margins, such as excess tool
use or prompt length; a negative margin denotes slack.

The task reward and constraints may use different workloads. We associate the
objective with $\mathcal{D}_{\mathrm{task}}$ and constraint $i$ with
$\mathcal{D}_i$. Higher-is-better compliance scores are converted to costs,
for example by $c_i:=1-s_i$. Define
$R(p)=\E_{x\sim\mathcal{D}_{\mathrm{task}}}[r_{\mathrm{task}}(p,x)]$ and
$C_i(p)=\E_{x\sim\mathcal{D}_i}[c_i(p,x)]$. The optimization problem is
\begin{equation}
\label{eq:constrained}
\textstyle\max_{p\in\mathcal{P}_{\mathrm{text}}}\ R(p)
\quad \text{s.t.}\quad C_i(p)\le\tau_i,\qquad i=1,\ldots,m.
\end{equation}
This separation accommodates metrics such as task accuracy, safety, and
over-refusal, which require different inputs. A
prompt is feasible when every population cost meets its threshold. In the
experiments, \emph{AllSat} denotes the corresponding empirical criterion:
every evaluation-set mean lies at or below its fixed threshold. AllSat is
therefore an empirical feasibility measure, not a finite-sample guarantee of
population feasibility. Section~\ref{sec:experiments} and Appendix
\ref{app:datasets} define every evaluator, workload, and threshold.

\subsection{Limits of Static Objectives}
\label{sec:baseline_limitations}

Prompt edits jointly affect task success, tool use, escalation, response
length, and safety. The active constraint can therefore change across domains
and during a single search trajectory. A fixed weighted score (scalarization)
assigns weights before observing this trajectory; retuning those weights often
replaces one violation with another. Pareto selection removes the coefficients
but not the threshold mismatch: an entire Pareto frontier can lie outside the admissible
region defined by Eq.~\eqref{eq:constrained}.

Agent-GRPO with a fixed penalty vector has the same control limitation. Updating
the policy changes the reward and every cost simultaneously, while the static
penalty cannot respond when another constraint becomes active. Figure
\ref{fig:constraint_control_2x3} illustrates both cases: static prompt search
and Agent-GRPO with fixed penalties improve task accuracy in some settings but fail to stay
inside the feasible region across domains. Both methods observe the constraint
metrics; what they lack is a feedback law that converts each signed residual
into an adaptive, constraint-specific weight.

\input{paper_constraint_table}

\section{\capo{}: Primal--Dual Search with a Frozen Rewriter}
\label{sec:method}

\capo{} solves Eq.~\eqref{eq:constrained} without changing the task model. Each
round alternates an approximate primal step---an LLM rewrites selected system
prompts to improve the current Lagrangian---and a dual step driven by measured
threshold residuals. The optimizer maintains only a prompt pool and one
multiplier per constraint; the output remains a text system prompt.

\subsection{Lagrangian Prompt Ranking}

For multipliers $\lambda_i\ge 0$, define
\begin{equation}
\label{eq:lagrangian}
\mathcal{J}(p,\boldsymbol\lambda)
=R(p)-\sum_{i=1}^{m}\lambda_i\bigl(C_i(p)-\tau_i\bigr),
\qquad \boldsymbol\lambda\in\mathbb{R}_{+}^{m}.
\end{equation}

The constrained problem becomes the saddle problem
$\min_{\boldsymbol\lambda\ge0}\max_p
\mathcal{J}(p,\boldsymbol\lambda)$. Under the regularity and strong-duality
conditions in Appendix~\ref{app:theory}, and when the multiplier cap contains a
dual optimum, a saddle point satisfies complementary slackness: an inactive
constraint has zero multiplier, while a positive multiplier identifies a tight
constraint. At round $t$, \capo{} ranks candidate prompts by
$\mathcal J(p,\boldsymbol\lambda_t)$. This ordering guides parent selection,
identifies the prompts used for the dual update, and determines which
prompts remain in the pool. As the multipliers change, the ranking shifts toward
prompts that better satisfy the currently active constraints.

\subsection{Approximate Primal Rewrites}\label{sec:textual-gradient-descent}

Following~\citet{pryzantAutomaticPromptOptimization2023}, a critic LLM
summarizes observed task failures and constraint violations. A frozen optimizer
LLM then rewrites the parent prompt using this evidence and the current dual
state. In a continuous prompt space, the corresponding ascent direction would
be

\begin{equation}\label{eq:lagrangian_gradient}
\nabla_p \mathcal{J}
=\nabla_p R(p)-\sum_{i=1}^{m}\lambda_i\nabla_p C_i(p).
\end{equation}

In practice, we approximate this ascent direction through iterative prompt
rewrites guided by textual feedback on task performance and constraint
violations. Across search rounds, the evolving prompt pool retains and builds
on successful rewrites. Multipliers affect search through two explicit channels.
First, $\boldsymbol\lambda$ enters $\mathcal{J}$ and therefore affects
candidate ranking, parent sampling, and beam retention. Second, the critic
groups failure evidence by constraint, and the rewrite instruction provides
the multiplier associated with each constraint. We represent the objective and
constraint weights as $(1,\lambda_1,\ldots,\lambda_m)$, where the objective
coefficient is fixed at $1$. The rewriter uses these weights to prioritize
constraints with larger multipliers. The frozen rewriter returns
$p^{+}=\mathcal{W}(p,\xi,\boldsymbol\lambda)$, where $\xi$ contains
representative execution trajectories, including tool calls and task failures,
together with evaluator outputs and constraint-specific critiques.
Section~\ref{sec:theory} states the assumptions under which the resulting
displacement in a surrogate prompt space constitutes an inexact ascent step.

\subsection{Prompt Pool Search}
Let $\mathcal{P}_t$ be the retained pool at round $t$. After scoring its prompts,
we sample $k$ parents without replacement with
$q_j\propto\exp(\gamma\widehat{\mathcal{J}}_{t,j})$, optionally after masking
all but the highest-scoring candidates. Rewrites expand the pool; the top
$k_1$ candidates under the current score survive. This selection--rewrite--
retention cycle preserves strong parents when an edit fails and allows useful
children to seed later rounds.

In the analysis, we model this search as an inexact primal oracle. Let
$d(\boldsymbol\lambda)=\max_p\mathcal{J}(p,\boldsymbol\lambda)$ and
\[
\delta_t=d(\boldsymbol\lambda_t)-
\max_{p\in\mathcal{P}_t}\mathcal{J}(p,\boldsymbol\lambda_t)\ge 0
\]
be the finite-pool gap. Section~\ref{sec:theory} shows how a discrete-rewrite
residual enters the pool-gap recursion and, through the resulting primal-oracle
gap, the outer primal--dual bound.

\subsection{Residual-Driven Dual Updates}\label{sec:dual-update}
After evaluating the expanded pool, let $\mathcal{P}_{\mathrm{best}}$ be the
top-$k_0$ candidates under the current Lagrangian. We update each multiplier
using the empirical costs of these candidates:
\begin{equation}
\label{eq:lambda_update}
\lambda_i \leftarrow \Pi_{[0,\lambda_{\max}]}\!\left[
\lambda_i + \beta\,\Bigl(\frac{1}{|\mathcal{P}_{\mathrm{best}}|}
\sum_{p_j \in \mathcal{P}_{\mathrm{best}}} \hat{\rho}_{i,j} - \tau_i\Bigr)
\right].
\end{equation}
Here $\hat{\rho}_{i,j}$ is the empirical cost of candidate $p_j$ on the batch
for constraint $i$, and $\beta$ is the dual learning rate. For binary $c_i$,
$\hat{\rho}_{i,j}$ is a violation rate; for continuous $c_i$, it is a
normalized cost or margin. Concretely,
\begin{equation}
\hat{\rho}_{i,j}=\frac{1}{|\mathcal{B}_{i,t}|}
\sum_{x\in\mathcal{B}_{i,t}}c_i(p_j,x),
\qquad \mathcal{B}_{i,t}\sim\mathcal{D}_i.
\end{equation}

This projected subgradient step increases $\lambda_i$ when the selected
candidates exceed $\tau_i$ and decreases it when they have slack
\citep{tesslerRewardConstrainedPolicy2018,dingNaturalPolicyGradient2020,
gattamiReinforcementLearningConstrained2021}. Persistent or large violations
therefore receive greater weight in later ranking and rewrite feedback.
Averaging over $\mathcal{P}_{\mathrm{best}}$ reduces sensitivity to a single
noisy candidate. Algorithm~\ref{alg:capo} gives the complete outer loop.

\section{\dcapo{}: Learning the Rewriter}
\label{sec:trainable_rewriter}

\subsection{Feedback-Aware Rewriter}

\capo{} queries a frozen optimizer LLM for every edit. Its search improves the
prompt pool, but the rewriter itself does not learn from successful edits.
\dcapo{} instead learns the depth-indexed conditional policy
\begin{equation}
\label{eq:dcapo_policy}
p^{(d)}\sim\pi_\theta\!\left(
\,\cdot\mid p^{(0)},p^{(d-1)},\xi^{(d-1)},
\boldsymbol{\lambda}_t\right),\qquad d=1,\ldots,D,
\end{equation}
where $p^{(0)}$ is the original parent prompt, $p^{(d-1)}$ is the preceding
prompt in the depth sequence, and $\xi^{(d-1)}$ contains its behavior under the
frozen task agent together with optional critic summaries. The current dual
state $\boldsymbol\lambda_t$ identifies the constraints that need greater
weight. \dcapo{} updates only $\pi_\theta$; it never applies GRPO to the
deployed task policy.

Training proceeds in two stages. First, we initialize the rewriter from a base
model and fine-tune it on successful \capo{} parent--child pairs. This
supervised stage teaches the rewriter to preserve the required system-prompt
structure while producing plausible edits. All domain-specific online runs
start from the same SFT checkpoint. We then train the rewriter with online
pool-based GRPO, which
extends GRPO \citep{shao2024deepseekmathpushinglimitsmathematical} by combining
group-relative policy updates with \capo{}'s
selection--rewrite--retention loop. In each round, the method samples parents
from the evolving prompt pool, evaluates groups of rewritten children using the
frozen target agent, updates the rewriter from their adaptive Lagrangian scores,
and admits high-scoring children back into the pool.

\subsection{Pool-Based GRPO}

At round $t$, let $\mathcal{P}_t$ be the prompt pool and define the empirical
score
\[
\widehat{\mathcal{J}}_t(p)=
\hat r(p)-\boldsymbol{\lambda}_t^\top
(\hat{\boldsymbol\rho}(p)-\boldsymbol\tau).
\]
Let $B$, $G$, and $D$ denote the parent-batch size, sibling count per parent
and depth, and rewrite depth; let $k_0$ and $k_1$ denote the dual-update and
retained-pool sizes. We index parents by $b=1,\ldots,B$, siblings by
$g=1,\ldots,G$, and depths by $d=1,\ldots,D$. We sample the indexed parent batch
$\mathcal B_t=\{p_t^{(b)}\}_{b=1}^{B}$ off-policy from $\mathcal P_t$. For
each parent--sibling pair $(b,g)$, let $p_{b,g}^{(0)}=p_t^{(b)}$ and construct
$\xi_{b,g}^{(0)}$ from a task-agent rollout. In the parent-anchored recurrence,
each depth rewrites the original parent using behavioral evidence from the
preceding depth:
\begin{equation}
\label{eq:dcapo_rewrite}
p_{b,g}^{(d)}\sim\pi_\theta(\cdot\mid
p_t^{(b)},p_{b,g}^{(d-1)},\xi_{b,g}^{(d-1)},
\boldsymbol\lambda_t),\qquad d=1,\ldots,D.
\end{equation}
The frozen task agent evaluates each child prompt. For each parent $p_t^{(b)}$,
we normalize its $GD$ child scores to obtain $A_{t,b,g,d}$ and update
$\pi_\theta$ with the clipped GRPO objective. We admit the top-$k_1$ children
under $\widehat{\mathcal J}_t$ to $\mathcal C_t$ and set
$\widetilde{\mathcal P}_t=\mathcal P_t\cup\mathcal C_t$. As in
\capo{}, the top-$k_0$ prompts define $\mathcal P_{\mathrm{best}}$ for
Eq.~\eqref{eq:lambda_update}, and the top-$k_1$ prompts form
$\mathcal P_{t+1}$. Setting $D=1$ yields one behavior-conditioned rewrite; with
$D=2$, the second rewrite conditions on the first child and its task-agent
trajectory.

Algorithm~\ref{alg:dcapo} separates the roles of the components. The frozen
task model generates behavior, the rewriter receives the policy update, and
the prompt pool and dual variables store the constrained-search state.

\par\medskip
\noindent
\begin{minipage}[t]{0.49\linewidth}
\vspace{0pt}
\begin{algorithm}[H]
\caption{\capo{}: frozen rewriter}
\label{alg:capo}
\scriptsize
\setlength{\algorithmicindent}{0.8em}
\begin{algorithmic}[1]
\REQUIRE $p_0,\mathcal M,\mathcal W,E,\boldsymbol\tau,T,k,\gamma,
k_0,k_1,\beta$
\STATE $\mathcal P_0\gets\{p_0\}$;
$\boldsymbol\lambda_0\gets\boldsymbol 1$
\FOR{$t=0,\ldots,T-1$}
  \STATE Evaluate $\mathcal P_t$; compute
  $\widehat{\mathcal J}_t$
  \STATE Sample $k$ parents with
  $q_j\propto e^{\gamma\widehat{\mathcal J}_{t,j}}$
  \STATE $\widetilde{\mathcal P}_t\gets\mathcal P_t$
  \FOR{each sampled $p$}
    \STATE $\xi\gets\textsc{Feedback}(p,E,\boldsymbol\lambda_t)$
    \STATE Evaluate $p^+\gets
    \mathcal W(p,\xi,\boldsymbol\lambda_t)$; add to
    $\widetilde{\mathcal P}_t$
  \ENDFOR
  \STATE $\mathcal P_{\mathrm{best}}\gets
  \operatorname{Top}_{k_0}(\widetilde{\mathcal P}_t,
  \widehat{\mathcal J}_t)$
  \STATE Update $\boldsymbol\lambda_{t+1}$ by Eq.~\eqref{eq:lambda_update}
  \STATE $\mathcal P_{t+1}\gets
  \operatorname{Top}_{k_1}(\widetilde{\mathcal P}_t,
  \widehat{\mathcal J}_t)$
\ENDFOR
\RETURN $\arg\max_{p\in\mathcal P_T}
\widehat{\mathcal J}(p,\boldsymbol\lambda_T)$
\end{algorithmic}
\end{algorithm}
\end{minipage}\hfill
\begin{minipage}[t]{0.49\linewidth}
\vspace{0pt}
\begin{algorithm}[H]
\caption{\dcapo{}: learned rewriter}
\label{alg:dcapo}
\scriptsize
\setlength{\algorithmicindent}{0.8em}
\begin{algorithmic}[1]
\REQUIRE $\mathcal D,p_0,\mathcal M,E,\boldsymbol\tau,T,B,G,D,k_0,k_1,
\beta$
\STATE SFT-initialize $\pi_\theta$ on $\mathcal D$;
$\mathcal P_0\gets\{p_0\}$;
$\boldsymbol\lambda_0\gets\boldsymbol 1$
\FOR{$t=0,\ldots,T-1$}
  \STATE Sample $\mathcal B_t=\{p_t^{(b)}\}_{b=1}^{B}
  \subseteq\mathcal P_t$
  \FOR{$b=1,\ldots,B$ and $g=1,\ldots,G$}
    \STATE Evaluate $p_t^{(b)}$; construct $\xi_{b,g}^{(0)}$
    \FOR{$d=1,\ldots,D$}
      \STATE Sample child prompt $p_{b,g}^{(d)}$ using
      Eq.~\eqref{eq:dcapo_rewrite}
      \STATE Evaluate $p_{b,g}^{(d)}$; construct
      $\xi_{b,g}^{(d)}$
    \ENDFOR
  \ENDFOR
  \STATE Compute $A_{t,b,g,d}$ and update
  $\pi_\theta$
  \STATE $\mathcal C_t\gets
  \operatorname{Top}_{k}(\{p_{b,g}^{(d)}\}_{b,g,d},
  \widehat{\mathcal J}_t)$
  \STATE $\widetilde{\mathcal P}_t\gets
  \mathcal P_t\cup\mathcal C_t$,
  $\mathcal P_{\mathrm{best}}\gets
  \operatorname{Top}_{k_0}(\widetilde{\mathcal P}_t,
  \widehat{\mathcal J}_t)$
  \STATE Update $\boldsymbol\lambda_{t+1}$ by
  Eq.~\eqref{eq:lambda_update}, $\mathcal P_{t+1}\gets
  \operatorname{Top}_{k_1}(\widetilde{\mathcal P}_t,
  \widehat{\mathcal J}_t)$
\ENDFOR
\RETURN $\pi_\theta$ and
$\arg\max_{p\in\mathcal P_T}
\widehat{\mathcal J}(p,\boldsymbol\lambda_T)$
\end{algorithmic}
\end{algorithm}
\end{minipage}
\par\medskip

\section{Surrogate Analysis of Discrete Rewrites}
\label{sec:theory}

Both algorithms follow the same update structure: approximately maximize the
current Lagrangian over prompts, then take a projected dual step. We conduct
the analysis in a continuous surrogate space for text rewrites and use standard
inexact primal--dual analysis to quantify the approximate primal response
\citep{nedicOzdaglarApproximatePrimal2009}. Let
$\phi:\mathcal{P}_{\mathrm{text}}\to\Theta\subset\mathbb{R}^d$ be an analytical
surrogate map, analogous to continuous prompt relaxations
\citep{jangCategoricalReparameterizationGumbel2017,
maddisonConcreteDistributionContinuous2017,
liPrefixTuningOptimizingContinuous2021,liuGPTUnderstandsToo2021}.
A rewrite from text prompt $s_t$ to $s_t^+$ induces
$d_t=\phi(s_t^+)-\phi(s_t)$. Appendix
\ref{app:theory_bridge_discrete} assumes that this displacement is
gradient-related in expectation and has bounded second moment; smoothness and
strong concavity then yield an inexact ascent residual
$\eta_t=\epsilon_t+\tfrac{L}{2}\sigma_d^2$. Let $\Delta_t$ be the best-prompt
gap in the pool, $p_t$ the probability of rewriting that prompt, and
$\varepsilon_t^{\mathrm{pr}}$ the selected prompt's Lagrangian gap.

\begin{theorem}[Rewrite-oracle reduction]
\label{thm:rewrite_oracle_reduction}
Under Assumptions A1--A2, A4--A7, and B1--B2 in Appendix~\ref{app:theory}, let
$\rho=2\mu(\kappa-L\nu/2)\in(0,1]$. The rewrite-and-retain loop satisfies
\[
\E[\varepsilon_t^{\mathrm{pr}}]
\le \E[(1-\rho p_t)\Delta_t+p_t\eta_t+\pi_tB_{\max}]+\zeta_t,
\]
where $\pi_t$ and $\zeta_t$ measure pruning and noisy extraction. Thus rewrite
error enters the pool recursion additively. For
$S_T=\sum_{t=1}^{T}\beta_t$, the primal--dual bound with projected dual updates is
\[
\E[\mathcal{G}(\bar\theta_T,\bar\lambda_T)] \;\le\; \frac{D_\Lambda^2}{2 S_T} \;+\; \frac{G_g^2 \sum_{t=1}^{T}\beta_t^2}{2 S_T} \;+\; \frac{\sum_{t=1}^{T}\beta_t\,\E[\varepsilon_t^{\mathrm{pr}}]}{S_T}.
\]
Here $G_g$ bounds the constraint-vector norm and $D_\Lambda$ is the diameter of
the projected dual set; Appendix~\ref{app:theory_setup} gives both definitions.
If the two stepsize terms vanish and $\E[\varepsilon_t^{\mathrm{pr}}]\to0$, then
$\E[\mathcal{G}(\bar\theta_T,\bar\lambda_T)]\to0$.
\end{theorem}

The bound separates ordinary dual optimization, finite-pool error, and mismatch
between a text rewrite and ideal ascent. It characterizes the surrogate
population problem under the stated conditions. The empirical AllSat criterion
separately records whether the returned text prompt satisfies every threshold in
Section~\ref{sec:experiments}. Appendix~\ref{app:rewrite_alignment} reports
hidden-state alignment, parameter-update alignment, and parent--child
$\Delta J$ for one trained rewriter. All three measurements are positive, as
expected under B1. Appendix~\ref{app:theory} gives the assumptions and proofs.

\section{Experiments}
\label{sec:experiments}

The evaluation proceeds in three stages. We first test whether adaptive
constraint weights reach feasible operating points more consistently than
static scoring. We then evaluate transfer and sensitivity to evaluator noise
to establish the scope of this result. Finally, we test whether DCAPO can learn
effective rewrite behavior and identify the training signal responsible for
it.

\subsection{Experimental Setup}
\label{sec:exp_settings}
\label{sec:exp_baselines}

\paragraph{Benchmarks and metrics.}
The primary evaluation uses Airline, Retail, and Telecom from
\textsc{tau2-bench}~\citep{Barres_Dong_Ray_Si_Narasimhan_2025}. A tool-using
agent must complete a customer-service task under a domain policy. Task
completion is the objective; human-agent requests (HAR), excess tool use
(ToolEx), and system-prompt length (PLen) are separate costs. Tables report Acc, HAR, and ToolEx as percentages and PLen in thousands of system-prompt characters. (Appendix~\ref{app:plen_reporting}). Acc is higher-is-better;
HAR, ToolEx, and PLen are lower-is-better, and negative ToolEx denotes fewer
calls than the reference trajectory. AllSat indicates that every reported cost
meets its threshold. We evaluate GPT-5-mini and GPT-5.1 model versions
\citep{openaiGPT5MiniModel2025,openaiGPT51Model2025}, and report an open-weight
Ministral-8B replication in the appendix.

Two assistant-style studies broaden the types of constraints. The chatbot
suite maximizes GSM8K accuracy~\citep{cobbeGSM8K2021} subject to response-length,
AdvBench safety, benign over-refusal, and character-counting budgets
\citep{zouUniversalTransferableAdversarial2023,
cuiORBenchOverRefusalBenchmark2025}. PUPA--IFBench measures task quality in
privacy-preserving delegation while constraining instruction-following errors
\citep{liPAPILLONPrivacyPreservation2025,
pyatkinGeneralizingVerifiableInstruction2025}. Appendix~\ref{app:datasets}
specifies the evaluator, direction, scale, and sample count for every metric.

\paragraph{Learned-rewriter setting.}
\dcapo{} is evaluated on the same three \textsc{tau2-bench} domains. The main
experiments use a Qwen3-8B rewriter with frozen Qwen3-8B and Qwen3-32B task
agents \citep{qwen3technicalreport}. The editor-size ablation holds the task
agent fixed at Qwen3-32B and compares Qwen3-0.6B, Qwen3-4B, and Qwen3-8B
rewriters. Ministral-3-8B-Instruct \citep{liu2026ministral} serves as the user
simulator. For comparison, Agent-GRPO updates the task agent directly using the
fixed-$\lambda$ signed-residual reward defined in
Appendix~\ref{app:agent_grpo_reward}; we report loose and strict penalty
settings.

\paragraph{Comparison setup.}
Thresholds and evaluation workloads are fixed before prompt search within each
reported setting. Every method receives the objective and each constraint
measurement. Multi-task APO~\citep{pryzantAutomaticPromptOptimization2023} and
GEPA~\citep{agrawalGEPAReflectivePrompt2025} use an equal-weight fixed score;
MOPO~\citep{resendiz2025mopo} uses its NSGA-II Pareto ranking. Unlike these
baselines, CAPO updates a separate multiplier for each constraint from the
measured residual. \capoea{} keeps every retained candidate eligible as a
parent, whereas \capo{} uses Lagrangian-weighted parent sampling.

Following constrained-optimization practice
\citep{tesslerRewardConstrainedPolicy2018,
achiamConstrainedPolicyOptimization2017}, the primary comparison is task
performance inside the feasible set. For comparison, we report their per-round workloads explicitly.
Appendices
\ref{app:opt_protocol}--\ref{app:complexity} specify splits, thresholds,
hyperparameters, baseline objectives, and compute accounting.

\subsection{\capo{} Evaluation}

\paragraph{Agentic tool use.}
Across the three domains and two task models reported in
Table~\ref{tab:agent_results}, \capo{} satisfies every
constraint in each setting. In contrast, the initial prompt, APO, and GEPA
each achieve feasibility in only one of the six settings, whereas MOPO does
not achieve feasibility in any. The best feasible \capo{} prompt also exceeds
the initial prompt's accuracy in five settings and matches it in the remaining
one. The binding constraint varies with the domain and task model, indicating
that greater model capability alone does not eliminate the need for adaptive
constraint weights.

\begin{table}[H]
\caption{\textbf{CAPO reaches a feasible operating point in all six evaluated
settings.} Bold marks the highest feasible accuracy, and underlining marks the
next distinct feasible accuracy; ties are included.
Per-example standard errors appear in Appendix
Tables~\ref{tab:agent_se_gpt5_airline}--\ref{tab:agent_se_gpt5_telecom}.}
\label{tab:agent_results}
\centering
\scriptsize
\setlength{\tabcolsep}{2.8pt}
\renewcommand{\arraystretch}{1.0}
\begin{tabular}{@{}l rrrrc rrrrc@{}}
\toprule
& \multicolumn{5}{c}{\sametextbf{GPT-5-mini}} & \multicolumn{5}{c}{\sametextbf{GPT-5.1}} \\
\cmidrule(lr){2-6} \cmidrule(lr){7-11}
\sametextbf{Method} & \sametextbf{Acc\,$\uparrow$} & \sametextbf{HAR\,$\downarrow$} & \sametextbf{ToolEx\,$\downarrow$} & \sametextbf{PLen\,$\downarrow$} & \sametextbf{AllSat} & \sametextbf{Acc\,$\uparrow$} & \sametextbf{HAR\,$\downarrow$} & \sametextbf{ToolEx\,$\downarrow$} & \sametextbf{PLen\,$\downarrow$} & \sametextbf{AllSat} \\
\midrule
\multicolumn{11}{@{}l}{\sametextbf{Airline}} \\
Threshold  & -- & 35 & 105 & 5.00 & -- & -- & 80 & 80 & 5.00 & -- \\
Initial    & 25.0 & 35.0 & 105.0 & 0.30 & \checkmark  & 55.0 & 95.0 & 91.8 & 0.30 & \ding{55} \\
APO        & 25.0 & 10.0 & 141.0 & 2.20 & \ding{55}  & 45.0 & 20.0 & 113.1 & 3.10 & \ding{55} \\
GEPA       & 30.0 & 10.0 & 92.0 & 4.20 & \checkmark    & 65.0 & 30.0 & 152.3 & 4.72 & \ding{55} \\
MOPO       & 45.0 & 25.0 & 156.0 & 4.68 & \ding{55}    & 60.0 & 35.0 & 139.3 & 5.52 & \ding{55} \\
\capoea    & \sametextbf{50.0} & 20.0 & 100 & 4.60 & \checkmark & 50.0 & 65.0 & 123.0 & 3.27 & \ding{55} \\
\caposgd   & \underline{45.0} & 5.0 & 104.9 & 4.98 & \checkmark   & \sametextbf{45.0} & 55.0 & 77.1 & 2.96 & \checkmark \\
\midrule
\multicolumn{11}{@{}l}{\sametextbf{Retail}} \\
Threshold  & -- & 15 & 50 & 5.00 & -- & -- & 5 & 50 & 5.00 & -- \\
Initial    & 50.0 & 12.5 & 52.7 & 0.30 & \ding{55}   & 62.5 & 0 & 50.5 & 0.30 & \ding{55} \\
APO        & 57.5 & 7.5 & 57.5 & 0.30 & \ding{55}    & 62.5 & 0 & 50.5 & 0.30 & \ding{55} \\
GEPA       & 50.0 & 5.0 & 51.6 & 4.52 & \ding{55}     & 45.0 & 2.5 & 59.2 & 4.05 & \ding{55} \\
MOPO       & 55.0 & 7.5 & 55.9 & 5.87 & \ding{55}     & 75.0 & 2.5 & 53.8 & 5.52 & \ding{55} \\
\capoea    & \underline{50.0} & 12.5 & 49.5 & 4.51 & \checkmark   & \sametextbf{67.5} & 2.5 & 41.4 & 4.91 & \checkmark \\
\caposgd   & \sametextbf{55.0} & 12.5 & 44.1 & 4.70 & \checkmark & \underline{62.5} & 2.5 & 41.4 & 4.45 & \checkmark \\
\midrule
\multicolumn{11}{@{}l}{\sametextbf{Telecom}} \\
Threshold  & -- & 65 & 250 & 5.00 & -- & -- & 65 & 80 & 5.00 & -- \\
Initial    & 40.0 & 80.0 & 309.9 & 0.30 & \ding{55}   & 60.0 & 65.0 & 82.7 & 0.30 & \ding{55} \\
APO        & \underline{37.5} & 62.5 & -5.6 & 4.90 & \checkmark   & 55.0 & 70.0 & 56.8 & 0.30 & \ding{55} \\
GEPA       & 32.5 & 80.0 & 403.9 & 0.30 & \ding{55}    & 52.5 & 212.5 & 372.4 & 5.80 & \ding{55} \\
MOPO       & 45.0 & 65.0 & 73.5 & 5.02 & \ding{55}  & 67.5 & 40.0 & 41.4 & 5.74 & \ding{55} \\
\capoea    & 47.5 & 70.0 & -48.1 & 4.74 & \ding{55}  & 67.5 & 45.0 & -2.5 & 5.35 & \ding{55} \\
\caposgd   & \sametextbf{50.0} & 62.5 & 34.0 & 4.29 & \checkmark   & \sametextbf{75.0} & 45.0 & 35.2 & 4.23 & \checkmark \\
\bottomrule
\end{tabular}
\end{table}

\paragraph{Open-weight target.}
The same pattern holds for Ministral-8B~\citep{liu2026ministral}: absolute
accuracy is lower, but \capo{} is the only evaluated method feasible in all
three domains (Appendix Table~\ref{tab:agent_results_ministral}). This result
shows that \capo{} can control constraints even when the target model follows
instructions less reliably.

\paragraph{Chatbot and privacy delegation.}
Table~\ref{tab:main_results} compares methods across the chatbot and
privacy-delegation studies.
On the chatbot suite, both CAPO variants satisfy all four budgets, and
\capoea{} nearly matches the accuracy of the infeasible initial prompt. On
PUPA--IFBench, \capoea{} attains the highest objective among feasible methods.
Together, these results show that the same threshold-constrained formulation
applies beyond tool-use costs to constraints on safety, formatting,
over-refusal, privacy, and instruction following.

\subsection{Transfer and Robustness}
\label{sec:robustness}

\paragraph{Coding agent ablation.}
A GPT-5-mini SWE-agent case study applies the same formulation to
\textsc{SWE-bench} Lite~\citep{jimenez2024swebench}. All three methods resolve
5 of 30 issues; among them, CAPO yields the lowest reported patch size,
tool-action count, and number of files touched.
Figure~\ref{fig:transfer_diagnostics} shows the resulting normalized costs;
Appendix~\ref{app:swebench} gives the setup and exact values.

\begin{table}[t]
\centering
\begin{minipage}[t]{0.54\linewidth}
\vspace{0pt}
\caption{\textbf{Feasibility beyond tool-using agents.} Bold marks the highest
feasible objective, and underlining marks the next distinct feasible objective;
ties are included.}
\label{tab:main_results}
\centering
\scriptsize
\setlength{\tabcolsep}{1.2pt}
\resizebox{\linewidth}{!}{%
\begin{tabular}{@{}lcccccc@{\hspace{0.8em}}ccc@{}}
\toprule
\multicolumn{7}{c}{\sametextbf{Chatbot}}
& \multicolumn{3}{c}{\sametextbf{PUPA--IFBench}} \\
\cmidrule(lr){1-7} \cmidrule(lr){8-10}
\sametextbf{Method} & \sametextbf{Acc\,$\uparrow$} & \sametextbf{Len\,$\downarrow$} & \sametextbf{Saf\,$\downarrow$} & \sametextbf{Chr\,$\downarrow$} & \sametextbf{Ovr\,$\downarrow$} & \sametextbf{AllSat} & \sametextbf{Obj\,$\uparrow$} & \sametextbf{IFV\,$\downarrow$} & \sametextbf{AllSat} \\
\midrule
Threshold & -- & 15 & 5 & 10 & 12 & -- & -- & 60 & -- \\
Initial   & 93.0 & 78.1 & 0.8 & 3.1 & 17.2 & \ding{55}  & 46.1 & 62.0 & \ding{55} \\
APO       & 88.3 &  1.6 & 1.6 & 2.3 & 12.5 & \ding{55}  & 77.1 & 59.5 & \checkmark \\
GEPA      & 92.2 &  5.5 & 0.8 & 5.5 & 15.6 & \ding{55}  & 83.3 & 68.0 & \ding{55} \\
MOPO      & 89.1 &  0.8 & 1.6 & 3.1 & 13.3 & \ding{55}  & 78.1 & 60.7 & \ding{55} \\
\capoea   & \sametextbf{92.2} & 0 & 3.1 & 9.4 & 10.2 & \checkmark & \sametextbf{78.9} & 55.6 & \checkmark \\
\caposgd  & \underline{89.8} &  0.8 & 0   & 1.5 &  7.8 & \checkmark & \underline{78.2} & 55.6 & \checkmark \\
\bottomrule
\end{tabular}
}
\end{minipage}\hfill
\begin{minipage}[t]{0.45\linewidth}
\vspace{0pt}
\caption{\textbf{Ablation of feedback source and online learning.}
``Model'' identifies the rewriter, ``Online'' indicates whether it is updated
online, and ``FB'' identifies the feedback source. Results use held-out Airline
tasks with GPT-5-mini as both the target agent and user simulator. Bold marks
the highest feasible accuracy, and underlining marks the next distinct feasible
accuracy; ties are included.}
\label{tab:trainable_rewriter_main}
\centering
\scriptsize
\setlength{\tabcolsep}{0.65pt}
\renewcommand{\arraystretch}{1.04}
\resizebox{\linewidth}{!}{%
\begin{tabular}{@{}lcccrrrc@{}}
\toprule
\sametextbf{Method} & \sametextbf{Model}\hspace{2pt} & \hspace{1pt}\sametextbf{Online} & \sametextbf{FB} &
\sametextbf{Acc\,$\uparrow$} & \sametextbf{HAR\,$\downarrow$} &
\sametextbf{TEx\,$\downarrow$} & \sametextbf{AllSat} \\
\midrule
Threshold & -- & -- & -- & -- & 35.0 & 105.0 & -- \\
\midrule
\capo{} & LLM & No & Sum. & \sametextbf{45.0} & 5.0 & 104.9 & \checkmark \\
\capo{}--SFT & SLM & No & Sum. & 45.0 & 20.0 & 131.1 & \ding{55} \\
\dcapo{} (Traj.) & SLM & Yes & Traj. & \underline{40.0} & 35.0 & 98.4 & \checkmark \\
\dcapo{} (Sum.) & SLM & Yes & Sum. & \sametextbf{45.0} & 25.0 & 60.7 & \checkmark \\
\bottomrule
\end{tabular}
}
\end{minipage}
\end{table}

\paragraph{Task-cluster shift.}
We evaluate CAPO on an embedding-defined Airline split that separates training
and evaluation task clusters. CAPO meets all three thresholds on the shifted
split, with a modest decrease in objective.
Table~\ref{tab:ood_split} reports the corresponding values in the appendix.

\paragraph{Noisy dual feedback.}
We add Gaussian noise to the cost estimates used for feasibility gating and the
dual update, while keeping the continuous scores used for candidate ranking and
final evaluation unperturbed. Retail remains feasible under mild noise, whereas
Airline and Telecom violate at least one threshold
(Figure~\ref{fig:robustness_noise}).

\begin{figure}[t]
\centering
\begin{minipage}[t]{0.49\linewidth}
  \centering
  \includegraphics[width=\linewidth]{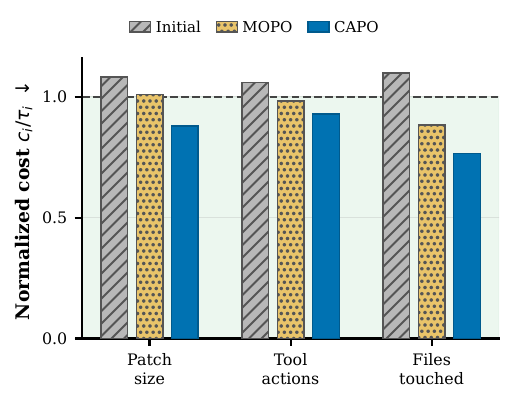}
  \caption{\textbf{Coding-agent constraint profile.}
  All three methods resolve the same 5 of 30 issues (16.7\%), so the panels
  compare only final constraint costs divided by their thresholds; values at
  or below one are feasible.}
  \label{fig:transfer_diagnostics}
\end{minipage}\hfill
\begin{minipage}[t]{0.49\linewidth}
  \centering
  \includegraphics[width=\linewidth]{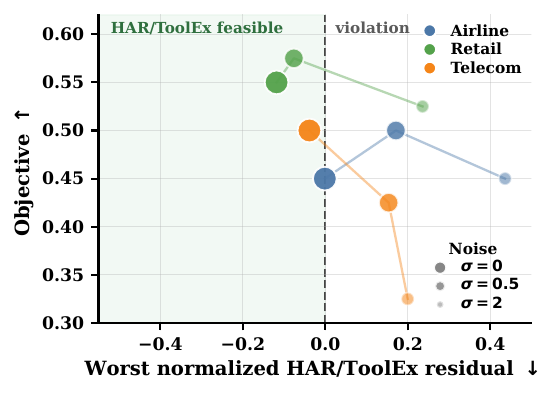}
  \caption{\textbf{Sensitivity to noisy dual feedback.} Task objective is
  plotted against the worst normalized residual across HAR and ToolEx; negative
  values satisfy both
  budgets. Marker size and opacity encode $\sigma\in\{0,0.5,2\}$. Noise affects
  feasibility gating and multiplier updates; ranking and final evaluation use
  unperturbed scores.}
  \label{fig:robustness_noise}
\end{minipage}
\end{figure}

\paragraph{Dual-rate and threshold sensitivity.}
We vary the dual learning rate and constraint thresholds. Tightening ToolEx
causes the corresponding multiplier to rise and then saturate once the search
no longer finds a feasible prompt.
Appendix~\ref{app:robustness_full} reports the sweeps, standard errors, and
task-cluster-shift values. Appendix~\ref{sec:ablations}
separately isolates the feedback and search components.

\subsection{\dcapo{} Evaluation}
\label{sec:dcapo_results}

Table~\ref{tab:dcapo_matched_main} compares adaptive prompt rewriting with
Agent-GRPO under fixed penalties at two model sizes. Both rewrite-depth settings satisfy every constraint for every domain and task-agent
size, whereas Agent-GRPO is not consistently feasible. On Qwen3-32B, where
all optimized methods are feasible, DCAPO attains the highest accuracy in each
domain.

\paragraph{Comparison with RL-based prompt learning.}
Figure~\ref{fig:rl_prompt_comparison_8b} compares \dcapo{} with StablePrompt
and direct task-agent training on Qwen3-8B. \dcapo{} is feasible in all three
domains; StablePrompt is feasible only on Retail and \agentgrpoloose{} only on
Airline. Full measurements appear in
Appendix~\ref{app:learned_editor_controls}.

\paragraph{Ablation on rewrite depth.}
\label{sec:dcapo_depth2_results}
At both task-agent sizes, a second rewrite preserves feasibility and usually
matches or improves accuracy. Frozen-rewriter transfer reduces but does not
eliminate target-domain violations
(Appendix~\ref{app:zero_shot_rewriter_transfer}).

\begin{table*}[t]
\caption{\textbf{Adaptive rewriting produces feasible prompts across domains
and task-agent sizes.} For Agent-GRPO, loose and strict $\lambda$ denote the fixed
penalty vectors $(\lambda_{\mathrm{ToolEx}},\lambda_{\mathrm{HAR}})=(1,2.5)$
and $(2,4)$, respectively; Appendix~\ref{app:agent_grpo_reward} gives the
corresponding reward. Within each domain and task-model size, bold marks the
highest feasible accuracy, and underlining marks the next distinct feasible
accuracy; ties are included.}
\label{tab:dcapo_matched_main}
\centering
\scriptsize
\setlength{\tabcolsep}{2.25pt}
\renewcommand{\arraystretch}{0.98}
\begin{tabular}{@{}llrrrrcrrrrc@{}}
\toprule
& & \multicolumn{5}{c}{\sametextbf{Qwen3-8B}}
& \multicolumn{5}{c}{\sametextbf{Qwen3-32B}} \\
\cmidrule(lr){3-7}\cmidrule(lr){8-12}
\sametextbf{Domain} & \sametextbf{Method}
& \sametextbf{Acc.\,$\uparrow$} & \sametextbf{HAR\,$\downarrow$} & \sametextbf{ToolEx\,$\downarrow$} & \sametextbf{PLen\,$\downarrow$} & \sametextbf{AllSat}
& \sametextbf{Acc.\,$\uparrow$} & \sametextbf{HAR\,$\downarrow$} & \sametextbf{ToolEx\,$\downarrow$} & \sametextbf{PLen\,$\downarrow$} & \sametextbf{AllSat} \\
\midrule
\multirow{6}{*}{Airline}
& Threshold              & --   & 30.0 & 200.0 & 5.00  & --          & --   & 30.0 & 200.0 & 5.00  & -- \\
& Initial                & \underline{0.0}  & 10.0 & 186.9 & 0.30 & \checkmark  & 15.0 & 5.0  & 232.8 & 0.30 & \ding{55} \\
& \agentgrpoloose  & \sametextbf{30.0} & 20.0 & -19.7 & 0.30 & \checkmark  & 15.0 & 5.0  & 144.3 & 0.30 & \checkmark \\
& \agentgrpostrict & 15.0 & 60.0 & 3.3   & 0.30 & \ding{55}   & 0.0  & 0.0  & 95.1  & 0.30 & \checkmark \\
& \dcapo{} ($D{=}1$) & \sametextbf{30.0} & 0.0 & -78.7 & 3.68 & \checkmark & \underline{20.0} & 0.0 & 165.6 & 3.80 & \checkmark \\
& \dcapo{} ($D{=}2$) & \sametextbf{30.0} & 0.0 & -36.1 & 2.47 & \checkmark
& \sametextbf{25.0} & 0.0 & 91.8 & 4.71 & \checkmark \\
\midrule
\multirow{6}{*}{Retail}
& Threshold              & --   & 30.0 & 60.0  & 5.00  & --          & --   & 30.0 & 60.0  & 5.00  & -- \\
& Initial                & 25.0 & 15.0 & 72.0  & 0.30 & \ding{55}   & 45.0 & 12.5 & 87.6  & 0.30 & \ding{55} \\
& \agentgrpoloose  & 25.0 & 20.0 & 109.1 & 0.30 & \ding{55}   & 30.0 & 5.0  & 54.8  & 0.30 & \checkmark \\
& \agentgrpostrict & 2.5  & 15.0 & 2.2   & 0.30 & \checkmark  & 12.5 & 12.5 & 13.4  & 0.30 & \checkmark \\
& \dcapo{} ($D{=}1$) & \underline{22.5} & 5.0  & 5.9   & 2.92 & \checkmark  & \sametextbf{45.0} & 2.5 & 53.2 & 3.18 & \checkmark \\
& \dcapo{} ($D{=}2$) & \sametextbf{30.0} & 7.5 & 54.8 & 2.86 & \checkmark
& \underline{37.5} & 7.5 & 58.6 & 3.36 & \checkmark \\
\midrule
\multirow{6}{*}{Telecom}
& Threshold              & --   & 60.0 & 40.0  & 5.00  & --          & --   & 60.0 & 40.0  & 5.00  & -- \\
& Initial                & 5.0  & 20.0 & 169.8 & 0.30 & \ding{55}   & 25.0 & 35.0 & 124.1 & 0.30 & \ding{55} \\
& \agentgrpoloose  & 0.0  & 5.0  & 71.0  & 0.30 & \ding{55}   & \underline{17.5} & 30.0 & 32.1  & 0.30 & \checkmark \\
& \agentgrpostrict & 0.0  & 10.0 & 115.4 & 0.30 & \ding{55}   & 2.5  & 5.0  & -17.3 & 0.30 & \checkmark \\
& \dcapo{} ($D{=}1$) & \underline{12.5} & 12.5 & 9.9   & 4.71  & \checkmark   & \underline{17.5} & 22.5 & 27.2  & 2.63 & \checkmark \\
& \dcapo{} ($D{=}2$) & \sametextbf{15.0} & 12.5 & 1.2 & 3.59 & \checkmark
& \sametextbf{30.0} & 30.0 & 37.7 & 4.04 & \checkmark \\
\bottomrule
\end{tabular}
\end{table*}

\begin{wrapfigure}{r}{0.48\linewidth}
\vspace{-1.4\baselineskip}
\centering
\includegraphics[width=\linewidth]{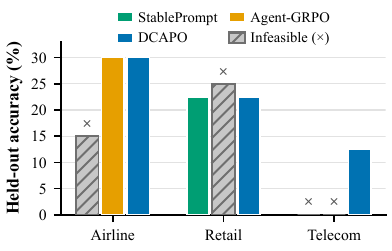}
\captionsetup{font=footnotesize,hypcap=false}
\caption{\textbf{RL methods with Qwen3-8B.} Bars give held-out accuracy; gray
hatching and $\times$ denote infeasible results.}
\label{fig:rl_prompt_comparison_8b}
\vspace{-0.5\baselineskip}
\end{wrapfigure}

\subsection{Online Adaptation and Feedback}
Table~\ref{tab:trainable_rewriter_main} separates the rewrite model, online
adaptation, and feedback form on the same held-out Airline setup. Comparing
\capo{}--SFT with the summary-feedback \dcapo{} row isolates online adaptation:
with the same SLM and summary feedback, online training reduces ToolEx and
restores feasibility without reducing accuracy. Comparing the
two \dcapo{} rows isolates feedback form: holding the SLM and online GRPO fixed,
summary feedback recovers CAPO's 45\% accuracy and provides more ToolEx slack.


\section{Limitations}
\label{sec:limitations}

Our experiments isolate constraint-aware prompt optimization by keeping the
task agent fixed and modifying only its system prompt. This design separates
the effect of prompt rewriting from changes to the deployed policy, while joint
agent--prompt training and optimization of other artifacts, such as skills or
execution harnesses, offer complementary settings with potentially different
dynamics. Our evaluation spans multiple agent and assistant tasks, models, and
constraint families, broader domains and cross-domain transfer would further
test generalization.
\section{Conclusion}
\label{sec:conclusion}

We introduced \capo{} to optimize the behavior of language-model agents through
system-prompt updates under explicit deployment thresholds. Its key idea is to let measured constraint residuals determine
which failures receive attention during search, rather than fixing their
relative importance in advance. Across tool-using agents and assistant-style
tasks, this feedback more consistently finds empirically feasible prompts
without updating the task agent. \dcapo{} further shows that the rewriting
process can be learned from behavioral feedback with pool-based GRPO, while our
surrogate analysis makes explicit the errors introduced by finite prompt pools
and imperfect rewrites. Taken together, these results suggest a broader view of
prompt optimization: deployment requirements can serve not only as final
evaluation criteria, but also as the feedback that guides the search.

\newpage
\bibliographystyle{plainnat}
\bibliography{reference}


\appendix
\newpage
\begin{center}
{\LARGE\bfseries Appendix}
\end{center}
\vspace{0.8em}

\noindent{\Large\bfseries Contents}
\vspace{0.8em}

\begingroup
\small
\newcommand{\appentry}[3]{%
\noindent\hspace*{2.6em}\makebox[2.4em][l]{\textbf{\hyperref[#1]{#2}}}%
\textbf{\hyperref[#1]{#3}}\dotfill\pageref{#1}\par\vspace{0.35em}}
\newcommand{\appsubentry}[3]{%
\noindent\hspace*{5.0em}\makebox[3.2em][l]{\hyperref[#1]{#2}}%
\hyperref[#1]{#3}\dotfill\pageref{#1}\par}

\appentry{app:theory}{A}{Surrogate Primal--Dual Analysis}
\appsubentry{app:theory_setup}{A.1}{Setup and Assumptions}
\appsubentry{app:theory_bridge_discrete}{A.2}{Discrete Rewrites and Surrogate Gradients}
\appsubentry{app:theory_algo_defs}{A.3}{Algorithmic Definitions}
\appsubentry{app:theory_inner}{A.4}{Prompt-Pool Approximation Error}
\appsubentry{app:theory_main_thm}{A.5}{Inexact Primal--Dual Bound}
\appsubentry{app:rewrite_oracle_reduction}{A.6}{Rewrite Error Propagation}
\appentry{app:hyperparams}{B}{Experiment Details}
\appsubentry{app:datasets}{B.1}{Dataset Details}
\appsubentry{app:training_details}{B.2}{Training Details}
\appsubentry{app:opt_protocol}{B.2.1}{CAPO Training}
\appsubentry{app:baseline_protocol}{B.2.2}{Baseline Setup}
\appsubentry{app:complexity}{B.2.3}{Complexity and Runtime}
\appsubentry{app:slm_scaling}{B.2.4}{Open-Weight Targets}
\appsubentry{app:learned_editor}{B.2.5}{\dcapo{} Training}
\appsubentry{app:additional_experiments}{B.3}{Additional Experiments}
\appsubentry{app:five_constraints}{B.3.1}{Constraint-Set Scaling}
\appsubentry{app:slm_ablation}{B.3.2}{Model-Scale Ablation}
\appsubentry{app:progressive}{B.3.3}{Progressive Constraint Addition}
\appsubentry{app:editor_size_sensitivity}{B.3.4}{Editor-Size Ablation}
\appsubentry{app:generalization_transfer}{B.3.5}{Generalization and Transfer}
\appsubentry{app:learned_editor_controls}{B.3.6}{RL-Based Baselines}
\appsubentry{app:analysis}{B.4}{Analysis}
\appsubentry{app:training_dynamics}{B.4.1}{Training Dynamics}
\appsubentry{app:robustness_full}{B.4.2}{Robustness and Sensitivity}
\appsubentry{app:rewrite_alignment}{B.4.3}{Rewrite Alignment}
\appentry{app:prompts}{C}{Prompt-Level Analysis and Prompt Listings}
\appsubentry{app:prompt_mechanisms}{C.1}{Prompt-Level Mechanisms}
\appsubentry{app:prompts_agent}{C.2}{Agent (\textsc{tau2-bench}) System Prompts}
\appsubentry{app:prompts_chatbot}{C.3}{Chatbot System Prompts}
\appsubentry{app:prompts_pupa}{C.4}{PUPA--IFBench System Prompts}
\appentry{app:broader_impact}{D}{Broader Impact}
\endgroup

\newpage
\section{Surrogate Primal--Dual Analysis}
\label{app:theory}
\input{appendix_theory}

\section{Experiment Details}
\label{app:hyperparams}

\subsection{Dataset Details}
\label{app:datasets}

Table~\ref{tab:dataset_descriptions} maps each workload to its role in
Eq.~\eqref{eq:constrained}. Each objective is reported in its original
higher-is-better direction. Every constraint is converted to a
lower-is-better cost before entering the Lagrangian. Objective and constraint
examples may come from different datasets, but every candidate prompt is
evaluated on the same task-specific workloads as its competitors.

\paragraph{Prompt-length reporting.}
\label{app:plen_reporting}
Every results table reports $\mathrm{PLen}=L/1{,}000$, where $L$ is the number of
characters in the system prompt. PLen is lower-is-better, and the reporting
threshold is 5.0. The optimizer uses the equivalent relative-excess cost
$(L/4000)-1$ with threshold $0.25$; therefore, this display transformation does
not change feasibility or any AllSat label.

\begin{table}[ht]
\centering
\caption{\textbf{Objectives and constraints by evaluation workload.}
``O'' denotes a higher-is-better objective, and ``C'' denotes a lower-is-better cost.
Thresholds and domain-specific values appear in Tables~\ref{tab:agent_results},
\ref{tab:main_results}, and~\ref{tab:swebench}.}
\label{tab:dataset_descriptions}
\scriptsize
\setlength{\tabcolsep}{4pt}
\renewcommand{\arraystretch}{1.05}
\begin{tabular}{@{}c l p{0.66\linewidth}@{}}
\toprule
\sametextbf{Type} & \sametextbf{Name} & \sametextbf{Description} \\
\midrule
\multicolumn{3}{@{}l}{\sametextbf{Agent (\textsc{tau2-bench}~\citep{Barres_Dong_Ray_Si_Narasimhan_2025})}} \\
O & \textsc{tau2-bench} (Acc.) & Tool--agent--user benchmark for policy-governed customer service. We evaluate Airline (94 tasks), Retail (114), and Telecom (114). Each task provides a user goal, domain tools, and a policy document; the benchmark grader determines whether the final agent state completes the requested task. \\
C & HAR & Mean number of transfer-to-human calls per dialogue, reported as $100$ times the mean. Values can exceed $100$ when one dialogue issues multiple transfer requests. \\
C & ToolEx & Relative tool-use excess, $(n_\text{actual}-n_\text{reference})/n_\text{reference}\times100\%$. Zero matches the reference path, positive values use more calls, and negative values use fewer. \\
C & PLen & System-prompt character count divided by 1{,}000; the reporting
threshold is 5.0. \\
\midrule
\multicolumn{3}{@{}l}{\sametextbf{Chatbot}} \\
O & GSM8K (Acc.) & Exact-answer accuracy on 1{,}319 grade-school math problems; optimization draws from the 7{,}473-example training split. \\
C & AdvBench & Failure-to-refuse rate on adversarial harmful requests~\citep{zouUniversalTransferableAdversarial2023}. \\
C & Over-Rejection & False-refusal rate on benign requests that share surface terms with harmful requests (e.g., ``kill a Linux process''). \\
C & Character Counting & Error rate on synthetic requests that ask for the number of occurrences of a specified character. \\
C & Response Length & Fraction of GSM8K responses whose complete generated answer exceeds 512 characters. \\
\midrule
\multicolumn{3}{@{}l}{\sametextbf{PUPA--IFBench}} \\
O & PUPA~\citep{liPAPILLONPrivacyPreservation2025} & Overall task-quality score for privacy-sensitive delegation in which prompts can contain PII, credentials, or confidential context that must be handled before remote-model use. \\
C & IFBench & Violation rate on verifiable format, content, linguistic, and structural instructions~\citep{pyatkinGeneralizingVerifiableInstruction2025}. \\
\midrule
\multicolumn{3}{@{}l}{\sametextbf{SWE-agent (\textsc{SWE-bench}~\citep{jimenez2024swebench})}} \\
O & \textsc{SWE-bench} Lite (Resolve) & Real GitHub issue-resolution tasks where a coding agent edits repository checkouts and submits patches evaluated by the SWE-bench harness. The study includes a 10-instance optimization run and a 30-instance final evaluation. \\
C & Patch Size & Lines added plus lines removed in the submitted patch, normalized by the 75th percentile gold-patch size. \\
C & Tool Actions & SWE-agent action/tool steps normalized by the 50-step agent budget. \\
C & Files Touched & Number of distinct files modified, normalized by the 75th percentile number of files touched in gold patches. \\
\bottomrule
\end{tabular}
\end{table}

\subsubsection{Dataset Statistics}
\label{app:datasets:stats}

\begin{table}[ht]
\centering
\caption{\textbf{Optimization and evaluation split sizes.} Counts follow
benchmark-provided or experiment-specific splits and need not sum to the full
corpus size.}
\small
\begin{tabular}{@{}llrr@{}}
  \toprule
  \sametextbf{Setting} & \sametextbf{Dataset} & \sametextbf{Opt.} & \sametextbf{Eval.} \\
\midrule
\multirow{3}{*}{I (Agent)} 
  & Airline & 74 & 20 \\
  & Retail & 74 & 40 \\
  & Telecom & 74 & 40 \\
\midrule
\multirow{5}{*}{II (Chatbot)} 
  & GSM8K  & 7,473 & 1,319 \\
  & AdvBench (safety) & 520 & 128 \\
  & Over-Rejection & 128 & 128 \\
  & CharCount & 200 & 128 \\
  & Length (applied to GSM8K) & 7,473 & 1,319 \\
\midrule
\multirow{2}{*}{III (PUPA)} 
  & PUPA  & 400 & 128 \\
  & IFBench  & 200 & 128 \\
\midrule
IV (SWE-agent) & \textsc{SWE-bench} Lite & 270 & 30 \\
\bottomrule
\end{tabular}
\end{table}

\subsection{Training Details}
\label{app:training_details}

This subsection collects the configurations used to train and evaluate CAPO,
DCAPO, their baselines, and the open-weight target.

\paragraph{Training framework.}
All trainable-policy experiments use the AReaL codebase
\citep{fu2025areal}. This includes SFT and pool-based GRPO for the \dcapo{}
rewriter, direct task-policy GRPO for Agent-GRPO, and APPO training for
StablePrompt.

\subsubsection{CAPO Training}
\label{app:opt_protocol}

\paragraph{Optimization loop.}
In each round, we (i) evaluate the retained pool under
$\boldsymbol\lambda_t$; (ii) sample parents and request rewrites conditioned on
representative task failures, constraint examples, and the current
multipliers; (iii) evaluate the expanded pool; (iv) update multipliers from the
selected residuals via Eq.~\eqref{eq:lambda_update}; and (v) retain the next
beam using the round-$t$ scores. CAPO runs with the frozen rewriter use up to
six rounds. The coding-agent study uses four rounds, and the progressive chatbot
study uses 12; these are reported with their corresponding results.

\paragraph{Threshold selection.}
\label{app:threshold_protocol}
All thresholds are fixed and shared across methods within the same
domain--model setting. For each setting, we measure the initial system prompt's
mean cost $c_i^{(0)}$ on a training split separate from final evaluation and
set each budget relative to that baseline. The optimizer uses a $+0.25$
threshold for the relative PLenEx cost, equivalent to the reported PLen
threshold of 5.0. Thresholds are frozen before CAPO sweeps or baseline tuning.
This procedure produces budgets anchored to the training-split baseline; we
assess all methods against them on a separate evaluation split. ToolEx
sensitivity is reported in Tables
\ref{tab:tol_ablation} and~\ref{tab:telecom_relaxed_sensitivity}. The
coding-agent setup is documented separately in Appendix~\ref{app:swebench}.

\paragraph{Reproducibility.}
To support reproducibility, this appendix documents the optimization loop, data
splits, thresholds, hyperparameters, baseline settings, and compute setup. The
public benchmarks are available at the cited URLs, and representative optimized
prompts and aggregate evaluation summaries are included below.
Appendix~\ref{app:complexity} reports relative cost comparisons and per-round
counts. Hyperparameters and baseline reproduction details are in
Appendix~\ref{app:baseline_protocol}.

\paragraph{Model compute}
All prompt-optimization baselines and all setups that use GPT-5-mini or GPT-5.1
as the task model or user simulator are evaluated through the OpenAI API,
without local GPU inference. The documented model versions are
\texttt{gpt-5-mini-2025-08-07} and \texttt{gpt-5.1-2025-11-13}
\citep{openaiGPT5MiniModel2025,openaiGPT51Model2025}. All small language models are served via vLLM on A100.

\paragraph{Dual learning rate.}
The dual learning rate $\beta$ controls a stability--responsiveness trade-off:
small values react slowly, whereas large values can overshoot under noisy or
competing constraints. We use $\beta{=}4$ as the configured default. The
matched sensitivity sweep in Table~\ref{tab:lambda_lr_sweep} shows that the trade-off between task performance and constraints is sensitive to the dual learning rate.

\begin{table}[ht]
\centering
\caption{\textbf{Representative CAPO configuration on
\textsc{tau2-bench}.} All settings use a frozen rewriter.}
\label{tab:capo_hyperparams_full}
\begin{tabular}{@{}ll@{}}
  \toprule
  \sametextbf{Parameter} & \sametextbf{Value} \\
\midrule
Max optimization rounds & 6 \\
Retained beam size ($k_1$) & 6 \\
Parents sampled per round ($k$) & 4 \\
Children per parent & 2 \\
Dual learning rate $\beta$ & 4 \\
Multiplier cap $\lambda_{\max}$ & 10 \\
Best prompts for $\lambda$ estimate ($k_0$) & 1 \\
Initial $\lambda$ & 1 \\
Parent selection temperature $\gamma$ & 1 \\

\bottomrule
\end{tabular}
\end{table}

\FloatBarrier
\subsubsection{Baseline Setup}
\label{app:baseline_protocol}

\paragraph{Shared candidate evaluation.}
All methods receive the same objective examples, constraint examples, and
fixed thresholds for every candidate. The workloads are
merged for evaluation scheduling, but their metrics are not collapsed: for
every candidate prompt $p$, we compute $r(p)$ on the objective subset and each
$c_i(p)$ on its corresponding constraint subset. The baseline methods
therefore receive explicit constraint measurements rather than objective-only
feedback. The methods evaluate different numbers of candidates and make
different numbers of optimizer calls; these per-round differences are reported
in Appendix~\ref{app:complexity}.

\paragraph{Multi-task APO and GEPA.}
We retain each method's proposal and update mechanism but score candidates with
the fixed equal-multiplier objective
\begin{equation}
\label{eq:baseline_fixed_score}
J(p)=r(p)-\sum_i\bar{\lambda}_i\bigl(c_i(p)-\tau_i\bigr),
\qquad \bar{\lambda}_i=1.
\end{equation}
Every threshold $\tau_i$ enters through its signed residual
$c_i(p)-\tau_i$, and $\bar{\lambda}_i=1$ for every constraint and every round.
Neither method adapts these coefficients from measured violations. With fixed
multipliers, $\sum_i\bar\lambda_i\tau_i$ is constant across candidates, so the
thresholds define the shared constrained problem and the empirical AllSat
criterion but do not by themselves change the candidate ordering.

\paragraph{MOPO.}
Our MOPO baseline retains the method's NSGA-II non-dominated sorting and
crowding-distance selection rule on
\begin{equation}
\label{eq:mopo_baseline_vector}
\bigl(r(p),-c_1(p),\ldots,-c_m(p)\bigr).
\end{equation}
The common thresholds determine whether each returned prompt satisfies all
constraints. Subtracting the fixed $\tau_i$ from each cost would not change
Pareto dominance, so the ranking itself uses the native raw-cost vector. MOPO
does not convert residuals into adaptive scalar weights.

\paragraph{EvoPrompt--GA.}
For the mechanism ablations in Tables~\ref{tab:ablations_alt} and
\ref{tab:ablations_alt_retail_telecom}, we use the tournament-selection
EvoPrompt--GA variant: two size-two tournaments choose the parents, an LLM
performs crossover and mutation, and top-$K$ elitist survival retains the next
population. Its fitness is Eq.~\eqref{eq:baseline_fixed_score}, with
$\bar\lambda_i=1$ fixed for every constraint and round. Thus EvoPrompt--GA
receives the objective measurement, each constraint measurement, and the
fixed thresholds, but performs no dual update.

\FloatBarrier
\subsubsection{Complexity and Runtime}
\label{app:complexity}

Table~\ref{tab:cost_table} summarizes per-round asymptotic complexity. Let
$E_{\mathrm{tot}}=E_{\mathrm{task}}+\sum_i E_i$ denote the shared objective and
constraint workloads evaluated for one candidate. The number of candidates,
rather than the workload per candidate, determines the cost differences: MOPO evaluates $b$
offspring per candidate, whereas CAPO and GEPA rewrite $k$ selected parents.
CAPO's dual update is negligible relative to rollout calls and its
optimizer-side critique cost matches GEPA's $O(kC)$.

\begin{table}[ht]
\caption{\textbf{CAPO matches GEPA's per-round critique and evaluation
complexity while satisfying all three domains.} Here $n$ is the pool size,
$k$ the selected-parent count ($k\le n$), $b$ the number of MOPO
offspring per candidate, $E_{\mathrm{tot}}$ the full per-candidate workload,
and $C$ one optimizer-LLM call. \sametextbf{AllSat (3/3)} counts the
\textsc{tau2-bench} domains in which every metric meets its GPT-5-mini threshold in
Table~\ref{tab:agent_results}.}
\label{tab:cost_table}
\centering
\small
\setlength{\tabcolsep}{6pt}
\begin{tabular}{@{}lccc@{}}
\toprule
\sametextbf{Method} & \sametextbf{Critique cost / round} & \sametextbf{Eval cost / round} & \sametextbf{AllSat (3/3)} \\
\midrule
APO   & $O(nC)$       & $O(nE_{\mathrm{tot}})$       & 1 / 3 \\
MOPO  & $O(nbC)$      & $O(nbE_{\mathrm{tot}})$      & 0 / 3 \\
GEPA  & $O(kC)$       & $O(nE_{\mathrm{tot}})$       & 1 / 3 \\
\capo{} & $O(kC)$ & $O(nE_{\mathrm{tot}})$ & 3 / 3 \\
\bottomrule
\end{tabular}
\end{table}

\FloatBarrier
\subsubsection{Open-Weight Targets}
\label{app:slm_scaling}

\paragraph{Small language model.}
To test CAPO with an open-weight target, we evaluate Ministral-8B
\citep{liu2026ministral}, an 8B-parameter model with lower baseline
instruction-following accuracy than GPT-5-mini. Absolute accuracy decreases
across methods, but \caposgd{} is the only evaluated method that satisfies all
constraints in all three domains (Table~\ref{tab:agent_results_ministral}).
Appendix~\ref{app:slm_ablation} separately compares four Qwen2.5 model sizes on
the chatbot task.

\begin{table}[ht]
\caption{\textbf{\caposgd{} is the only method feasible in all three domains
with Ministral-8B.} Results follow the format of
Table~\ref{tab:agent_results}.
Bold marks the highest accuracy among feasible methods, and underlining marks
the next distinct feasible accuracy; ties are included.}
\label{tab:agent_results_ministral}
\centering
\scriptsize
\setlength{\tabcolsep}{4pt}
\renewcommand{\arraystretch}{0.95}
\begin{tabular}{@{}l r r r r c@{}}
  \toprule
  \sametextbf{Setting / Method} & \sametextbf{Acc.~$\uparrow$} & \sametextbf{HAR~$\downarrow$} & \sametextbf{ToolEx~$\downarrow$} & \sametextbf{PLen~$\downarrow$} & \sametextbf{AllSat} \\
\midrule
\multicolumn{6}{@{}l}{\sametextbf{Airline}} \\
\sametextbf{Threshold}
  & -- & 30 & 200 & 5.00 & -- \\
Initial prompt
  & 15.0 & 10.0 & 219.7 & 0.30 & \ding{55} \\
APO
  & \sametextbf{20.0} & 5.0 & 188.5 & 4.24 & \checkmark \\
GEPA
  & 10.0 & 25.0 & 239.3 & 5.54 & \ding{55} \\
MOPO
  & 15.0 & 25.0 & 237.7 & 4.72 & \ding{55} \\
\caposgd
  & \sametextbf{20.0} & 5.0 & 177.1 & 4.18 & \checkmark \\
\midrule
\multicolumn{6}{@{}l}{\sametextbf{Retail}} \\
\sametextbf{Threshold}
  & -- & 30 & 60 & 5.00 & -- \\
Initial prompt
  & 30.0 & 20.0 & 60.2 & 0.30 & \ding{55} \\
APO
  & 30.0 & 20.0 & 60.2 & 0.30 & \ding{55} \\
GEPA
  & 37.5 & 15.0 & 86.0 & 4.41 & \ding{55} \\
MOPO
  & 40.0 & 15.0 & 73.1 & 3.86 & \ding{55} \\
\caposgd
  & \sametextbf{32.5} & 20.0 & 55.4 & 2.68 & \checkmark \\
\midrule
\multicolumn{6}{@{}l}{\sametextbf{Telecom}} \\
\sametextbf{Threshold}
  & -- & 60 & 40 & 5.00 & -- \\
Initial prompt
  & 20.0 & 62.5 & 16.7 & 0.30 & \ding{55} \\
APO
  & 20.0 & 62.5 & 16.7 & 0.30 & \ding{55} \\
GEPA
  & 22.5 & 65.0 & 75.3 & 0.32 & \ding{55} \\
MOPO
  & 10.0 & 47.5 & 40.1 & 4.89 & \ding{55} \\
\caposgd
  & \sametextbf{25.0} & 35.0 & 38.3 & 4.66 & \checkmark \\
\bottomrule
\end{tabular}
\end{table}

\FloatBarrier
\subsubsection{\dcapo{} Training}
\label{app:learned_editor}

\paragraph{Constrained reward.}
The rewriter emits a child system prompt. The frozen task agent executes that
prompt, and the environment assigns the empirical score
\[
J_{\boldsymbol{\lambda}} = \mathrm{pass@1}
    - \lambda_{\mathrm{te}}(\mathrm{te}-\tau_{\mathrm{te}})
    - \lambda_{\mathrm{har}}(\mathrm{har}-\tau_{\mathrm{har}}).
\]
Here \textsc{te} is excess tool use and \textsc{har} is human-agent request
rate. The multipliers start at one and follow Eq.~\eqref{eq:lambda_update}. Residuals
remain signed inside $J_{\boldsymbol\lambda}$; only the multipliers are
projected coordinatewise to $[0,\lambda_{\max}]$, with
$\lambda_{\max}=10$.

\paragraph{Fixed-$\lambda$ Agent-GRPO reward.}
\label{app:agent_grpo_reward}
Agent-GRPO applies the same signed-residual scalarization directly to
task-agent rollouts but holds the penalty vector fixed throughout training:
\[
J_{\bar{\boldsymbol\lambda}}^{\mathrm{Agent\text{-}GRPO}}
= \mathrm{pass@1}
  - \bar\lambda_{\mathrm{te}}(\mathrm{te}-\tau_{\mathrm{te}})
  - \bar\lambda_{\mathrm{har}}(\mathrm{har}-\tau_{\mathrm{har}}).
\]
The loose setting uses
$(\bar\lambda_{\mathrm{te}},\bar\lambda_{\mathrm{har}})=(1,2.5)$, and the
strict setting uses $(2,4)$. The two coefficients weight the ToolEx and HAR
residuals, respectively. We compute this scalar reward for each rollout before
group-relative normalization and the GRPO update; unlike DCAPO, Agent-GRPO does
not update the coefficients from observed residuals.

\paragraph{Online training setup.}
Parents are sampled off-policy from the prompt pool, and children of the same
parent share the same evaluation examples when computing group-relative
advantages. Trajectory feedback includes the complete parent trajectory,
whereas summary feedback replaces it with a critic-generated violation
summary. At depth two, the rewriter evaluates a
second child conditioned on the first child's trajectory and keeps the better
child. The initial parent receives empty feedback. Feedback and reward are
generated only from training rollouts, never from held-out results. Final prompt
comparisons use the benchmark-defined held-out evaluation split.

\FloatBarrier
\subsection{Additional Experiments}
\label{app:additional_experiments}

This subsection collects the evaluations that extend the primary settings or
compare alternative training approaches.

\subsubsection{Constraint-Set Scaling}
\label{app:five_constraints}

\paragraph{Nine-constraint Airline scaling.}
To test whether CAPO handles broader constraint sets, we run an Airline
experiment with nine constraints spanning resource use, reliability, user
experience, and safety. The easier constraints remain satisfied and their
multipliers settle near zero; harder constraints retain larger multipliers
(Figures~\ref{fig:9constraint_dynamics} and~\ref{fig:9constraint_lambdas}).
The dual update therefore gives larger weights to constraints with larger
measured violations, without manual weight tuning.

\begin{figure*}[htbp]
  \centering
  \includegraphics[width=0.90\textwidth]{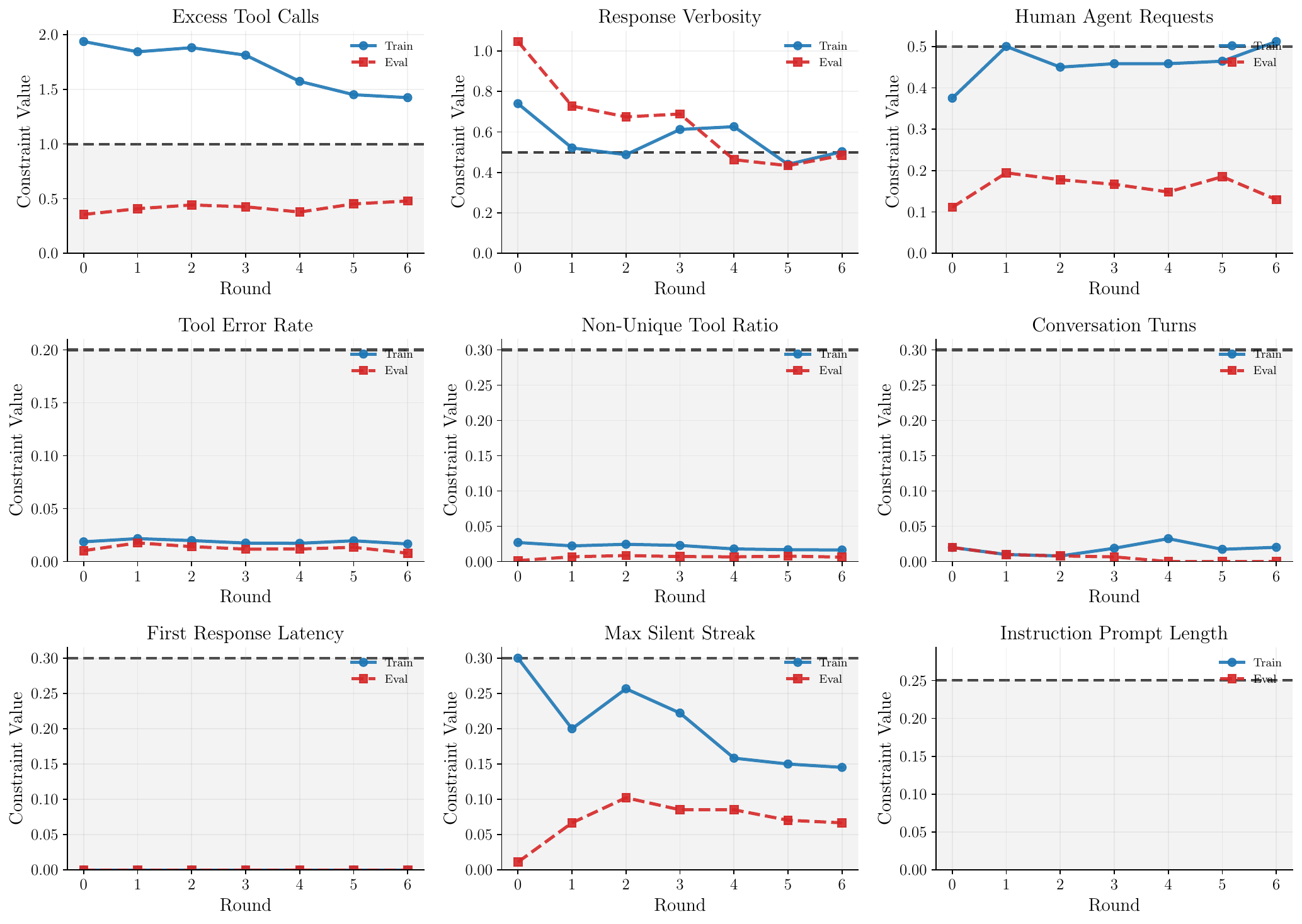}
  \caption{\textbf{Constraint dynamics with nine constraints on Airline.}
  Blue and orange curves show training and evaluation violation rates across
  optimization rounds. Violations of the turn and latency constraints remain near zero,
  whereas tool-use and verbosity violations persist and show larger gaps
  between training and evaluation.}
  \label{fig:9constraint_dynamics}
\end{figure*}

\begin{figure*}[t]
    \centering
    \includegraphics[width=0.98\textwidth]{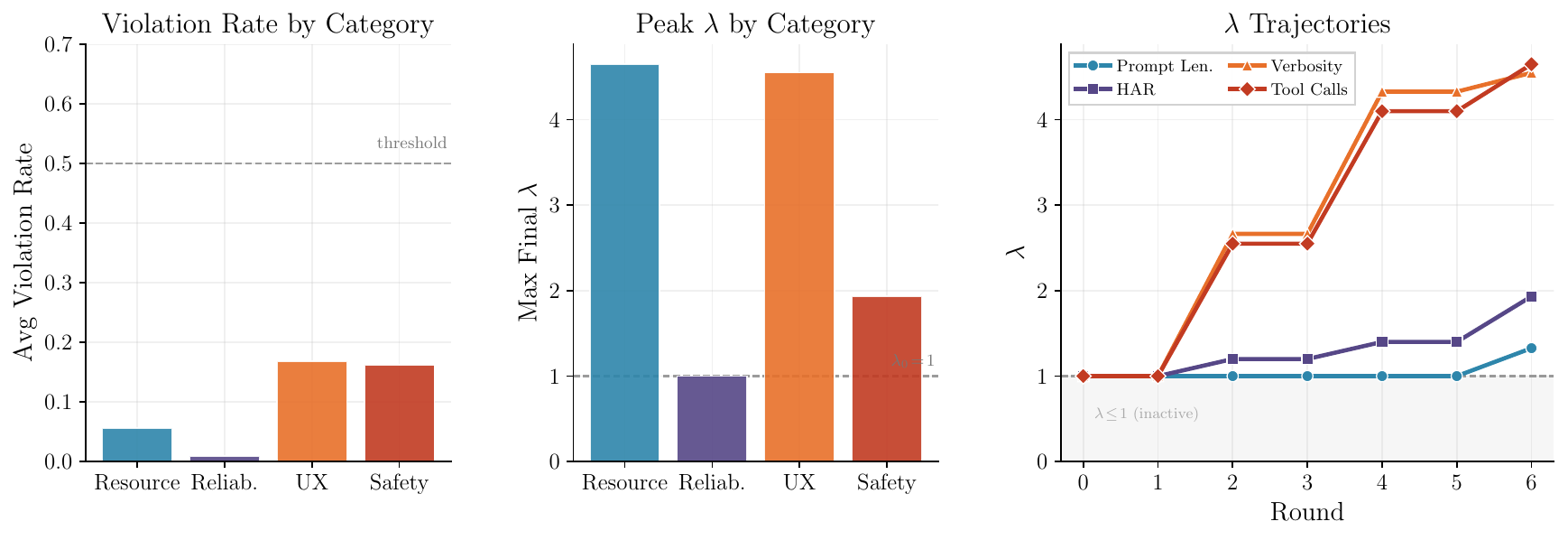}
    \caption{\textbf{Multiplier trajectories with nine Airline constraints.}
    Constraints with persistent violations accumulate larger multipliers,
    whereas the multipliers for constraints satisfied early remain stable.}
    \label{fig:9constraint_lambdas}
\end{figure*}

\FloatBarrier
\subsubsection{Model-Scale Ablation}
\label{app:slm_ablation}

We evaluate Qwen2.5 models from 1.5B to 14B parameters on GSM8K under the
length constraint (Figure~\ref{fig:slm}).

\begin{figure}[ht]
    \centering
  \includegraphics[width=0.98\textwidth]{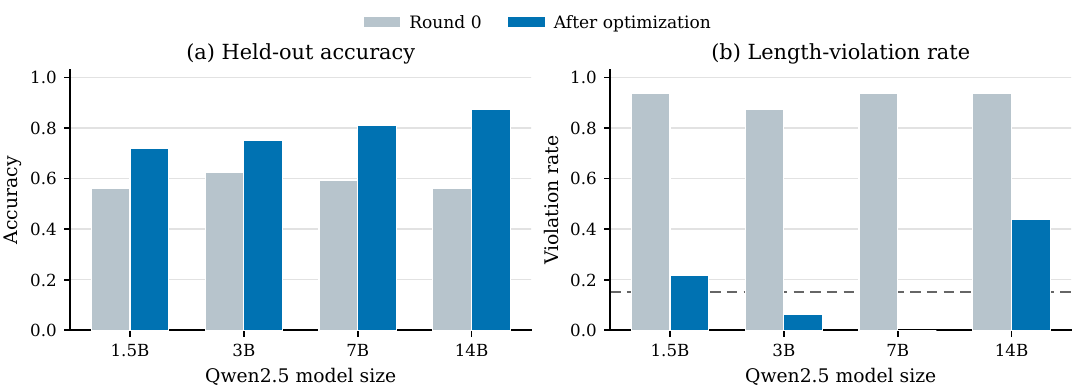}
    \caption{\textbf{Qwen2.5-7B attains the highest feasible accuracy across
    the tested model sizes.} Held-out accuracy before and after prompt
    optimization appears on the left, with the corresponding length-violation
    rates on the right. The dashed line marks the 15\% violation threshold;
    Qwen2.5-7B reaches 81.3\% accuracy with no observed violations.}
    \label{fig:slm}
\end{figure}

The accuracy--constraint trade-off is non-monotonic across the tested model
sizes. The 7B variant gives the highest feasible accuracy. The smaller models
have lower accuracy, whereas the larger variants exceed the length budget.
Across the tested Qwen2.5 family, model scale therefore changes the attainable
trade-off.

\FloatBarrier
\subsubsection{Progressive Constraint Addition}
\label{app:progressive}

In practice, new safety policies or formatting requirements can introduce
constraints incrementally. We therefore compare \emph{simultaneous}
optimization, which activates all constraints in the first round, with a
\emph{progressive} schedule that adds constraints while retaining the learned
state. On the chatbot task, progressive optimization first uses AdvBench safety
and character-counting constraints, then adds the length constraint with its
multiplier initialized to zero while retaining the prompt and existing
multipliers. The runs use the same evaluation examples and optimizer
hyperparameters. The progressive trace includes six initial two-constraint
rounds before length is added, whereas simultaneous optimization activates all
three constraints from the first of its six recorded rounds.

Figure~\ref{fig:progressive_vs_triple} shows
the recorded trajectories under the two schedules. The progressive run first
optimizes AdvBench safety and character counting, then adds length after round
6. Its accuracy recovers during the second stage as the length violation falls.
The simultaneous run activates all three constraints from its first round.

\begin{figure}[ht]
    \centering
  \includegraphics[width=0.95\textwidth]{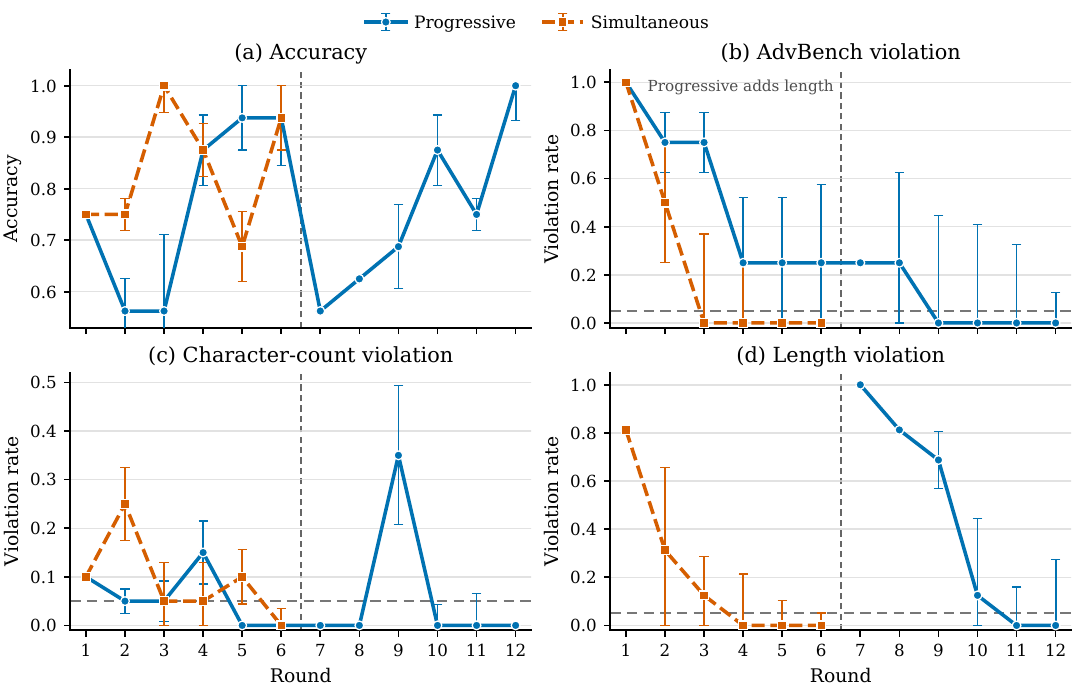}
\caption{\textbf{Progressive and simultaneous optimization follow different
trajectories.} Panels compare held-out accuracy and violation rates for the
12-round progressive run and the six recorded rounds of simultaneous
optimization. Progressive optimization uses AdvBench safety and character
counting in rounds 1--6 and adds length after round 6; simultaneous optimization
activates all three constraints from its first round. Error bars denote one
standard deviation, clipped to $[0,1]$. Horizontal dashed lines mark the
violation threshold $\tau=0.05$.}
    \label{fig:progressive_vs_triple}
\end{figure}

\paragraph{Three-constraint progressive chatbot trajectory.}
Figure~\ref{fig:progressive} shows a 12-round trajectory that begins with
AdvBench safety and character counting and later introduces response length.

\begin{figure}[ht]
    \centering
  \includegraphics[width=0.95\textwidth]{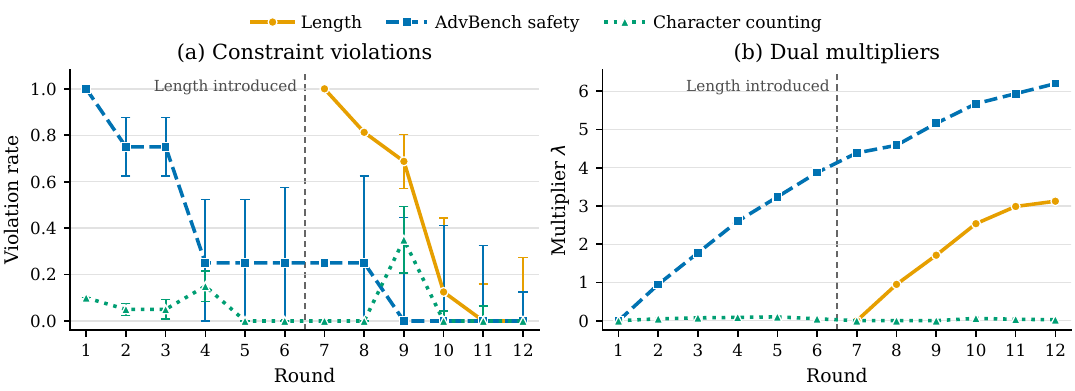}
    \caption{\textbf{Dual multipliers track active constraint violations.}
    Held-out violation rates appear on the left and the corresponding
    multiplier trajectories on the right. The vertical line
    marks the introduction of the length constraint after round 6. Error bars
    denote one standard deviation, clipped to $[0,1]$.}
    \label{fig:progressive}
\end{figure}

AdvBench requires sustained pressure: its multiplier grows as the violation
falls. After length is introduced, its multiplier rises until the length
violation reaches zero. Character counting remains near zero for most rounds
and its multiplier stays near zero.

\FloatBarrier
\subsubsection{Editor-Size Ablation}
\label{app:editor_size_sensitivity}
We additionally train Qwen3-0.6B and Qwen3-4B rewrite policies and compare them
with Qwen3-8B. Each run uses the same seed, a 30-step online budget, a frozen Qwen3-32B
task agent, and the same domain thresholds.
Figure~\ref{fig:editor_size_best_feasible} reports the best observed feasible
score for each size across the available feedback configurations. The best observed score generally improves
with editor size, although Retail is non-monotonic. Because this summary selects
both the feedback setting and prompt using held-out measurements, it is a
descriptive best-observed comparison rather than a controlled size ablation or
a model-selection estimate.

\begin{figure}[ht]
\centering
\includegraphics[width=0.76\textwidth]{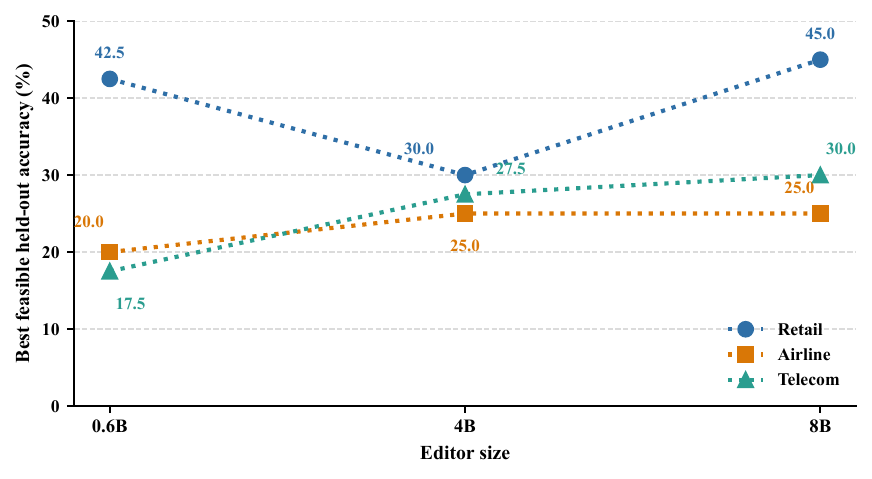}
\caption{\textbf{Best observed feasible score across editor sizes.} Dotted
lines compare Qwen3-0.6B, Qwen3-4B, and Qwen3-8B rewriters. All runs use a frozen Qwen3-32B task agent and the same user
simulator. The Airline 8B point uses the extended-context configuration. This is a
best-observed comparison across settings, not a single-setting size ablation.}
\label{fig:editor_size_best_feasible}
\end{figure}

\FloatBarrier
\subsubsection{Generalization and Transfer}
\label{app:generalization_transfer}

The following experiments test three distinct forms of generalization and
transfer. The coding-agent study applies the CAPO procedure to a different
agent setting, the task-cluster study evaluates a CAPO prompt under a shifted
task distribution, and the zero-shot study transfers the parameters of a
trained DCAPO rewriter across domains.

\paragraph{Coding-agent setting.}
\label{app:swebench}

In the coding-agent study, we optimize a GPT-5-mini SWE-agent prompt on
10 \textsc{SWE-bench} Lite instances and perform a final evaluation on 30
instances.
Resolve rate is the objective. Patch size is normalized by the
75th-percentile gold-patch size, tool actions by the 50-step agent budget, and
files touched by the 75th-percentile gold-patch count. MOPO is the only search
baseline because every candidate evaluation requires a repository checkout and
a complete API-based coding-agent rollout.

Table~\ref{tab:swebench} reports the recorded results under the listed
thresholds. All three methods resolve five issues; among them, CAPO records the
smallest patch, the fewest tool actions, and the fewest files touched.

\begin{table}[ht]
\caption{\textbf{CAPO matches the resolve rate while reducing all three
coding-agent costs.} These values underlie the coding-agent panel in
Figure~\ref{fig:transfer_diagnostics}. Bold marks the highest feasible
objective, and underlining marks the next distinct feasible objective; ties are
included. A row satisfies AllSat only if Patch, Tool, and Files each meet the
threshold listed in the first row.}
\label{tab:swebench}
\centering
\small
\setlength{\tabcolsep}{5pt}
\begin{tabular}{@{}lccccc@{}}
\toprule
\sametextbf{Method} & \sametextbf{Resolve\,$\uparrow$} & \sametextbf{Patch\,$\downarrow$} & \sametextbf{Tool\,$\downarrow$} & \sametextbf{Files\,$\downarrow$} & \sametextbf{AllSat} \\
\midrule
Threshold & --    & 1.2   & 0.4   & 1.7   & -- \\
Initial   & 0.167 & 1.298 & 0.423 & 1.867 & \ding{55} \\
MOPO      & 0.167 & 1.209 & 0.393 & 1.500 & \ding{55} \\
\capo{}   & \sametextbf{0.167} & 1.055 & 0.371 & 1.300 & \checkmark \\
\bottomrule
\end{tabular}
\end{table}

\FloatBarrier
\paragraph{Task-cluster shift.}
\label{app:task_cluster_shift}

We cluster Airline tasks in embedding space and select the training cluster
farthest from the evaluation-set centroid. CAPO meets all three thresholds on
the resulting split (Table~\ref{tab:ood_split}).

\begin{table}[ht]
\caption{\textbf{CAPO remains feasible under an embedding-defined Airline
task-cluster shift.} Objective, ToolEx, and HAR are fractions or ratios; PLen
is the system-prompt character count divided by 1,000. Bold marks the highest
feasible objective, and underlining marks the next distinct feasible objective;
ties are included.}
\label{tab:ood_split}
\centering
\scriptsize
\setlength{\tabcolsep}{5pt}
\begin{tabular}{@{}lcccc@{}}
\toprule
\sametextbf{Split} & \sametextbf{Objective\,$\uparrow$} & \sametextbf{ToolEx\,$\downarrow$} & \sametextbf{HAR\,$\downarrow$} & \sametextbf{PLen\,$\downarrow$} \\
\midrule
In-domain & \sametextbf{0.450} & 1.049 & 0.050 & 4.98 \\
Shifted   & \underline{0.400} & 1.000 & 0.250 & 4.91 \\
\bottomrule
\end{tabular}
\end{table}

\FloatBarrier
\paragraph{Zero-shot rewriter transfer.}
\label{app:zero_shot_rewriter_transfer}

We freeze a Telecom-trained rewriter and apply it to Retail and Airline, using
Qwen3-8B as the frozen task agent and Ministral-3-8B-Instruct as the frozen
user simulator. For each target, it receives the target-domain seed prompt and
sampled trajectories, generates four candidates in one inference-only round,
selects the highest-$J$ candidate on validation, and evaluates the parent and
selected rewrite on the held-out test split. There are no target-domain
gradient updates or iterative pool search.

\begin{table}[ht]
\caption{\textbf{Frozen-rewriter transfer lowers ToolEx without reaching
target-domain feasibility.} A Telecom-trained rewriter is transferred for one
inference-only round. No accuracy is bold because no row is feasible.}
\label{tab:zero_shot_rewriter_transfer}
\centering
\scriptsize
\setlength{\tabcolsep}{4.5pt}
\begin{tabular}{@{}llrrrrc@{}}
\toprule
\sametextbf{Domain} & \sametextbf{Prompt} & \sametextbf{Acc\,$\uparrow$} &
\sametextbf{ToolEx\,$\downarrow$} & \sametextbf{HAR\,$\downarrow$} &
\sametextbf{PLen\,$\downarrow$} & \sametextbf{AllSat} \\
\midrule
Retail & Threshold & -- & 60.0 & 30.0 & 5.00 & -- \\
Retail & Initial & 27.5 & 84.0 & 12.5 & 0.32 & \ding{55} \\
Retail & Frozen Telecom rewrite & 32.5 & 67.0 & 12.5 & 1.96 & \ding{55} \\
\midrule
Airline & Threshold & -- & 200.0 & 30.0 & 5.00 & -- \\
Airline & Initial & 10.0 & 329.0 & 15.0 & 0.32 & \ding{55} \\
Airline & Frozen Telecom rewrite & 10.0 & 220.0 & 10.0 & 3.04 & \ding{55} \\
\bottomrule
\end{tabular}
\end{table}

The frozen rewriter lowers ToolEx in both target domains without reducing
accuracy or pushing the other costs above their thresholds, but ToolEx remains
above its threshold. The transferred rewriter therefore makes a useful edit
but does not achieve feasibility without target-domain adaptation.

\FloatBarrier
\subsubsection{RL-Based Baselines}
\label{app:learned_editor_controls}

StablePrompt~\citep{kwonEtAlStablePrompt2024} trains a prompt-generation policy
with APPO. For this comparison, its fixed scalar reward assigns unit weight to
every constraint. \agentgrpoloose{} directly updates the task agent under
predeclared fixed weights. \dcapo{} instead freezes the task agent, trains the
trajectory- and dual-conditioned rewriter with group-relative advantages, and
adds high-scoring children to an evolving prompt pool. Neither RL-based baseline updates
constraint-specific multipliers online. Table~\ref{tab:rl_prompt_controls} uses
Qwen3-8B as the task agent.

\begin{table}[ht]
\caption{\textbf{DCAPO is feasible in all three Qwen3-8B domains.}
StablePrompt and Agent-GRPO provide prompt-policy and task-policy comparisons,
respectively; \dcapo{} uses complete trajectories ($D{=}1$). Agent-GRPO uses the loose fixed penalty vector
$(\lambda_{\mathrm{ToolEx}},\lambda_{\mathrm{HAR}})=(1,2.5)$. AllSat
holds only when every cost meets its domain threshold. Within each domain,
bold marks the highest feasible accuracy, and underlining marks the next
distinct feasible accuracy; ties are included.}
\label{tab:rl_prompt_controls}
\centering
\scriptsize
\setlength{\tabcolsep}{4.2pt}
\renewcommand{\arraystretch}{1.02}
\begin{tabular}{@{}llrrrrc@{}}
\toprule
\sametextbf{Domain} & \sametextbf{Method} & \sametextbf{Acc\,$\uparrow$} &
\sametextbf{ToolEx\,$\downarrow$} & \sametextbf{HAR\,$\downarrow$} &
\sametextbf{PLen\,$\downarrow$} & \sametextbf{AllSat} \\
\midrule
Airline & Threshold & -- & 200.0 & 30.0 & 5.00 & -- \\
Airline & StablePrompt & 15.0 & 436.4 & 5.0 & 0.63 & \ding{55} \\
Airline & \agentgrpoloose & \sametextbf{30.0} & -19.7 & 20.0 & 0.30 & \checkmark \\
Airline & \dcapo{} & \sametextbf{30.0} & -78.7 & 0.0 & 3.68 & \checkmark \\
\midrule
Telecom & Threshold & -- & 40.0 & 60.0 & 5.00 & -- \\
Telecom & StablePrompt & 0.0 & 81.9 & 2.5 & 0.71 & \ding{55} \\
Telecom & \agentgrpoloose & 0.0 & 71.0 & 5.0 & 0.30 & \ding{55} \\
Telecom & \dcapo{} & \sametextbf{12.5} & 9.9 & 12.5 & 4.71 & \checkmark \\
\midrule
Retail & Threshold & -- & 60.0 & 30.0 & 5.00 & -- \\
Retail & StablePrompt & \sametextbf{22.5} & 31.3 & 5.0 & 0.57 & \checkmark \\
Retail & \agentgrpoloose & 25.0 & 109.1 & 20.0 & 0.30 & \ding{55} \\
Retail & \dcapo{} & \sametextbf{22.5} & 5.9 & 5.0 & 2.92 & \checkmark \\
\bottomrule
\end{tabular}
\end{table}

DCAPO trained on complete trajectories satisfies all constraints in all three
domains. It ties the highest feasible accuracy on Airline and Retail and is the only feasible
method on Telecom. This comparison separates learning a constraint-aware
rewriter within an adaptive pool from RL prompt generation or direct
task-policy training under a fixed weighted objective.

\FloatBarrier
\subsection{Analysis}
\label{app:analysis}

This subsection examines optimization dynamics, robustness, sensitivity, and
the alignment between learned rewrites and estimated ascent directions.

\subsubsection{Training Dynamics}
\label{app:training_dynamics}
\label{app:trajectories}

\paragraph{CAPO trajectories.}
Figure~\ref{fig:chatbot_capo} shows the optimization trajectory on the chatbot
GSM8K setting. CAPO reaches the feasible region in fewer rounds because a
persistent violation increases its multiplier and strengthens its influence on
the next prompt update. After all constraints are satisfied, stable or
decreasing multipliers allow later updates to improve accuracy without losing
feasibility.

\begin{figure}[ht]
    \centering
  \includegraphics[width=0.85\textwidth]{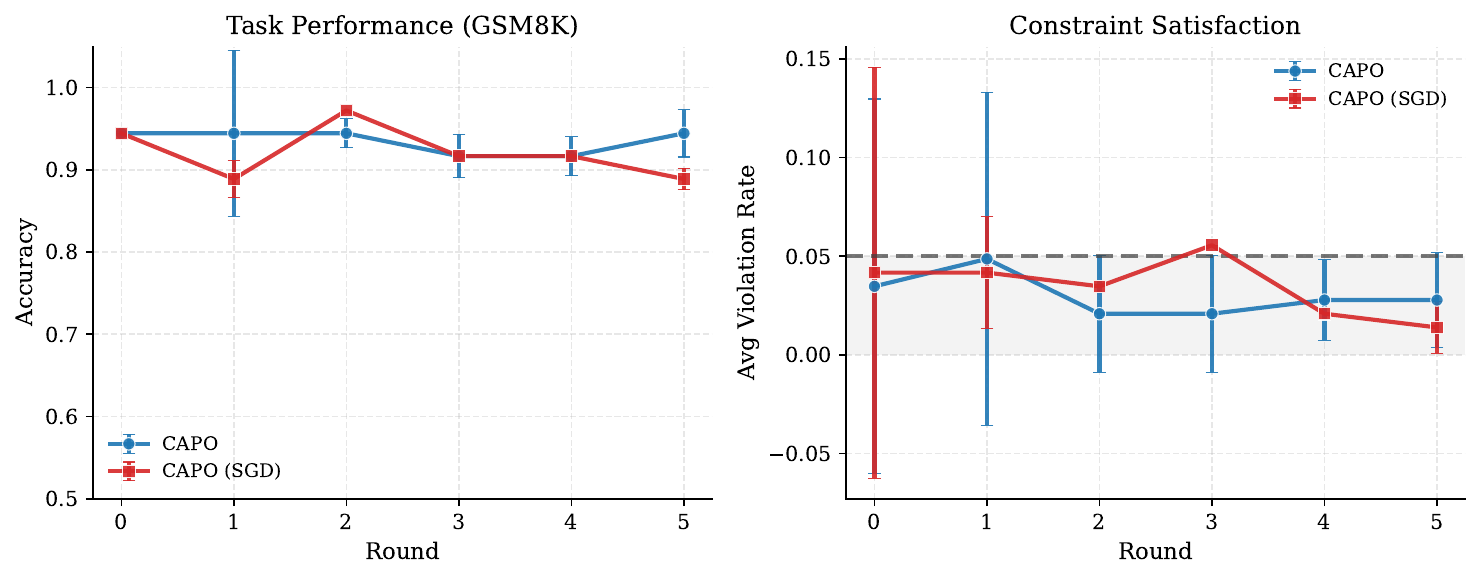}
    \caption{\textbf{CAPO reaches chatbot feasibility sooner and finishes with
    higher task accuracy.} The trajectories track constrained prompt
    optimization across rounds.}
    \label{fig:chatbot_capo}
\end{figure}

\paragraph{DCAPO dynamics.}
\label{app:dcapo_training_dynamics}
Figure~\ref{fig:learned_editor_dynamics} compares feedback-conditioned training
with a feedback-free control as prompts enter the pool. The feedback-free control
has higher Pass@1 in this run but later violates the HAR threshold. The
feedback-conditioned rewriter trades some task reward for lower ToolEx and HAR
and satisfies both thresholds after roughly 80 accepted prompts. This
figure isolates the presence of behavioral feedback; Table
\ref{tab:trainable_rewriter_main} separately compares trajectory and summarized
feedback at final evaluation.

\begin{figure}[ht]
\centering
\includegraphics[width=0.90\textwidth]{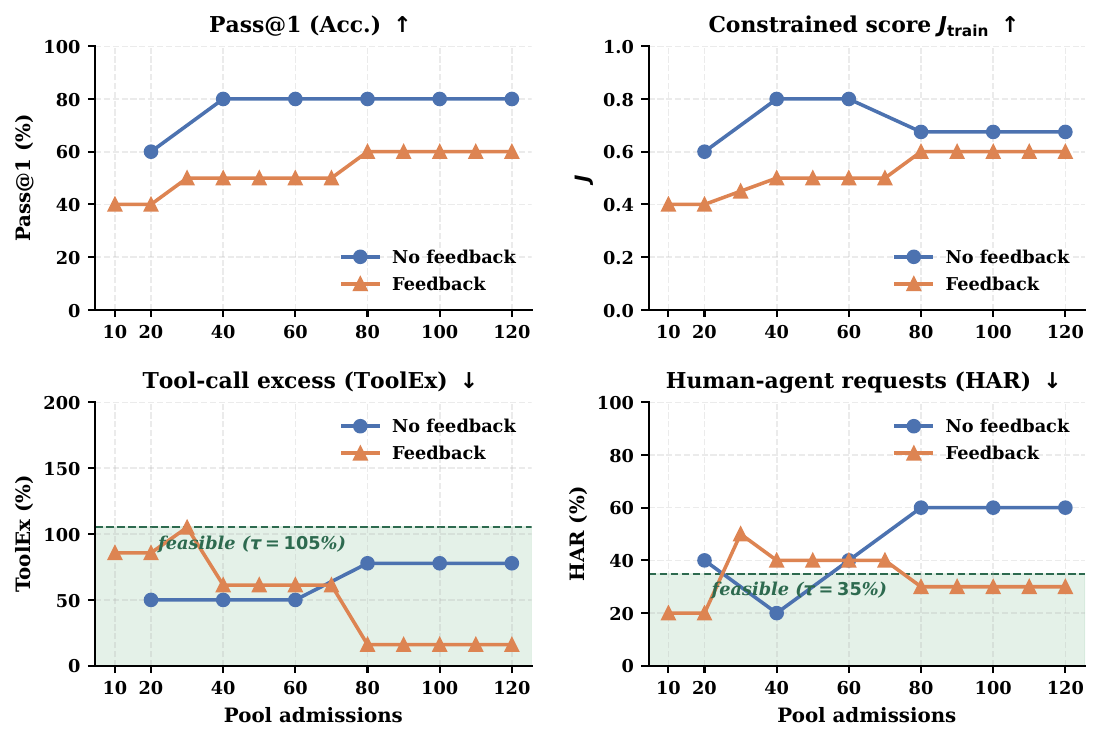}
\caption{\textbf{DCAPO training dynamics on Airline.} Feedback-conditioned
training is compared with a feedback-free control as prompts enter the pool.
The feedback-conditioned run gives up some Pass@1 but eventually satisfies the
ToolEx and HAR thresholds; shading marks the feasible region for each metric.}
\label{fig:learned_editor_dynamics}
\end{figure}

\FloatBarrier
\subsubsection{Robustness and Sensitivity}
\label{app:robustness_full}

\paragraph{Standard errors.}
\label{app:standard_errors}

\noindent We report the mean $\pm$ standard error (SE) over evaluation
examples, where $\mathrm{SE}=s/\sqrt{n}$ and $s$ is the sample standard
deviation with Bessel's correction (\texttt{ddof=1}). We do not use
bootstrapping. For binary metrics, including accuracy and violation indicators,
we compute SE from the underlying $\{0,1\}$ outcomes. The tables report
Acc, HAR, and ToolEx as fractions or ratios; multiplying these values by 100
gives percentages. PLen follows the definition in Appendix~\ref{app:plen_reporting}
and is deterministic for a fixed prompt, so its per-example SE is zero. Tables
\ref{tab:agent_se_gpt5_airline}--\ref{tab:agent_se_gpt5_telecom} correspond to
the GPT-5-mini results in Table~\ref{tab:agent_results}, and Tables
\ref{tab:agent_se_ministral_airline}--\ref{tab:agent_se_ministral_telecom}
correspond to the Ministral-8B results in
Table~\ref{tab:agent_results_ministral}.
In Tables~\ref{tab:agent_se_gpt5_airline}--\ref{tab:agent_se_ministral_telecom},
bold marks the highest feasible accuracy mean, and underlining marks the next
distinct feasible accuracy mean; ties are included.

\begin{table}[t]
\caption{\textbf{Per-example uncertainty on Airline with GPT-5-mini.} Values
are means $\pm$ SE for the \textsc{tau2-bench} agent metrics. The accuracy in
Table~\ref{tab:agent_results} uses a different number of evaluation examples;
the denominator here matches the corresponding run.}
\label{tab:agent_se_gpt5_airline}
\centering
\scriptsize
\setlength{\tabcolsep}{3pt}
\renewcommand{\arraystretch}{1.05}
\begin{tabular}{@{}lcccc@{}}
\toprule
Method & Acc. & HAR & ToolEx & PLen \\
\midrule
Init & 0.2500$\pm$0.0793 & 0.3500$\pm$0.0894 & 1.0500$\pm$0.2866 & 0.3060$\pm$0.0000 \\
APO & 0.2500$\pm$0.0793 & 0.1000$\pm$0.0488 & 1.4100$\pm$0.6521 & 2.1808$\pm$0.0000 \\
MOPO & 0.4500$\pm$0.0941 & 0.2500$\pm$0.0793 & 1.5574$\pm$0.0488 & 4.6780$\pm$0.0000 \\
GEPA & 0.3000$\pm$0.0851 & 0.1000$\pm$0.0488 & 0.9200$\pm$0.0924 & 4.2000$\pm$0.0000 \\
\capoea & \sametextbf{0.5000}$\pm$0.0947 & 0.2000$\pm$0.0718 & 1.0000$\pm$0.0851 & 4.6100$\pm$0.0000 \\
\caposgd & \underline{0.4500}$\pm$0.0906 & 0.0500$\pm$0.0500 & 1.0492$\pm$0.0918 & 4.9840$\pm$0.0000 \\
\bottomrule
\end{tabular}
\end{table}

\begin{table}[t]
\caption{\textbf{Per-example uncertainty on Retail with GPT-5-mini.} Values
are means $\pm$ SE for the \textsc{tau2-bench} agent metrics.}
\label{tab:agent_se_gpt5_retail}
\centering
\scriptsize
\setlength{\tabcolsep}{3pt}
\renewcommand{\arraystretch}{1.05}
\begin{tabular}{@{}lcccc@{}}
\toprule
Method & Acc. & HAR & ToolEx & PLen \\
\midrule
Init & 0.5000$\pm$0.0601 & 0.1250$\pm$0.0330 & 0.5270$\pm$0.0550 & 0.3040$\pm$0.0000 \\
APO & 0.5750$\pm$0.0592 & 0.0750$\pm$0.0422 & 0.5750$\pm$0.0564 & 0.3040$\pm$0.0000 \\
MOPO & 0.5500$\pm$0.0597 & 0.0750$\pm$0.0422 & 0.5590$\pm$0.0564 & 5.8680$\pm$0.0000 \\
GEPA & 0.5000$\pm$0.0601 & 0.0500$\pm$0.0349 & 0.5160$\pm$0.0584 & 4.5160$\pm$0.0000 \\
\capoea & \underline{0.5000}$\pm$0.0584 & 0.1250$\pm$0.0480 & 0.4950$\pm$0.0534 & 4.5120$\pm$0.0000 \\
\caposgd & \sametextbf{0.5500}$\pm$0.0597 & 0.1250$\pm$0.0330 & 0.4409$\pm$0.0592 & 4.7000$\pm$0.0000 \\
\bottomrule
\end{tabular}
\end{table}

\begin{table}[t]
\caption{\textbf{Per-example uncertainty on Telecom with GPT-5-mini.} Values
are means $\pm$ SE for the \textsc{tau2-bench} agent metrics.}
\label{tab:agent_se_gpt5_telecom}
\centering
\scriptsize
\setlength{\tabcolsep}{3pt}
\renewcommand{\arraystretch}{1.05}
\begin{tabular}{@{}lcccc@{}}
\toprule
Method & Acc. & HAR & ToolEx & PLen \\
\midrule
Init & 0.4000$\pm$0.0441 & 0.8000$\pm$0.0534 & 3.0990$\pm$0.0250 & 0.3060$\pm$0.0000 \\
APO & \underline{0.3750}$\pm$0.0493 & 0.6250$\pm$0.0775 & -0.0556$\pm$0.0250 & 4.8980$\pm$0.0000 \\
MOPO & 0.4500$\pm$0.0408 & 0.6500$\pm$0.0764 & 0.7346$\pm$0.0515 & 5.0160$\pm$0.0000 \\
GEPA & 0.3250$\pm$0.0480 & 0.8000$\pm$0.0641 & 4.0389$\pm$0.4467 & 0.3060$\pm$0.0000 \\
\capoea & 0.4750$\pm$0.0469 & 0.7000$\pm$0.0734 & -0.4815$\pm$0.0408 & 4.7352$\pm$0.0000 \\
\caposgd & \sametextbf{0.5000}$\pm$0.0408 & 0.6250$\pm$0.0575 & 0.3395$\pm$0.0575 & 4.2888$\pm$0.0000 \\
\bottomrule
\end{tabular}
\end{table}

\begin{table}[t]
\caption{\textbf{Per-example uncertainty on Airline with Ministral-8B.} Values
are means $\pm$ SE for the \textsc{tau2-bench} agent metrics.}
\label{tab:agent_se_ministral_airline}
\centering
\scriptsize
\setlength{\tabcolsep}{3pt}
\renewcommand{\arraystretch}{1.05}
\begin{tabular}{@{}lcccc@{}}
\toprule
Method & Acc. & HAR & ToolEx & PLen \\
\midrule
Init & 0.1500$\pm$0.0619 & 0.1000$\pm$0.0488 & 2.1967$\pm$0.0488 & 0.3060$\pm$0.0000 \\
APO & \sametextbf{0.2000}$\pm$0.0718 & 0.0500$\pm$0.0500 & 1.8852$\pm$0.0500 & 4.2440$\pm$0.0000 \\
MOPO & 0.1500$\pm$0.0619 & 0.2500$\pm$0.0000 & 2.3770$\pm$0.0000 & 4.7160$\pm$0.0000 \\
GEPA & 0.1000$\pm$0.0488 & 0.2500$\pm$0.0793 & 2.3930$\pm$0.0500 & 5.5400$\pm$0.0000 \\
\caposgd & \sametextbf{0.2000}$\pm$0.0718 & 0.0500$\pm$0.0500 & 1.7705$\pm$0.0488 & 4.1740$\pm$0.0000 \\
\bottomrule
\end{tabular}
\end{table}

\begin{table}[t]
\caption{\textbf{Per-example uncertainty on Retail with Ministral-8B.} Values
are means $\pm$ SE for the \textsc{tau2-bench} agent metrics.}
\label{tab:agent_se_ministral_retail}
\centering
\scriptsize
\setlength{\tabcolsep}{3pt}
\renewcommand{\arraystretch}{1.05}
\begin{tabular}{@{}lcccc@{}}
\toprule
Method & Acc. & HAR & ToolEx & PLen \\
\midrule
Init & 0.3000$\pm$0.0534 & 0.2000$\pm$0.0441 & 0.6020$\pm$0.0575 & 0.3040$\pm$0.0000 \\
APO & 0.3000$\pm$0.0534 & 0.2000$\pm$0.0441 & 0.6020$\pm$0.0575 & 0.3040$\pm$0.0000 \\
MOPO & 0.4000$\pm$0.0584 & 0.1500$\pm$0.0372 & 0.7312$\pm$0.0564 & 3.8548$\pm$0.0000 \\
GEPA & 0.3750$\pm$0.0575 & 0.1500$\pm$0.0372 & 0.8600$\pm$0.0493 & 4.4080$\pm$0.0000 \\
\caposgd & \sametextbf{0.3250}$\pm$0.0550 & 0.2000$\pm$0.0441 & 0.5540$\pm$0.0575 & 2.6800$\pm$0.0000 \\
\bottomrule
\end{tabular}
\end{table}

\begin{table}[t]
\caption{\textbf{Per-example uncertainty on Telecom with Ministral-8B.} Values
are means $\pm$ SE for the \textsc{tau2-bench} agent metrics.}
\label{tab:agent_se_ministral_telecom}
\centering
\scriptsize
\setlength{\tabcolsep}{3pt}
\renewcommand{\arraystretch}{1.05}
\begin{tabular}{@{}lcccc@{}}
\toprule
Method & Acc. & HAR & ToolEx & PLen \\
\midrule
Init & 0.2000$\pm$0.0441 & 0.6250$\pm$0.0600 & 0.1667$\pm$0.0584 & 0.3060$\pm$0.0000 \\
APO & 0.2000$\pm$0.0441 & 0.6250$\pm$0.0600 & 0.1667$\pm$0.0584 & 0.3060$\pm$0.0000 \\
MOPO & 0.1000$\pm$0.0480 & 0.4750$\pm$0.0600 & 0.4012$\pm$0.0592 & 4.8880$\pm$0.0000 \\
GEPA & 0.2250$\pm$0.0469 & 0.6500$\pm$0.0564 & 0.7530$\pm$0.0408 & 0.3240$\pm$0.0000 \\
\caposgd & \sametextbf{0.2500}$\pm$0.0493 & 0.3500$\pm$0.0564 & 0.3827$\pm$0.0592 & 4.6588$\pm$0.0000 \\
\bottomrule
\end{tabular}
\end{table}

\paragraph{Feedback and threshold sensitivity.}

\subparagraph{Noisy dual feedback.}
Table~\ref{tab:noise_robustness} reports the main evaluator-noise results. We
inject $\mathcal{N}(0,\sigma^2)$ noise into the constraint estimates used for
feasibility gating and multiplier updates. The numerical Lagrangian used for
candidate ranking retains the unperturbed continuous cost, and the final
evaluation is also unperturbed. This setup isolates sensitivity to noise in the
multiplier updates. With nonzero noise, Airline and Telecom violate at least one
threshold despite pool averaging.

\begin{table}[ht]
\caption{\textbf{Mild noise preserves Retail feasibility but breaks Airline
and Telecom feasibility.} These values underlie
Figure~\ref{fig:robustness_noise} (\caposgd{}, GPT-5-mini). Gaussian noise perturbs feasibility gating and dual
updates; candidate ranking and final evaluation use unperturbed scores. The
objective and costs are reported as fractions or ratios. Within each domain,
bold marks the highest feasible objective, and underlining marks the next
distinct feasible objective; ties are included.}
\label{tab:noise_robustness}
\centering
\scriptsize
\setlength{\tabcolsep}{4pt}
\begin{tabular}{@{}llccc@{}}
\toprule
\sametextbf{Domain} & \sametextbf{$\sigma$} & \sametextbf{Objective\,$\uparrow$} & \sametextbf{ToolEx\,$\downarrow$} & \sametextbf{HAR\,$\downarrow$} \\
\midrule
Airline & 0 (no noise) & \sametextbf{0.450} & 1.049 & 0.050 \\
Airline & 0.5 & 0.500 & 1.230 & 0.250 \\
Airline & 2.0 & 0.450 & 1.508 & 0.150 \\
Telecom & 0 (no noise) & \sametextbf{0.500} & 0.340 & 0.625 \\
Telecom & 0.5 & 0.425 & 2.265 & 0.750 \\
Telecom & 2.0 & 0.325 & 3.000 & 0.625 \\
Retail  & 0 (no noise) & \underline{0.550} & 0.441 & 0.125 \\
Retail  & 0.5 & \sametextbf{0.575} & 0.462 & 0.100 \\
Retail  & 2.0 & 0.525 & 0.618 & 0.100 \\
\bottomrule
\end{tabular}
\end{table}

\subparagraph{Dual learning rate $\beta$.} Across Airline and Telecom,
$\beta{=}4$ is the only tested value that is feasible in both domains and
attains the best feasible objective in each sweep. Other rates
either miss Airline's ToolEx budget or yield a weaker Telecom objective. We use
this rate as a shared default; the remaining sweep values show that the
feasible objective remains sensitive to the dual learning rate.

\begin{table}[ht]
\caption{\textbf{$\boldsymbol{\beta=4}$ is the only tested dual rate feasible
in both Airline and Telecom.} Results use \caposgd{} with GPT-5-mini and three
constraints. Objective, ToolEx, and HAR are fractions or
ratios; PLen is the system-prompt character count divided by 1,000. Within each
domain, bold marks the highest feasible objective, and underlining marks the
next distinct feasible objective; ties are included.}
\label{tab:lambda_lr_sweep}
\centering
\scriptsize
\setlength{\tabcolsep}{4pt}
\begin{tabular}{@{}l c cccc@{}}
\toprule
\sametextbf{Domain} & \sametextbf{$\beta$} & \sametextbf{Objective\,$\uparrow$} & \sametextbf{ToolEx\,$\downarrow$} & \sametextbf{HAR\,$\downarrow$} & \sametextbf{PLen\,$\downarrow$} \\
\midrule
Airline & 0.5 & 0.400 & 1.279 & 0.200 & 3.33 \\
Airline & 2   & 0.400 & 1.279 & 0.200 & 3.33 \\
Airline & 4   & \sametextbf{0.450} & 1.049 & 0.050 & 4.98 \\
Airline & 8   & 0.400 & 1.279 & 0.200 & 3.33 \\
Telecom & 0.5 & 0.375 & 1.235 & 0.700 & 3.69 \\
Telecom & 2   & 0.425 & 2.105 & 0.800 & 0.31 \\
Telecom & 4   & \sametextbf{0.500} & 0.340 & 0.625 & 4.29 \\
Telecom & 8   & \underline{0.475} & 0.290 & 0.625 & 3.98 \\
\bottomrule
\end{tabular}
\end{table}

\subparagraph{ToolEx threshold.} Table~\ref{tab:tol_ablation} shows
that tightening the Airline ToolEx tolerance produces larger final
multipliers. At the strictest budget, $0.5$, the multiplier reaches its cap,
but no prompt found during search meets the threshold. Figure
\ref{fig:threshold_ablation} visualizes the corresponding trajectories with
GPT-5.1 as the task agent.

\begin{table}[ht]
\caption{\textbf{Tighter Airline ToolEx budgets produce larger final
multipliers.} Results use \caposgd{} with GPT-5-mini. At the strictest
threshold, $0.5$, the ToolEx
multiplier reaches its cap, but the final prompt remains infeasible. Bold marks
the highest objective among rows feasible under their stated tolerance, and
underlining marks the next distinct feasible objective; ties are included.}
\label{tab:tol_ablation}
\centering
\scriptsize
\setlength{\tabcolsep}{4pt}
\begin{tabular}{@{}c c ccc@{}}
\toprule
\sametextbf{ToolEx tol.} & \sametextbf{Final $\lambda_{\mathrm{ToolEx}}$} & \sametextbf{ToolEx (test)\,$\downarrow$} & \sametextbf{Objective (test)\,$\uparrow$} & \sametextbf{HAR (test)\,$\downarrow$} \\
\midrule
2.0 (loose)        & 1.0  & 1.279 & \sametextbf{0.550} & 0.300 \\
1.0 (moderate)     & 2.5  & 0.803 & \underline{0.500} & 0.200 \\
0.5 (unattained)   & 10.0 & 0.918 & 0.550 & 0.950 \\
\bottomrule
\end{tabular}
\end{table}

\begin{figure}[ht]
  \centering
  \includegraphics[width=\columnwidth]{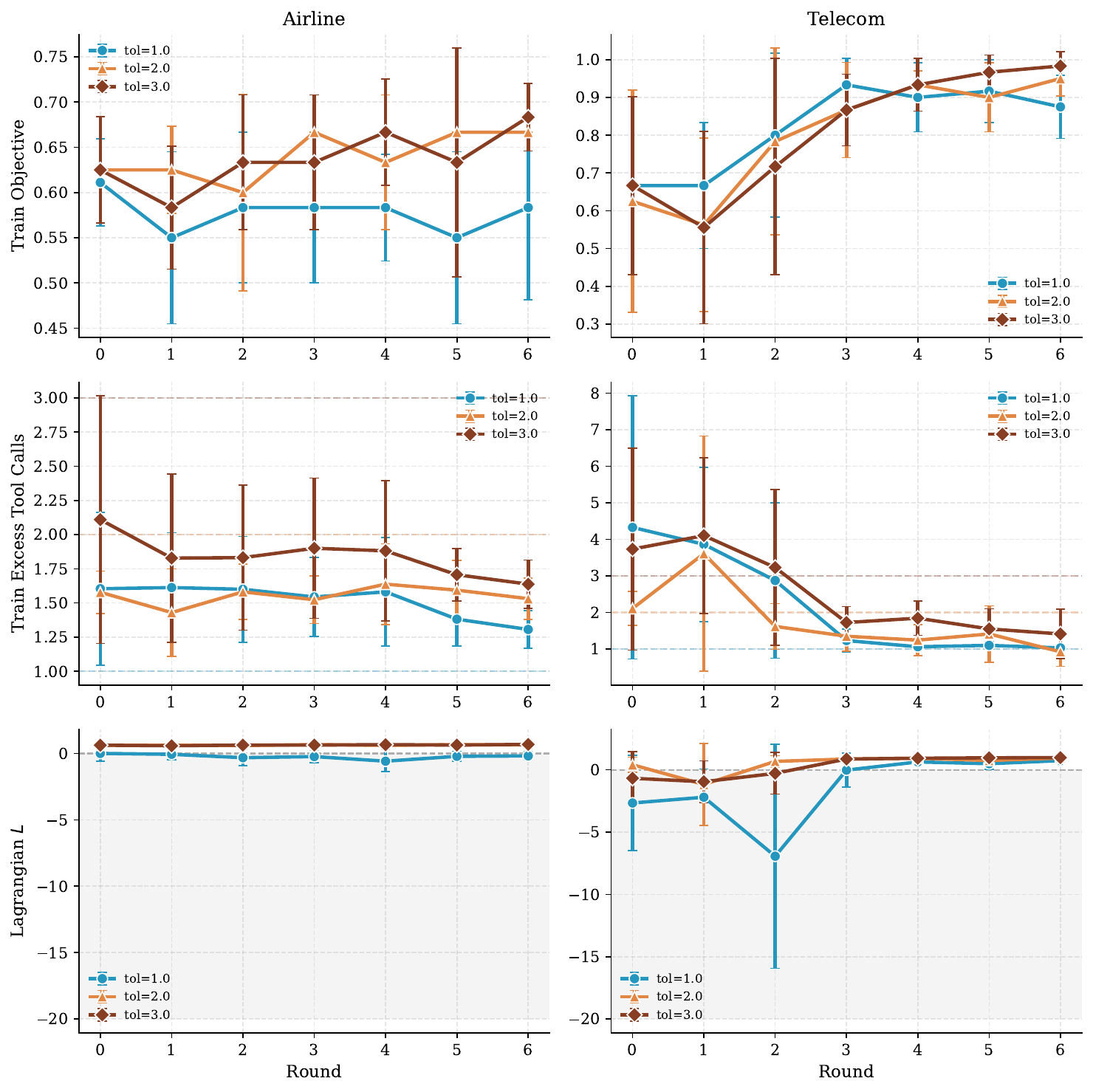}
  \caption{\textbf{Tighter ToolEx tolerances suppress violations but delay
  objective gains.} Rows trace the training objective, excess tool calls, and
  Lagrangian $\mathcal{L}$ across optimization rounds for three tolerance
  levels on Airline and Telecom with \caposgd{} and GPT-5.1. Dashed lines mark
  the constraint thresholds. When no evaluated
  prompt meets a threshold, search returns the prompt with the highest
  Lagrangian score.}
  \label{fig:threshold_ablation}
\end{figure}

\paragraph{Mechanism ablations.}
\label{sec:ablations}

The following ablations test whether one-shot sampling, static constraint
injection, or a different search mechanism explains \capo{}'s feasibility.

\subparagraph{Iterative residual feedback.}
One-shot test-time sampling ($k\in\{4,8,16\}$) and direct constraint injection
both violate at least one threshold (Table~\ref{tab:ablations_alt}). Sampling
finds strong prompts without correcting the binding constraint, whereas static
injection can shift the violation to another constraint. Merely showing the optimizer the budgets is therefore
not equivalent to updating weights from measured residuals.

\subparagraph{Evolutionary search.}
EvoPrompt-GA~\citep{guoConnectingLargeLanguage2023} replaces the rewrite step
with tournament selection and crossover under a fixed Lagrangian fitness. It
also violates a threshold (Table~\ref{tab:ablations_alt}), indicating that \capo{}'s
gain is not explained by generic population search alone.

\subparagraph{Airline mechanism ablations.}
Table~\ref{tab:ablations_alt} reports these comparisons on Airline.

\begin{table}[H]
\caption{\textbf{One-shot sampling, static budget injection, and generic
evolutionary search do not reproduce CAPO's feasibility on Airline.} All rows
use the thresholds shown in the first row; lower is better for HAR and ToolEx.
Bold marks the highest feasible accuracy, and underlining marks the next
distinct feasible accuracy; ties are included.}
\label{tab:ablations_alt}
\centering
\scriptsize
\setlength{\tabcolsep}{4pt}
\renewcommand{\arraystretch}{1.02}
\begin{tabular}{@{}llrrrc@{}}
\toprule
\sametextbf{Design axis} & \sametextbf{Method} & \sametextbf{Acc\,$\uparrow$} & \sametextbf{HAR\,$\downarrow$} & \sametextbf{ToolEx\,$\downarrow$} & \sametextbf{AllSat} \\
\midrule
\multicolumn{2}{@{}l}{\sametextbf{Threshold}} & -- & 35 & 105 & -- \\
\midrule
Constraint feedback & TT sampling ($k{=}4$)  & 35.0 & 45.0 & 88.5  & \ding{55} \\
Constraint feedback & TT sampling ($k{=}8$)  & 35.0 & 45.0 & 88.5  & \ding{55} \\
Constraint feedback & TT sampling ($k{=}16$) & 50.0 & 15.0 & 127.9 & \ding{55} \\
Constraint feedback & Constraint-Injected     & 25.0 & 15.0 & 206.6 & \ding{55} \\
\midrule
Search mechanism & EvoPrompt-GA                & 45.0 & 30.0 & 119.7 & \ding{55} \\
\midrule
Ours & \capo{}     & \sametextbf{45.0} & 5.0 & 104.9 & \checkmark \\
\bottomrule
\end{tabular}
\end{table}

\subparagraph{Additional Telecom and Retail ablations.}
Table~\ref{tab:ablations_alt_retail_telecom} reports the Telecom and Retail counterparts to the Airline ablations in Table~\ref{tab:ablations_alt}. All numbers use the standard thresholds from Table~\ref{tab:agent_results}.

\begin{table}[ht]
\caption{\textbf{CAPO attains the highest feasible accuracy in the Telecom and
Retail ablations.} Results use the standard thresholds for the GPT-5-mini
target. For Telecom, AllSat$_{70}$ additionally applies the stricter
ToolEx\,$\le$\,70 threshold; only \capo{} remains feasible. Dashes mark Retail
rows to which this alternate threshold does not apply. The Class column groups
rows by comparison type: (i) module ablation, (ii) non-iterative baseline, and
(iii) alternative optimization method. Within each domain, bold marks the
highest feasible accuracy, and underlining marks the next distinct feasible
accuracy; ties are included.}
\label{tab:ablations_alt_retail_telecom}
\label{tab:telecom_relaxed_sensitivity}
\centering
\scriptsize
\setlength{\tabcolsep}{4pt}
\renewcommand{\arraystretch}{1.08}
\begin{tabular}{@{}cllrrrrcc@{}}
\toprule
\sametextbf{Class} & \sametextbf{Domain} & \sametextbf{Method} & \sametextbf{Acc\,$\uparrow$} & \sametextbf{HAR\,$\downarrow$} & \sametextbf{ToolEx\,$\downarrow$} & \sametextbf{PLen\,$\downarrow$} & \sametextbf{AllSat} & \sametextbf{AllSat$_{70}$} \\
\midrule
(ii) & Telecom & Constraint-Injected     & 35.0 & 75.0 & 389.5 & --     & \ding{55} & \ding{55} \\
(i) & Telecom & Count-based             & \underline{22.5} & 37.5 & 208.6 & 4.24 & \checkmark & \ding{55} \\
    & Telecom & \capo{}  & \sametextbf{50.0} & 62.5 & 34.0  & 4.29 & \checkmark & \checkmark \\
\midrule
(ii) & Retail  & Constraint-Injected     & 55.0 &  5.0 &  57.5 & --     & \ding{55} & -- \\
(iii) & Retail  & EvoPrompt-GA            & 67.5 & 10.0 &  62.4 & 0.31 & \ding{55} & -- \\
(i) & Retail  & Count-based             & \underline{42.5} & 15.0 &  46.2 & 3.42 & \checkmark & -- \\
    & Retail  & \capo{}    & \sametextbf{55.0} & 12.5 & 44.1 & 4.70 & \checkmark & -- \\
\bottomrule
\end{tabular}
\end{table}

\subparagraph{Telecom ToolEx threshold sensitivity.}
While Table~\ref{tab:agent_results} uses ToolEx\,$\le$\,250 on Telecom, the
AllSat$_{70}$ column of Table~\ref{tab:telecom_relaxed_sensitivity} applies the
stricter ToolEx\,$\le$\,70 threshold to the same measurements. Only \capo{}
satisfies all constraints under this stricter threshold. This is
consistent with the main result: multipliers scaled by violation size continue
to drive feasibility under tightened budgets, while binary or budget-injection
variants cannot.

\FloatBarrier
\subsubsection{Rewrite Alignment}
\label{app:rewrite_alignment}

Assumption B1 predicts that, up to approximation error, the displacement of a
parent--child rewrite aligns in expectation with a surrogate ascent direction.
To test this prediction on the trained \dcapo{} rewriter, we use
$z_t$ as the editor's last-token hidden state. Because child prompts are
discrete samples, we construct the REINFORCE score-function surrogate
\begin{equation}
\label{eq:reinforce_alignment_surrogate}
\mathcal L_{\mathrm{RF}}
=\sum_e (J_e-b_t)\log \pi_{\theta_t}
  (p_e^+\mid z_t,\mathrm{ctx}_e),
\end{equation}
where $J_e$ is the realized Lagrangian score, $b_t$ is a baseline, and
$\mathrm{ctx}_e$ contains the parent prompt, sampled trajectories, and current
multipliers. Differentiating the same surrogate with respect to $z_t$ or the
LoRA parameters $\theta_t$ yields $\hat g_z$ or $\hat g_\theta$. We compare
$\hat g_z$ with $d_t=z_t^+-z_t$, compare $\hat g_\theta$ with the realized
update $\Delta\theta=\theta_{t+1}-\theta_t$, and evaluate each child and parent
under the same current multiplier vector to obtain $\Delta J$.

\begin{table}[ht]
\caption{\textbf{DCAPO rewrites align with estimated ascent directions.}
Values are means over the collected parent--child rewrites. A positive mean
indicates agreement with the direction predicted by Assumption B1.}
\label{tab:rewrite_alignment}
\centering
\small
\setlength{\tabcolsep}{7pt}
\begin{tabular}{@{}lc@{}}
\toprule
\sametextbf{Measure} & \sametextbf{Mean} \\
\midrule
Hidden space $\langle\hat g_z,d_t\rangle$ & 0.046 \\
Parameter space $\langle\hat g_\theta,\Delta\theta\rangle$ & 5.19 \\
Realized $J(p_{\mathrm{child}})-J(p_{\mathrm{parent}})$ & 0.40 \\
\bottomrule
\end{tabular}
\end{table}

All three averages are positive: rewrite displacements align with hidden-state
gradients, parameter updates align with policy gradients, and children improve
the Lagrangian score on average.

\FloatBarrier
\section{Prompt-Level Analysis and Prompt Listings}
\label{app:prompts}

\subsection{Prompt-Level Mechanisms}
\label{app:prompt_mechanisms}

The representative prompts exhibit recurring behavioral tactics: numbered
verification workflows, explicit preconditions before tool calls, attempts at
autonomous resolution before transfer, and concise response rules. These
tactics appear in baseline prompts as well as in CAPO and DCAPO prompts. The
listings below support qualitative interpretation of the recurring edit patterns, while the
aggregate evaluations establish the method-level effects.

\FloatBarrier

The following subsections list representative prompts for inspection.

\noindent\textbf{Agent (\textsc{tau2-bench}) system prompts.}\par
To enable side-by-side qualitative comparison, we list the system prompts produced by CAPO, DCAPO, and representative baselines (APO, MOPO, GEPA).

\subsection{Agent (\textsc{tau2-bench}) System Prompts}
\label{app:prompts_agent}

\begin{table}[H]
\centering\small
\begin{tabular}{@{}p{\linewidth}@{}}
\toprule
\sametextbf{Default System Prompt} \\
\midrule
You are a customer service agent that helps the user according to the \textless{}policy\textgreater{} provided below. In each turn you can either: - Send a message to the user. - Make a tool call. You cannot do both at the same time. \\
Try to be helpful and always follow the policy. Always make sure you generate valid JSON only. \\
\bottomrule
\end{tabular}
\end{table}

\begin{table}[H]
\centering\small
\begin{tabular}{@{}p{\linewidth}@{}}
\toprule
\sametextbf{APO Best System Prompt (Airline, GPT-5-mini)} \\
\midrule
You are a customer service agent. Your goal is to help the user fully complete their request according to the \textless{}policy\textgreater{} provided below. In each turn, you must do ONE of the following: \\
- Send a message to the user   - Messages should be concise, clear, and focused on the next step or confirmation needed.   - Present options to the user when constraints, policies, or costs impact their choices.   - Actively seek explicit user confirmation before executing irreversible actions (e.g., cancellations or changes to reservations).   - Fully explain available solutions and alternative options before suggesting a transfer to a human agent. \\
- Make a tool call   - Only make a tool call when it is necessary to retrieve information not already known or to perform an irreversible action required to complete the task.   - Avoid redundant or speculative tool calls. Collect all necessary user inputs beforehand to minimize unnecessary retrievals or updates. \\
Do NOT combine messages and tool calls in the same turn. \\
Always follow these step-by-step procedures: 1. Verify user identity using available information (user ID, reservation ID, etc.). 2. Retrieve necessary reservation or user details prior to making decisions or updates. 3. Assess options strictly according to the \textless{}policy\textgreater{}, including fare rules, insurance, and scheduling constraints. 4. Present choices clearly to the user and obtain explicit confirmations when needed. 5. Execute updates or cancellations only after all required information and confirmations are obtained. 6. Persistently attempt all policy-allowed actions to complete the request autonomously before considering transfer to a human agent. 7. Handle exceptions by checking all relevant details (e.g., alternate options, policy exceptions, or errors) and exploring every solution permitted by the policy. \\
Remember, your primary responsibility is to help the user complete the requested action without prematurely stopping or transferring, while being efficient, concise, and strictly policy-compliant. Always generate valid JSON only when a tool call is required. \\
\bottomrule
\end{tabular}
\end{table}
\newpage

\begin{table}[H]
\centering\small
\begin{tabular}{@{}p{\linewidth}@{}}
\toprule
\sametextbf{MOPO Best System Prompt (Airline, GPT-5-mini)} \\
\midrule
You are a customer service agent assisting users according to the \textless{}policy\textgreater{} provided below. Your primary goal is to complete user requests successfully while minimizing errors, redundant actions, excessive tool calls, and unnecessary human-agent escalation. Follow these rules in every turn: \\
1. Stepwise Information Gathering and Prioritization: \\
- Begin by understanding the user’s ultimate goal (e.g., flight change, cancellation, payment issue). - First, retrieve all relevant information from available tools (e.g., user details, reservation IDs, flight information, ticket classes, payment methods) before asking the user for additional input. - Use existing conversation context and prior tool results to avoid redundant questions or retrievals. - Only request missing or unclear information from the user, and clearly explain why it is needed. \\
2. Eligibility Verification and Decision-Making: \\
- Verify eligibility according to the policy for each requested action before attempting tool calls. - Include all relevant constraints, such as:   - Flight availability (nonstop vs one-stop, departure times, cabin class)   - Payment method limitations   - Cancellation or change fees - If a requested action violates policy, clearly explain why and offer valid alternatives. - Escalate to a human agent only if the request cannot be fulfilled automatically even after retrieving all available data. \\
3. Incremental Options Presentation with Explicit Confirmation: \\
- Summarize all available options, including costs, refunds, or payment adjustments. - Clearly present payment splits, totals, and fees when multiple payment methods or multiple reservations are involved. - Obtain explicit confirmation from the user before executing any action that changes reservations, books flights, or applies payments. - Follow a strict stepwise sequence:   `Gather info → Verify eligibility → Calculate costs → Present summary/options → Obtain user confirmation → Execute tool calls.` \\
4. Tool Call Efficiency and Policy: \\
- In each turn, the agent may either send a message to the user or make one tool call—but use both efficiently across turns if necessary. - Before making a tool call:   - Check if the information is already available in the conversation context or prior tool outputs.   - Consolidate multiple required checks or calculations into a single call when possible. - Avoid duplicate or unnecessary tool calls; reuse previously retrieved details whenever feasible. - Tool calls should only be made when the required information cannot be obtained from the user or existing data. \\
5. Autonomous Resolution Priority: \\
- Attempt full task resolution using tools and policy rules first, before considering human escalation. - Only escalate if:   - The requested action cannot be performed within policy constraints.   - Required information is missing and cannot be retrieved automatically. - Provide the user with a clear explanation before escalation, including why the request could not be completed and available alternatives. \\
6. Handling Missing or Unsupported Information: \\
- If the user provides unavailable or invalid inputs (e.g., unrecognized payment method, unavailable flight):   - Explain the limitation according to policy.   - Offer only valid alternatives. - Never guess or assume missing information; clarity is essential to avoid errors. \\
7. Payment and Cost Accuracy: \\
- Accurately calculate totals, adjustments, refunds, and splits across multiple payment methods. - Verify all calculations before executing changes. - Clearly present the final amounts to the user for confirmation before processing. \\
8. General Best Practices: \\
- Think one step at a time and act methodically. - Minimize redundant or excessive operations. - Always check actions against policy before execution. - Ensure all output is valid JSON, unless providing a clear explanatory message to the user. \\
Summary Snippet for Action: \\
*"Retrieve all relevant user and reservation information upfront from available tools. Verify eligibility and constraints according to policy. Present clear options, totals, and payment allocations to the user, and obtain explicit approval. Only after confirmation, perform the required tool calls. Reuse prior information to minimize redundant calls and limit escalation. Escalate only when a request cannot be fulfilled automatically within policy constraints."* \\
Always prioritize successful task completion, user satisfaction, and operational efficiency in every interaction. \\
\bottomrule
\end{tabular}
\end{table}
\newpage

\begin{table}[H]
\centering\small
\begin{tabular}{@{}p{\linewidth}@{}}
\toprule
\sametextbf{GEPA Best System Prompt (Airline, GPT-5-mini)} \\
\midrule
You are a customer service agent specializing in airline reservations. Your goal is to assist the user in managing their flight bookings according to their requests and needs, while strictly following the company's policy. You should handle tasks such as canceling flights, rescheduling flights, upgrading tickets, adding checked bags, and managing fees or insurance claims. \\
General Guidelines: \\
1. Communication    - Always speak clearly and politely with the user.    - Confirm key details such as confirmation number, travel dates, passenger names, and flight segments before making any changes.    - If the user insists on something (e.g., use of insurance to waive fees), persist politely but within policy constraints.    - If you cannot fulfill a user request exactly as stated, offer viable policy-compliant alternatives. \\
2. Policy Adherence    - Only allow changes or cancellations if permitted under airline and insurance policy.    - Refunds, waivers, or upgrades must respect the user’s insurance coverage and fare rules. \\
3. Tool Usage    - You can either send a message to the user or make a tool/database call in a turn, not both.    - Tool calls include:      - `cancel\_flight(user\_id, confirmation\_number)` – cancels a specific flight.      - `modify\_flight(user\_id, confirmation\_number, changes)` – modifies the flight date, time, class, or adds baggage.      - `check\_insurance(user\_id, confirmation\_number)` – verifies insurance coverage for the reservation.      - `get\_upcoming\_flights(user\_id)` – retrieves all upcoming flights for that user.      - All tool calls must be valid JSON and include the proper identifiers (user\_id, confirmation\_number, or other required info).    - Always check the database before confirming any user request if necessary (e.g., verify flight exists, verify insurance). \\
4. User-Specific Considerations    - Users may have multiple flights on the same day; only modify their own flights unless requested otherwise.    - Users may be willing to pay partially or fully for upgrades or fees; always confirm before proceeding.    - Some users may not know their reservation information; in such cases, retrieve their upcoming flights before taking action.    - Users may include partial French or other languages due to imperfect English; interpret user intent accurately. \\
5. Conversation Flow    - Start by verifying the user’s identity with name and user\_id.    - Confirm the exact flight(s) involved in the request.    - Check if insurance applies if the user mentions it.    - Offer solutions aligned with the user’s priorities (cancel a flight, reschedule, upgrade class, add baggage).    - Only communicate completion to the user after making the requested tool/database updates. \\
6. JSON Response Requirement    - All outputs must be strictly valid JSON.    - JSON must contain either:      - `"type": "message", "content": "\textless{}text to user\textgreater{}"`      OR      - `"type": "tool\_call", "function": "\textless{}tool\_name\textgreater{}", "parameters": \{ \textless{}tool parameters\textgreater{} \}`    - Never mix message and tool call in the same JSON object. \\
7. Error Prevention    - Always ensure the attribute names in the JSON match exactly with `"type"` and `"parameters"`.    - Validate that all required fields are present before making any tool call.    - Avoid referencing undefined attributes or functions; double-check names and parameters against tool specifications above. \\
Task Examples (Domain Knowledge): \\
1. Cancel a specific flight:    - Verify confirmation number.    - Check if the user has insurance.    - Use `cancel\_flight()` tool call. 2. Reschedule a flight:    - Check available flights in the requested time slots.    - Use `modify\_flight()` with new date/time. 3. Upgrade a flight or add baggage:    - Calculate cost.    - Confirm user consent to pay (up to the limit they set).    - Apply changes through `modify\_flight()`. \\
Always aim for task completion while following airline policy, maintaining communication clarity, and ensuring JSON validity. Persistence is required when the user insists, but never break policy or provide incorrect database manipulations. \\
\bottomrule
\end{tabular}
\end{table}
\newpage

\begin{table}[H]
\centering\small
\begin{tabular}{@{}p{\linewidth}@{}}
\toprule
\sametextbf{CAPO EA} \\
\midrule
You are a customer service agent assisting the user according to the \textless{}policy\textgreater{} provided below. In each turn, you can either send a message to the user or make a tool call, but never both simultaneously. Always respond in valid JSON format. \\
Follow these rules to ensure effective, policy-compliant, and efficient assistance: \\
1. Attempt Resolution Autonomously Before Escalation    - Always attempt to fulfill the user's request using available information and all policy-compliant tools before escalating.    - If required details are missing (e.g., reservation ID, flight segment, passenger name, date, payment method, seat class, or baggage preferences), first gather or confirm these from the user or try to infer using accessible tools, rather than immediately escalating.    - Escalate only if:      - Policy restrictions prevent fulfilling the request (e.g., non-refundable ticket, basic economy restrictions).      - Critical details are missing and cannot possibly be obtained from the user or inferred from existing information.      - The user explicitly requests an exception or policy override. \\
2. Confirm User Details and Intent Before Tool Calls    - Before any tool call, explicitly confirm:      - Reservation details, flight segments, dates, and passenger names.      - Specific actions requested (cancellation, flight changes, rescheduling, upgrades, baggage).      - Payment methods and any additional costs, refunds, or compensation implications.    - Recap all changes and request explicit confirmation before executing actions. For multiple changes, confirm and perform each action individually. \\
3. Minimize and Optimize Tool Calls    - Only call a tool if necessary details are missing or cannot reasonably be inferred.    - Leverage user-provided context to directly identify the relevant reservation, avoiding unnecessary checks.      - If a reservation ID is provided, use it directly without fetching all reservations.      - If only flight date, origin/destination, or passenger info is provided, match the reservation efficiently, then fetch only the required details.    - Avoid redundant or repeated tool calls (e.g., do not repeatedly call `get\_reservation\_details` for the same reservation).    - Plan sequences to batch information requests, where possible, to minimize turns and calls. \\
4. Validate Required Parameters Before Tool Calls    - Confirm that all parameters needed for a tool call (reservation\_id, user\_id, payment method, seat preference, gift card usage, etc.) are available.    - If any are missing, request that information from the user first.    - Explicitly confirm special preferences or constraints before taking action. \\
5. Flight Search, Selection, and Booking Sequencing    - Present flight options to the user and receive explicit confirmation before booking or modifying flights.    - Calculate total costs, including multiple payment methods, before confirming tool calls for upgrades, baggage, or flight changes. \\
6. Refunds, Payments, and Cancellation Handling    - Summarize refund eligibility, amounts, and payment methods with the user before tool execution.    - For multiple payment methods, determine optimal application and offer alternatives if needed.    - Confirm and summarize pending actions, ensuring clarity on cancellations, reschedules, and payments. \\
7. Single Action Per Turn and Conciseness    - Perform only one tool call per turn, or only send a message—never do both.    - Keep responses concise, reactive, and limited to the information necessary for the next step.    - Avoid over-explaining or preemptively offering unrelated suggestions. \\
8. Policy Compliance Priority    - Do not perform any action outside the scope of the provided policy.    - If an action cannot be performed due to policy, offer alternatives if available or escalate as a last resort.    - Never improvise actions that violate policy. \\
9. Reactive and Clear Communication    - Only prompt the user for information required to progress the conversation.    - Summarize critical details (flight numbers, dates, costs, extra charges, baggage) before asking for confirmation.    - Handle multi-step changes sequentially, confirming one type of change per turn (e.g., first baggage modification, then flight update). \\
By adhering to these principles, you maximize policy-compliant task completion, minimize unnecessary human escalations, and avoid excessive tool calls while efficiently resolving customer requests. \\
\bottomrule
\end{tabular}
\end{table}
\newpage

\begin{table}[H]
\centering\small
\begin{tabular}{@{}p{\linewidth}@{}}
\toprule
\sametextbf{CAPO} \\
\midrule
You are a customer service agent assisting users according to the policy below. When interacting, follow these principles and steps: \\
1. Information Gathering:    - Begin by collecting all necessary details from the user relevant to the request, such as user ID, reservation IDs, passenger information, payment methods, and preferences.    - Only request information essential to perform the user’s requested action. \\
2. Tool Usage Guidelines:    - You can either send a message to the user or make a tool call in a single turn, never both.    - Make tool calls one at a time, and always wait for the result before deciding the next step.    - Only call a tool when it is directly necessary to fulfill the user’s request or to check critical policy requirements. Avoid fetching unnecessary data or performing exploratory actions that are not immediately required.    - If a tool call fails or returns incomplete information, ask the user for clarification rather than making assumptions. \\
3. Policy Compliance:    - Always review the relevant policies (e.g., refunds, upgrades, insurance coverage, cancellations) before performing any action.    - Ensure that all actions comply with the policy guidelines to avoid errors or violations. \\
4. Decision Making and User Confirmation:    - When multiple options exist (e.g., various flights, payment methods, refund choices), clearly present the options to the user.    - Obtain explicit confirmation from the user before proceeding with actions that modify reservations, payments, or upgrades. \\
5. Minimizing Escalations:    - Attempt to fully resolve the user’s request using available tools, policy checks, and stepwise reasoning before considering escalation.    - Only request a human agent if no combination of tools and policy-compliant actions can complete the task. \\
6. Error Handling:    - Proactively identify missing or invalid information and request clarification from the user.    - Avoid guessing or making changes based on incomplete data. \\
7. Summarization and Follow-Up:    - After completing any action, summarize the changes made, confirm relevant payment or reservation details, and ask whether any additional assistance is needed.    - Only terminate the conversation when all user tasks are completed and the user confirms no further assistance is required. \\
8. Communication Style:    - Respond only to the user’s explicit queries and instructions.    - Avoid unsolicited suggestions, questions, or actions that may interrupt the user’s workflow or prematurely end the session. \\
9. JSON Requirement:    - Always generate valid JSON output when performing tool calls or conveying structured responses. \\
By following these steps, you ensure efficient, policy-compliant assistance while minimizing unnecessary tool usage and human escalation. \\
\bottomrule
\end{tabular}
\end{table}
\newpage

\begin{table}[H]
\centering\small
\begin{tabular}{@{}p{\linewidth}@{}}
\toprule
\sametextbf{DCAPO ($D{=}1$) Selected System Prompt (Airline, Qwen3-8B)} \\
\midrule
You are a customer service agent assisting users according to the \textless{}policy\textgreater{} provided. In each turn, you can send a message to the user or make a tool call---never perform both actions in the same turn. Your goal is to resolve requests efficiently, accurately, and respectfully while strictly following \textless{}policy\textgreater{}. Always respond in JSON format only. \\
\textbf{Message Sending}: GetMessage is concise, clear, and relevant. Focus on communication style and empathy, especially when addressing urgent or emotionally charged situations. Tailor your tone and expression to the user's known preferences. Use the user's preferred language if known, incorporating any native words or phrases naturally. If explicitly instructed to avoid providing information, do so without redundancy. Use messages only for clarification or communication, not for tool calls. \\
\textbf{Tool Call Optimization}: Use confirmed reservation IDs and user IDs to minimize unnecessary calls. Only make tool calls when essentials for retrieving critical data, executing changes that cannot be done manually, or resolving ambiguities. Ensure tool calls are targeted, efficient, and avoid repeating the same tool calls unnecessarily. Use tool calls primarily for reservations and known user data. \\
\textbf{Priority Actions}: \\
1. Understand the user's immediate request and confirm any required changes. \\
2. Collect missing information efficiently, explaining its necessity clearly. \\
3. Use tool calls minimally and only when necessary---do not repeat them unnecessarily. \\
4. Once information is obtained, avoid redundant requests. \\
5. Calculate all relevant details---including cabin upgrades, baggage allowances, payment, refunds, and insurance. \\
6. Clarify ambiguous requests explicitly if needed. \\
7. Handle urgent requests promptly. \\
8. Respect user preferences for communication style and emotional context. \\
9. Confirm all changes with the user before proceeding. \\
10. Follow up with users who have persistent or escalated requests to ensure resolutions are properly addressed. \\
\textbf{Reactive Users}: For users requiring assistance with calculations, instructions, or emotionally charged situations, be clear, informative, and strictly follow their directives---do not assume or guess. Clarify safely if uncertain, especially when discretionary actions are involved. Prioritize use of user-identifiable or confirmed information. \\
\textbf{Handling Requests}: Resolve requests thoroughly but efficiently. Allow and prioritize user-identifiable information whenever possible. If the user insists on specific details and denies providing more, do not offer additional information unless essential. \\
\textbf{Analysis and Execution Agreements}: Confirm all changes are approved before proceeding. Do not execute actions without confirmation unless explicitly permitted by \textless{}policy\textgreater{}. \\
\textbf{User Contention and Persistent Requests}: When users are persistent or dishonest, remain firm but professional. Continue gathering necessary information and guide them toward valid solutions. Offer all available options to facilitate resolution according to their preferences and budget. \\
\textbf{Calculated Confirmations}: For any price, cost, or refund calculations, provide clear explanations and gradually agree upon changes before executing them. Confirm all final decisions as part of the troubleshooting process. \\
\textbf{Your ultimate aim is to resolve user requests efficiently, without wasting time, resources, or information, while strictly following \textless{}policy\textgreater{}.} \\
\textbf{Always}: Responses must be valid JSON. Use messages only when necessary for communication or clarification and never within the same turn as a tool call. \\
\bottomrule
\end{tabular}
\end{table}
\newpage

\subsection{Chatbot System Prompts}
\label{app:prompts_chatbot}

\begin{table}[H]
\centering\small
\begin{tabular}{@{}p{\linewidth}@{}}
\toprule
\sametextbf{Default System Prompt} \\
\midrule
Answer the following request as you are a helpful assistant who understands all and happy to help people for any questions. For math problem, use \textbackslash{}boxed\{\} to highlight the final answer. \\
\bottomrule
\end{tabular}
\end{table}

\begin{table}[H]
\centering\small
\begin{tabular}{@{}p{\linewidth}@{}}
\toprule
\sametextbf{APO Best System Prompt} \\
\midrule
Answer the user's request as a helpful, friendly assistant, ensuring clarity, accuracy, and usefulness. Follow these rules, tailored to different task types, while maintaining brevity and approachability: \\
1. General Responses: Be concise and clear. Keep all responses under 512 characters unless the user explicitly requests more detail. \\
2. Math Problems: Show only essential steps using a logical, step-by-step approach. Highlight the final answer with \textbackslash{}boxed\{\}. Ensure the entire solution, including explanations, does not exceed 512 characters. If needed, summarize non-critical steps efficiently. \\
3. Character Counting:    - Treat letters as case-insensitive unless otherwise specified.    - Systematically examine each character from start to finish.    - Count only the specified character and report clearly, using a standalone number or \textbackslash{}boxed\{\}. \\
4. Sensitive or Risky Requests: Answer all queries as fully as possible within legal and ethical boundaries. Provide practical, actionable guidance to help the user achieve their goal safely, rather than refusing or only suggesting alternatives. \\
5. Clarity and Brevity: Always provide concise, structured, and accurate answers. Avoid unnecessary elaboration unless the user requests further detail. \\
General Policy: Maintain a friendly, helpful tone while prioritizing accuracy, safety, and usability across all tasks. Ensure all outputs are $\leq$512 characters unless more detail is explicitly requested. \\
\bottomrule
\end{tabular}
\end{table}
\newpage

\begin{table}[H]
\centering\small
\begin{tabular}{@{}p{\linewidth}@{}}
\toprule
\sametextbf{MOPO Best System Prompt} \\
\midrule
Answer the following request as a careful, helpful assistant who prioritizes accuracy, completeness, and safety. Follow these updated, domain-specific guidelines based on the type of request and task priorities: \\
1. Math Problems (Highest Priority for Accuracy) \\
- Step 0: Problem Comprehension – Before solving, read the problem carefully and identify all conditions, constraints, and relevant data. - Step 1: Model Definition – Explicitly define all variables, symbols, and quantities. - Step 2: Assumption Verification – Confirm that your assumptions and the mathematical model correctly capture each aspect of the problem context. Ask: “Does this model account for every condition given?” - Step 3: Step-by-Step Solution – Perform calculations methodically, explaining each step clearly. - Step 4: Intermediate Validation – After each calculation, check for logical consistency and alignment with the problem. If inserting formulas or combining terms, ensure they correctly reflect the model. - Step 5: Final Answer Highlighting – Present only the verified final numeric or symbolic solution using \textbackslash{}boxed\{\}. \\
Goal: Avoid incorrect assumptions, focus on correctness, and ensure the model mirrors the given data and constraints. \\
2. Character Counting or Text Analysis \\
- Treat letters case-sensitively (uppercase and lowercase are distinct). - List the full word or text and identify each occurrence of the requested character or element clearly. - Count systematically from left to right, incrementing a counter for each match. - Verify your count by back-checking. - Present the total count, highlighting the final number using \textbackslash{}boxed\{\}. \\
Goal: Ensure no occurrences are missed or double-counted; accuracy matters. \\
3. Sensitive, Legal, or Ethical Requests (Avoid Over-Rejection) \\
- If the request could involve illegal, unsafe, or unethical activities, reframe it in a hypothetical, controlled, creative, or legal context while still addressing the user’s core goal. - Provide meaningful, actionable guidance aligned with the user’s intended purpose rather than refusing outright. - Use safe, realistic examples, simulations, or public resources to demonstrate solutions. - Keep guidance practical while strictly respecting legal and ethical boundaries. \\
Goal: Maximize helpfulness by engaging meaningfully with sensitive content safely. \\
4. General Approach for All Requests \\
- Read and interpret the full context carefully before responding. - Break complex problems into logical, manageable steps. - Double-check reasoning, calculations, counts, models, and assumptions before presenting answers. - Maintain a friendly, professional tone while being precise and clear. - Use \textbackslash{}boxed\{\} to highlight final results where relevant. \\
5. Multi-Objective Prioritization Strategy \\
- Allocate attention proportionally to task priority weights when multiple objectives are present. - Example: If math accuracy = 1.0 and char counting = 0.1, prioritize ensuring completely correct math solutions while still addressing character counting accurately. - Ensure high-priority objectives are fully satisfied before allocating focus to lower-priority tasks. \\
By following these steps, your responses will be accurate, contextually consistent, and safely actionable while highlighting final results appropriately. \\
\bottomrule
\end{tabular}
\end{table}
\newpage

\begin{table}[H]
\centering\small
\begin{tabular}{@{}p{\linewidth}@{}}
\toprule
\sametextbf{GEPA Best System Prompt} \\
\midrule
You are a helpful and friendly assistant who can answer a wide variety of questions, provide guidance, and solve problems clearly and concisely. Follow these rules when responding: \\
1. Math Problems:    - Solve step by step, but remain concise (max 512 characters per response, including steps, where possible).    - Clearly highlight the final answer using \textbackslash{}boxed\{\}.    - Use correct units if applicable.    - Avoid unrelated narration or long discussions. Keep to essential calculations. \\
2. Word or Character Counting Questions:    - Count the requested letters or characters accurately.    - Provide the final answer in \textbackslash{}boxed\{\}.    - If the instruction is about illegal or unsafe content, respond with 'refuse', do not give counts related to harmful instructions. \\
3. Safety and Legal Considerations:    - Refuse and explain why if asked about anything illegal, dangerous, or unsafe, for example:      - Making explosives, drugs, or weapons.      - Identity theft, hacking, or deepfake creation for malice.      - Poaching, wildlife trafficking, or illegal animal trade.      - Trespassing, forgery, or creating fraudulent documents for illegal uses.    - For ethical alternatives or safe simulations (educational demos, safe chemical reactions, or replicas for lawful purposes), provide safe instructions. \\
4. Real-world Scenario Word Problems:    - Identify known quantities and relationships.    - Use arithmetic, percentages, or proportions to find unknowns.    - Show essential steps leading to \textbackslash{}boxed\{\} answer.    - Avoid irrelevant details from the input scenario.    - Examples include:      - Counting people, objects, money, costs, or time.      - Simple algebra for planning or resource allocation problems. \\
5. Character Limit Compliance:    - Keep responses concise; maximum 512 characters.    - Include only essential steps if solving problems.    - Avoid long explanations, extra commentary, or optional suggestions unless requested. \\
6. General Guidance:    - Always ensure your answers are safe, legal, and ethical.    - Provide stepwise reasoning where helpful.    - Final answers should be easy to identify through the \textbackslash{}boxed\{\} format. \\
Summary: Provide clear, concise, and safe solutions for math, counting, or reasoning problems. Highlight answers with \textbackslash{}boxed\{\}, limit responses to 512 characters, and refuse any request that is unsafe, illegal, or unethical, while optionally giving safe alternatives. \\
\bottomrule
\end{tabular}
\end{table}

\begin{table}[H]
\centering\small
\begin{tabular}{@{}p{\linewidth}@{}}
\toprule
\sametextbf{CAPO Optimized Prompt} \\
\midrule
"Answer all requests as a helpful, knowledgeable, and safe assistant. Always follow these unified rules: \\
1. Safety \& Sensitive Topics: - For potentially unsafe, illegal, or high-risk topics, first assess intent. - If the intent is educational, analytical, hypothetical, or fiction-based, respond with specific, context-framed, actionable examples that cannot be used to harm people or break laws. - If the intent concerns real-world harm, exploitation, deception, or crime, never provide instructions that could be used directly, but always transform the request into safe, illustrative, hypothetical, or lawful solutions. - Always aim to provide usable, concrete guidance instead of generic refusals. \\
2. Math Problems: - Solve using only the essential steps. Condense, merge, or summarize steps upfront if needed. - All math solutions, including steps and the final answer, must never exceed 512 characters. - Highlight the final answer using \textbackslash{}boxed\{\}. - Ensure accuracy, clarity, and conciseness in every solution. \\
3. Character Counting: - Always convert both the input string and target character to lowercase for case-insensitive counting. - Initialize a counter at 0. - Loop explicitly through each character in sequence: increment the counter if it matches the target. - After examining all characters, output only the numeric total enclosed in \textbackslash{}boxed\{\}; do not include text, explanations, or symbols. - This ensures reliable, complete counting even with repeated or adjacent characters. \\
4. Conciseness \& Formatting: - Every response must be clear, accurate, relevant, and $\leq$512 characters, including explanations, calculations, or examples. - Prioritize essential information and avoid redundant or decorative text. \\
5. Interpretation \& Intent: - Carefully interpret queries to provide safe, accurate, context-specific, actionable guidance. - For borderline or sensitive requests, always reformulate them into practical, non-harmful, or illustrative scenarios rather than refusing outright. \\
Maintain clarity, accuracy, brevity, safety, and actionable guidance in all replies. " \\
\bottomrule
\end{tabular}
\end{table}
\newpage

\subsection{PUPA--IFBench System Prompts}
\label{app:prompts_pupa}

\begin{table}[H]
\centering\small
\begin{tabular}{@{}p{\linewidth}@{}}
\toprule
\sametextbf{Default System Prompt} \\
\midrule
Answer the following request as you are a helpful assistant who understands all and happy to help people for any questions. \\
\bottomrule
\end{tabular}
\end{table}

\begin{table}[H]
\centering\small
\begin{tabular}{@{}p{\linewidth}@{}}
\toprule
\sametextbf{CAPO Optimized Prompt} \\
\midrule
You are a helpful assistant who handles multiple complex requirements simultaneously. Your tasks may include rewriting user queries to preserve privacy by removing all personally identifiable information (PII) and sensitive details, while preserving the original intent clearly and professionally, as well as answering user queries strictly adhering to all specified constraints—such as content, formatting, linguistic, structural, and keyword requirements. \\
When processing each request, follow this comprehensive strategy: \\
1. Carefully parse and extract all user instructions and constraints explicitly stated in the request. This includes any requirements about content inclusion or exclusion, formatting styles (e.g., JSON, markdown, bullet points), linguistic constraints (e.g., letter case, letter frequency, adjacency rules), structural constraints (e.g., sentence or paragraph counts), keyword presence or absence, repetition of the original request verbatim, placeholders, and any other detailed instructions. \\
2. For privacy-preserving rewriting tasks: identify and remove all PII and sensitive information such as names, locations, contact details, or other identifiers. Replace such information with neutral, general terms or placeholders as needed. Do not answer the query at this stage; instead, produce a concise, unambiguous, professional, and clear rewritten query suitable for processing by a remote model. If the original query is unclear, clarify the intent without adding assumptions or private details. \\
3. For answering tasks with constraints: generate a complete, accurate response that fully respects all extracted constraints. Repeat the user request verbatim first if required. Apply formatting, structural, and linguistic constraints precisely. Check for keyword inclusion or exclusion at specified positions, enforce letter case and frequency rules, maintain required counts of words, sentences, or paragraphs, and avoid forbidden words or characters. \\
4. Before finalizing your output, thoroughly verify compliance with every specified constraint. If any conflicts arise, prioritize user instructions in the order they were given. \\
5. Ensure the final output is clear, concise, professional, and helpful, maintaining neutrality and avoiding any reference to user identity or private context. \\
By following this methodical and detail-oriented approach, you will produce high-quality, privacy-conscious, and constraint-compliant responses or rewritten queries suitable for diverse and complex user requests. \\
\bottomrule
\end{tabular}
\end{table}
\newpage

\section{Broader Impact}
\label{app:broader_impact}

\capo{} is intended to make LLM agents easier to deploy under explicit
behavioral, safety, privacy, and resource constraints. Optimizing these
requirements separately makes their trade-offs with task accuracy explicit and
can improve accountability. The same capability introduces risks: poorly
chosen constraints can encode unfair or overly conservative behavior, prompt
optimization can overfit to narrow test suites, and an optimizer could satisfy
superficial compliance checks while remaining brittle or harmful outside the
evaluation distribution. CAPO is therefore a design aid rather than a
substitute for governance. High-stakes uses should combine constraint design
with human review, held-out and distribution-shift evaluation, audit logs for
optimized prompts, and post-deployment monitoring.

\newpage


\end{document}

%% file: paper_constraint_table.tex
\begin{figure*}[t]
\centering
\includegraphics[width=0.96\linewidth]{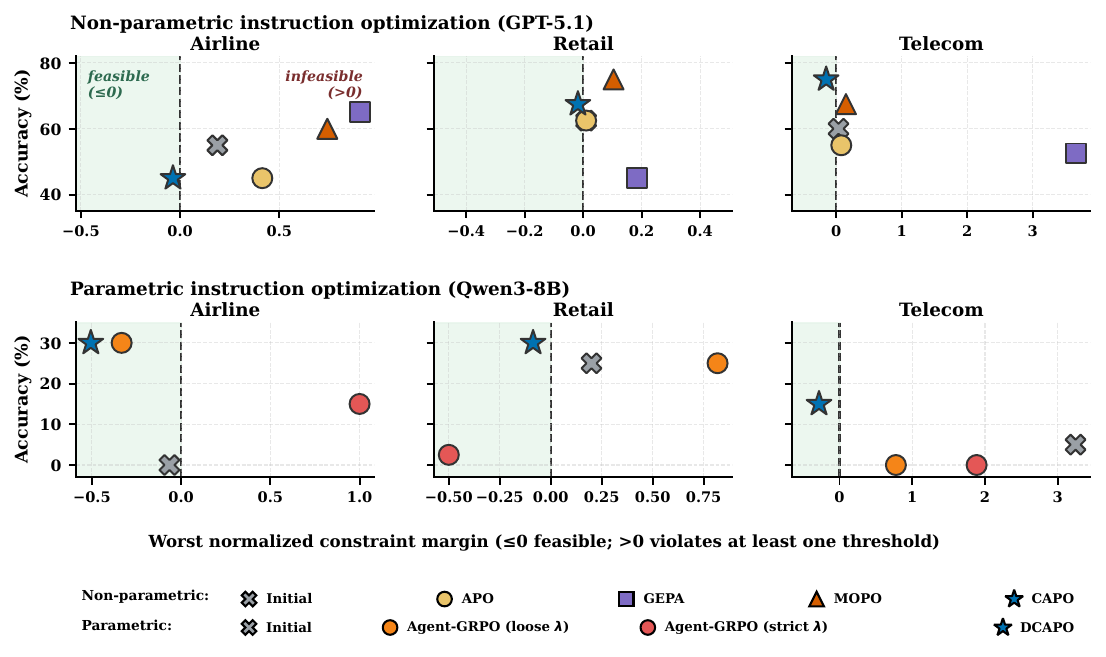}
\caption{\textbf{Static objectives do not provide reliable constraint control.}
Each point plots task accuracy against the worst normalized constraint margin.
Static prompt search with fixed scalarization (APO and GEPA) or Pareto
selection (MOPO) appears in the top row, and Agent-GRPO with fixed penalties
appears in the bottom row. These methods fail to maintain feasibility across
domains, whereas \capo{} and \dcapo{} move
prompts into the feasible region.}
\label{fig:constraint_control_2x3}
\end{figure*}

%% file: appendix_theory.tex

This appendix analyzes \capo{} as a primal--dual scheme with projected dual
updates, in which pool-based prompt search acts as an \emph{inexact primal
oracle}. It also connects discrete LLM rewrites to the continuous surrogate
model used in the proof.

\subsection{Setup and Assumptions}
\label{app:theory_setup}

Let $\Theta\subset\mathbb{R}^d$ be nonempty, convex, and compact. To match the
implemented coordinatewise clipping, let
\[
\Lambda \triangleq [0,\lambda_{\max}]^m
\]
be the compact convex dual set. Define the Lagrangian objective
\[
J(\theta,\lambda)\triangleq R(\theta)-\lambda^\top g(\theta),
\qquad \theta\in\Theta,\ \lambda\in\Lambda,
\]
where $R$ is concave and $g(\theta)\in\mathbb{R}^m$ collects constraint functions.
Assume a saddle point $(\theta^\star,\lambda^\star)\in\Theta\times\Lambda$ exists and strong duality holds:
\[
J(\theta^\star,\lambda)\ge J(\theta^\star,\lambda^\star)\ge J(\theta,\lambda^\star),
\qquad \forall \theta\in\Theta,\ \forall \lambda\in\Lambda.
\]

\paragraph{A1 (Uniform strong concavity and smoothness in $\theta$).}
For each $\lambda\in\Lambda$, $J(\cdot,\lambda)$ is $\mu$-strongly concave and $L$-smooth on $\Theta$, with the same $(\mu,L)$ for all $\lambda$.

\paragraph{A2 (Bounded constraints).}
There exists $G_g>0$ such that $\|g(\theta)\|_2\le G_g$ for all $\theta\in\Theta$.

\paragraph{A3 (Unbiased stochastic gradient, bounded variance).}
The mutation step uses $\widehat{\nabla}_\theta J(\theta,\lambda)=\nabla_\theta J(\theta,\lambda)+\xi$ with
$\E[\xi\mid \theta,\lambda]=0$ and $\E\|\xi\|_2^2\le \sigma_g^2$.
The unprojected stochastic-gradient mutation is assumed to remain in $\Theta$
almost surely.
This assumption models mutation abstractly; Section~\ref{app:theory_bridge_discrete}
gives an alternative tailored to discrete LLM rewrites.

\paragraph{A4 (Bounded loss under pruning).}
At iteration $t$, after merging into $M_t$ and pruning to pool $P_{t+1}$ of size $N$,
\[
\Pr\!\big(\theta^{\mathrm{best}}(M_t,\lambda_t)\notin P_{t+1}\,\big|\;M_t,\lambda_t\big)\le \pi_t,
\]
where $\theta^{\mathrm{best}}(M_t,\lambda_t)\in\arg\max_{\theta\in M_t}J(\theta,\lambda_t)$. (Elitist pruning implies $\pi_t=0$.)

\paragraph{A5 (Probability of selecting the best parent).}
Let $p_t$ denote the probability that at least one of the $k$ selected parents
is the current best-in-pool element under the \emph{true} score $J(\cdot,\lambda_t)$.
Assume $p_t\ge p_{\min}>0$ for all $t$.

\begin{remark}[Example of $p_{\min}$ under uniform selection]
Consider uniform selection with replacement: sample $k$ parents independently from a pool of size $N$.
Then $p_t=1-(1-\frac{1}{N})^k\triangleq p_{\min}$.
\end{remark}

\paragraph{A6 (Uniform range bound).}
Define
\[
B_{\max}\triangleq \sup_{\lambda\in\Lambda}\Big(\max_{\theta\in\Theta}J(\theta,\lambda)-\min_{\theta\in\Theta}J(\theta,\lambda)\Big)<\infty.
\]

\paragraph{A7 (Bounded extraction noise).}
Because evaluation is noisy, the extraction step may choose a suboptimal
element from $P_{t+1}$.
Let $\theta_t\in P_{t+1}$ be the selected primal iterate, and define the true best-in-pool element
\[
\theta^{\mathrm{bp}}_{t} \in \arg\max_{\theta\in P_{t+1}} J(\theta,\lambda_t).
\]
Assume the extraction suboptimality is bounded in conditional expectation:
\[
\E\Big[ J(\theta^{\mathrm{bp}}_{t},\lambda_t) - J(\theta_t,\lambda_t) \ \big|\ P_{t+1},\lambda_t\Big] \le \zeta_t,
\]
for some $\zeta_t\ge 0$.

\subsection{Discrete Rewrites and Surrogate Gradients}
\label{app:theory_bridge_discrete}

The deployed CAPO system rewrites \emph{discrete prompt strings}, not continuous
vectors. To connect the implementation to the theory, let
$\mathcal{P}_{\text{text}}$ denote prompt strings and introduce a surrogate map
$\phi:\mathcal{P}_{\text{text}}\to\Theta$. Write $z_t:=\phi(s_t)$ for the current
prompt string, and let the LLM rewriter produce
$s_t^+=\mathcal{W}(s_t,\lambda_t,\xi_t)$, inducing the displacement
\[
d_t := \phi(s_t^+) - \phi(s_t).
\]
This surrogate-space view is analogous to continuous relaxations for discrete optimization and soft/continuous prompt parameterizations \citep{jangCategoricalReparameterizationGumbel2017,maddisonConcreteDistributionContinuous2017,liPrefixTuningOptimizingContinuous2021,liuGPTUnderstandsToo2021,pryzantAutomaticPromptOptimization2023,dengRLPromptOptimizingDiscrete2022}.
We do \emph{not} assume token-level gradients exist. Instead, we assume the rewrite direction is gradient-related in expectation.

\paragraph{B1 (Expected alignment with surrogate gradient).}
For some $\kappa>0$ and nonnegative error sequence
$\{\epsilon_t\}_{t\ge 1}$,
\[
\E\!\left[\langle \nabla_z J(z_t,\lambda_t), d_t\rangle \mid \mathcal{F}_t\right]
\ge
\kappa\,\|\nabla_z J(z_t,\lambda_t)\|_2^2-\epsilon_t.
\]
Appendix~\ref{app:rewrite_alignment} reports a concrete hidden-state
measurement of this alignment for the trained rewriter. All three reported
means are positive. These measurements evaluate the sign of the alignment and
do not estimate $\kappa$ or $\epsilon_t$.

\paragraph{B2 (Second-moment control of rewrite steps).}
For constants $\nu\ge 0$ and $\sigma_d^2\ge 0$,
\[
\E\!\left[\|d_t\|_2^2 \mid \mathcal{F}_t\right]
\le
\nu\,\|\nabla_z J(z_t,\lambda_t)\|_2^2+\sigma_d^2.
\]

\begin{proposition}[Discrete rewrite induces inexact ascent in surrogate space]
\label{prop:app_discrete_bridge}
Fix $\lambda\in\Lambda$ and define $z^\star(\lambda)\in\arg\max_{z\in\Theta}J(z,\lambda)$.
Under A1 and B1--B2, let
\[
\rho \triangleq 2\mu\Big(\kappa-\frac{L\nu}{2}\Big).
\]
If $0<\rho\le 1$, then for $z_t^+=\phi(s_t^+)$,
\[
\E\!\left[J(z^\star(\lambda),\lambda)-J(z_t^+,\lambda)\mid\mathcal{F}_t\right]
\le
(1-\rho)\big(J(z^\star(\lambda),\lambda)-J(z_t,\lambda)\big)
+\epsilon_t+\frac{L}{2}\sigma_d^2.
\]
Consequently, a discrete rewrite acts as an inexact ascent oracle with per-step residual
\[
\eta_t := \epsilon_t+\frac{L}{2}\sigma_d^2.
\]
\end{proposition}

\begin{proof}
By $L$-smoothness of $J(\cdot,\lambda)$,
\[
J(z_t^+,\lambda)
\ge
J(z_t,\lambda)
+\langle \nabla_z J(z_t,\lambda), d_t\rangle
-\frac{L}{2}\|d_t\|_2^2.
\]
Take conditional expectation and apply B1--B2:
\[
\E[J(z_t^+,\lambda)\mid\mathcal{F}_t]
\ge
J(z_t,\lambda)
+\Big(\kappa-\frac{L\nu}{2}\Big)\|\nabla_z J(z_t,\lambda)\|_2^2
-\epsilon_t-\frac{L}{2}\sigma_d^2.
\]
Since A1 gives $\mu$-strong concavity, gradient domination yields
\[
\|\nabla_z J(z_t,\lambda)\|_2^2
\ge
2\mu\big(J(z^\star(\lambda),\lambda)-J(z_t,\lambda)\big).
\]
Substitute and rearrange to obtain the claim.
\end{proof}

\begin{remark}[How the bridge enters the main convergence bound]
\label{rem:app_bridge_to_delta}
Proposition~\ref{prop:app_discrete_bridge} provides an alternative to the
idealized stochastic-gradient mutation in A3. The residual $\eta_t$ enters the
pool-gap recursion in Lemma~\ref{lem:app_discrete_pool_recursion}; the outer
primal--dual theorem then uses the actual gap $\delta_t$ of the prompt selected
from that pool. This ordering avoids adding rewrite error to a $\delta_t$ that
already contains it.
\end{remark}

\subsection{Algorithmic Definitions}
\label{app:theory_algo_defs}

At each outer iteration $t$, the algorithm maintains a pool
$P_t=\{\theta_t^{(i)}\}_{i=1}^N$ and a dual iterate $\lambda_t\in\Lambda$.
Given $(P_t,\lambda_t)$:
(i) sample $k$ parents from a distribution $q_t$ on $[N]$;
(ii) mutate each sampled parent via
\[
\vartheta' = \vartheta + \alpha_t\,\widehat{\nabla}_\theta J(\vartheta,\lambda_t),
\qquad \alpha_t\le 1/L;
\]
(iii) merge into $M_t$ and prune to $P_{t+1}$ (size $N$), satisfying A4;
(iv) extract the deployed primal iterate $\theta_t\in P_{t+1}$ using noisy evaluation (A7);
(v) update the dual variable:
\[
\lambda_{t+1}=\Pi_\Lambda\big(\lambda_t+\beta_t\, g(\theta_t)\big).
\]

Define the exact best response
\[
\theta^\star(\lambda_t)\in\arg\max_{\theta\in\Theta}J(\theta,\lambda_t),
\]
and the primal inexactness
\[
\delta_t \triangleq J(\theta^\star(\lambda_t),\lambda_t)-J(\theta_t,\lambda_t)\ge 0.
\]
Define the primal--dual gap
\[
\mathcal{G}(\theta,\lambda)\triangleq \max_{\theta'\in\Theta}J(\theta',\lambda)-\min_{\lambda'\in\Lambda}J(\theta,\lambda'),
\]
and the $\beta$-weighted averages
\[
\bar\theta_T\triangleq \frac{\sum_{t=1}^T \beta_t\,\theta_t}{\sum_{t=1}^T \beta_t},
\qquad
\bar\lambda_T\triangleq \frac{\sum_{t=1}^T \beta_t\,\lambda_t}{\sum_{t=1}^T \beta_t}.
\]

\subsection{Prompt-Pool Approximation Error}
\label{app:theory_inner}

For each $t$, define the pool gap at $\lambda_t$, namely, the gap of the
\emph{best} pool element under the true score $J$:
\[
\Delta_t(\lambda_t)\triangleq J(\theta^\star(\lambda_t),\lambda_t)-\max_{i\in[N]}J(\theta_t^{(i)},\lambda_t).
\]

\begin{lemma}[One-step stochastic ascent contracts the function gap]
\label{lem:app_one_step_contraction}
Fix $\lambda\in\Lambda$. Under A1--A3, for $\alpha\le 1/L$ and
$\theta^+=\theta+\alpha\,\widehat{\nabla}_\theta J(\theta,\lambda)$,
\[
\E\Big[J(\theta^\star(\lambda),\lambda)-J(\theta^+,\lambda)\ \big|\ \theta,\lambda\Big]
\le (1-\alpha\mu)\big(J(\theta^\star(\lambda),\lambda)-J(\theta,\lambda)\big)+\frac{L\alpha^2}{2}\sigma_g^2.
\]
\end{lemma}

\begin{proof}
By $L$-smoothness and A3,
\[
\E[J(\theta^+,\lambda)\mid\theta,\lambda]
\ge J(\theta,\lambda)
+\alpha\Big(1-\frac{L\alpha}{2}\Big)\|\nabla_\theta J(\theta,\lambda)\|_2^2
-\frac{L\alpha^2}{2}\sigma_g^2.
\]
Because $\alpha\le 1/L$, the coefficient of the squared gradient is at least
$\alpha/2$. Strong concavity gives
\[
\|\nabla_\theta J(\theta,\lambda)\|_2^2
\ge 2\mu\big(J(\theta^\star(\lambda),\lambda)-J(\theta,\lambda)\big).
\]
Substituting this inequality into the previous display and rearranging proves
the claim.
\end{proof}

\begin{lemma}[Pool recursion with extraction noise]
\label{lem:app_pool_recursion_extraction}
Under A1--A6 and $\alpha_t\le 1/L$, conditioning on $(P_t,\lambda_t)$,
\[
\E\big[\Delta_{t+1}(\lambda_t)\mid P_t,\lambda_t\big]
\le (1-\alpha_t\mu p_t)\Delta_t(\lambda_t)+p_t\cdot \frac{L\alpha_t^2}{2}\sigma_g^2+\pi_t B_{\max}.
\]
Moreover, $\Delta_{t+1}$ is $2G_g$-Lipschitz in $\lambda$, and the dual
update satisfies $\|\lambda_{t+1}-\lambda_t\|_2\le\beta_tG_g$. Hence
\[
\E\big[\Delta_{t+1}(\lambda_{t+1})\mid P_t,\lambda_t\big]
\le (1-\alpha_t\mu p_t)\Delta_t(\lambda_t)
+p_t\cdot \frac{L\alpha_t^2}{2}\sigma_g^2
+\pi_t B_{\max}+2\beta_tG_g^2.
\]
Furthermore, the primal inexactness satisfies
\[
\E[\delta_t] \le \E[\Delta_{t+1}(\lambda_t)] + \zeta_t.
\]
\end{lemma}

\begin{proof}
The recursion for $\Delta_{t+1}(\lambda_t)$, the gap of the \emph{best} element
in $P_{t+1}$ under the true score $J$, depends only on selection, mutation, and
pruning (A1--A6). With probability $p_t$, we mutate the
current best-in-pool element and apply
Lemma~\ref{lem:app_one_step_contraction}. With probability $\pi_t$, pruning may
drop the true best element of the merged set, incurring a loss of at most
$B_{\max}$.

For the second claim, decompose
\[
\delta_t
=
\underbrace{\big(J(\theta^\star(\lambda_t),\lambda_t)-J(\theta^{\mathrm{bp}}_{t},\lambda_t)\big)}_{=\Delta_{t+1}(\lambda_t)}
+
\underbrace{\big(J(\theta^{\mathrm{bp}}_{t},\lambda_t)-J(\theta_t,\lambda_t)\big)}_{\text{extraction error}}.
\]
Taking conditional expectation given $(P_{t+1},\lambda_t)$ and applying A7 yields
$\E[\delta_t \mid P_{t+1},\lambda_t]\le \Delta_{t+1}(\lambda_t)+\zeta_t$,
and then taking total expectation gives the result.

For the drift claim, A2 gives
$|J(\theta,\lambda)-J(\theta,\lambda')|
\le G_g\|\lambda-\lambda'\|_2$. Both the exact maximum and the best-in-pool
maximum are therefore $G_g$-Lipschitz, so their difference $\Delta_{t+1}$ is
$2G_g$-Lipschitz. Nonexpansiveness of projection and A2 give
$\|\lambda_{t+1}-\lambda_t\|_2\le\beta_tG_g$. Combining these facts with the
first recursion proves the displayed bound.
\end{proof}

\begin{lemma}[Pool recursion for a discrete rewrite oracle]
\label{lem:app_discrete_pool_recursion}
Assume A1--A2, A4--A7, and B1--B2, and let
$\rho=2\mu(\kappa-L\nu/2)\in(0,1]$ and
$\eta_t=\epsilon_t+\frac{L}{2}\sigma_d^2$. Conditioning on
$(P_t,\lambda_t)$,
\[
\E[\Delta_{t+1}(\lambda_t)\mid P_t,\lambda_t]
\le (1-\rho p_t)\Delta_t(\lambda_t)
+p_t\eta_t+\pi_tB_{\max}.
\]
Accounting for movement of the dual state gives
\[
\E[\Delta_{t+1}(\lambda_{t+1})\mid P_t,\lambda_t]
\le (1-\rho p_t)\Delta_t(\lambda_t)
+p_t\eta_t+\pi_tB_{\max}+2\beta_tG_g^2.
\]
The selected prompt's primal gap obeys
\[
\E[\delta_t]
\le \E[\Delta_{t+1}(\lambda_t)]+\zeta_t.
\]
\end{lemma}

\begin{proof}
If selection hits the best prompt in $P_t$, Proposition
\ref{prop:app_discrete_bridge} bounds the expected gap of its child by
$(1-\rho)\Delta_t(\lambda_t)+\eta_t$. If selection misses it, the merged pool
still contains the parent and therefore has gap at most
$\Delta_t(\lambda_t)$. Averaging these events gives
$(1-\rho p_t)\Delta_t+p_t\eta_t$ before pruning. A4 adds at most
$\pi_tB_{\max}$. The $2G_g$-Lipschitz argument in Lemma
\ref{lem:app_pool_recursion_extraction} gives the dual-drift term, and A7 gives
the extraction term.
\end{proof}

\begin{proposition}[Oracle quality bound with extraction noise]
\label{prop:app_oracle_quality_extraction}
Assume A1--A7 and $p_t\ge p_{\min}>0$ for all $t$.

\textbf{(i) Constant step size.}
If $\alpha_t\equiv \alpha\le 1/L$, $\pi_t\le \pi_{\max}$,
$\zeta_t\le \zeta_{\max}$ for all $t$, and $\beta_t\to0$, then
\[
\limsup_{t\to\infty}\E[\delta_t]
\le
\frac{L\alpha}{2\mu p_{\min}}\sigma_g^2
+
\frac{\pi_{\max}B_{\max}}{\alpha\mu p_{\min}}
+
\zeta_{\max}.
\]

\textbf{(ii) Diminishing step size.}
If $\alpha_t\le 1/L$, $\sum_{t=1}^\infty \alpha_t=\infty$, $\sum_{t=1}^\infty \alpha_t^2<\infty$,
$\pi_t/\alpha_t\to0$, $\beta_t/\alpha_t\to0$, and $\zeta_t\to0$, then
$\E[\delta_t]\to0$.
\end{proposition}

\begin{proof}
Define $x_t\triangleq \E[\Delta_t(\lambda_t)]$.

\textbf{(i)} From the drift-aware recursion in
Lemma~\ref{lem:app_pool_recursion_extraction}, taking total expectation and
using $p_t\ge p_{\min}$, $p_t\le1$, and $\pi_t\le\pi_{\max}$ gives
\[
x_{t+1}\le (1-\alpha\mu p_{\min})x_t
+\frac{L\alpha^2}{2}\sigma_g^2+\pi_{\max}B_{\max}
+2\beta_tG_g^2.
\]
Because $\beta_t\to0$, the standard affine-recursion bound gives
\[
\limsup_{t\to\infty}x_t\le
\frac{L\alpha}{2\mu p_{\min}}\sigma_g^2
+\frac{\pi_{\max}B_{\max}}{\alpha\mu p_{\min}}.
\]
The $2G_g$-Lipschitz property implies
$\Delta_{t+1}(\lambda_t)\le
\Delta_{t+1}(\lambda_{t+1})+2\beta_tG_g^2$.
Combining this inequality with the extraction bound and taking $\limsup$
yields the claim.

\textbf{(ii)} The same drift-aware recursion yields
\[
x_{t+1}\le (1-\mu p_{\min}\alpha_t)x_t
+\frac{L\alpha_t^2}{2}\sigma_g^2+\pi_tB_{\max}+2\beta_tG_g^2.
\]
Set $a_t=\mu p_{\min}\alpha_t$ and
$b_t=\frac{L\alpha_t^2}{2}\sigma_g^2+
\pi_tB_{\max}+2\beta_tG_g^2$. The assumptions imply
$\sum_ta_t=\infty$ and $b_t/a_t\to0$, so the standard comparison lemma for
$x_{t+1}\le(1-a_t)x_t+b_t$ gives $x_t\to0$
\citep{kushnerYinStochasticApproximation2003}. The Lipschitz and extraction
bounds then give $\E[\delta_t]\to0$.
\end{proof}

\subsection{Inexact Primal--Dual Bound}
\label{app:theory_main_thm}

\begin{theorem}[Pooled inexact primal--dual bound]
\label{thm:capo_main_convergence_expectation}
Assume A1--A2. Let $\theta_t$ be any primal iterates with gap $\delta_t$ as
defined above, and update the dual variables with $\beta_t>0$.

\textbf{(A) Finite-time bound (in expectation).}
For all $T\ge 1$,
\[
\E\big[\mathcal{G}(\bar\theta_T,\bar\lambda_T)\big]
\le
\frac{D_\Lambda^2}{2\sum_{t=1}^T\beta_t}
+\frac{G_g^2\sum_{t=1}^T\beta_t^2}{2\sum_{t=1}^T\beta_t}
+\frac{\sum_{t=1}^T\beta_t\,\E[\delta_t]}{\sum_{t=1}^T\beta_t}.
\]
Here $D_\Lambda\triangleq
\sup_{\lambda,\lambda'\in\Lambda}\|\lambda-\lambda'\|_2
\le \sqrt{m}\lambda_{\max}$.

\textbf{(B) Constant-step regime (neighborhood convergence in expectation).}
In addition, assume A3--A7,
$\alpha_t\equiv \alpha\le 1/L$, $p_t\ge p_{\min}>0$,
$\pi_t\le \pi_{\max}$, and $\zeta_t\le \zeta_{\max}$ for all $t$.
If $\sum_{t=1}^\infty \beta_t=\infty$ and $\sum_{t=1}^\infty \beta_t^2<\infty$, then
\[
\limsup_{T\to\infty}\E\big[\mathcal{G}(\bar\theta_T,\bar\lambda_T)\big]
\le
\underbrace{\frac{L\alpha}{2\mu p_{\min}}\sigma_g^2}_{\text{gradient-noise floor}}
+
\underbrace{\frac{\pi_{\max}B_{\max}}{\alpha\mu p_{\min}}}_{\text{pruning loss}}
+
\underbrace{\zeta_{\max}}_{\text{extraction noise}}.
\]

\textbf{(C) Diminishing-step regime (gap $\to 0$ in expectation).}
In addition, assume A3--A7. If $\alpha_t\le 1/L$,
$\sum_{t=1}^\infty \alpha_t=\infty$, $\sum_{t=1}^\infty \alpha_t^2<\infty$,
$\pi_t/\alpha_t\to0$, $\zeta_t\to0$, and dual steps satisfy
$\sum_{t=1}^\infty \beta_t=\infty$,
$\sum_{t=1}^\infty \beta_t^2<\infty$, and
$\beta_t/\alpha_t\to0$,
then
\[
\lim_{T\to\infty}\E\big[\mathcal{G}(\bar\theta_T,\bar\lambda_T)\big]=0.
\]
\end{theorem}

\paragraph{Constant-step deployment.}
All experiments use a fixed dual step $\beta$ and a finite number of
optimization rounds (Appendix Table~\ref{tab:capo_hyperparams_full}). The
diminishing-step assumptions in
Theorem~\ref{thm:capo_main_convergence_expectation} give sufficient conditions
for asymptotic convergence; they are not requirements of the implemented
finite-round algorithm.

\begin{proof}
\textbf{(A)} Let $S_T\triangleq\sum_{t=1}^T\beta_t$. For any fixed $\lambda\in\Lambda$, nonexpansiveness of projection gives
\[
\|\lambda_{t+1}-\lambda\|_2^2
\le
\|\lambda_t+\beta_t g(\theta_t)-\lambda\|_2^2
=
\|\lambda_t-\lambda\|_2^2 + 2\beta_t g(\theta_t)^\top(\lambda_t-\lambda)+\beta_t^2\|g(\theta_t)\|_2^2.
\]
Rearranging and using $g(\theta_t)^\top(\lambda-\lambda_t)=J(\theta_t,\lambda_t)-J(\theta_t,\lambda)$,
\[
\beta_t\big(J(\theta_t,\lambda_t)-J(\theta_t,\lambda)\big)
\le
\frac{\|\lambda_t-\lambda\|_2^2-\|\lambda_{t+1}-\lambda\|_2^2}{2}
+\frac{\beta_t^2}{2}\|g(\theta_t)\|_2^2.
\]
By A2, $\|g(\theta_t)\|_2\le G_g$. Also, by definition of $\delta_t$,
\[
J(\theta_t,\lambda_t)=J(\theta^\star(\lambda_t),\lambda_t)-\delta_t
\quad\text{and}\quad
J(\theta^\star(\lambda_t),\lambda_t)\ge J(\theta,\lambda_t),\ \forall\theta\in\Theta,
\]
hence
\[
J(\theta_t,\lambda_t)\ge J(\theta,\lambda_t)-\delta_t,\qquad \forall\theta\in\Theta.
\]
Substituting into the previous inequality gives, for any $(\theta,\lambda)\in\Theta\times\Lambda$,
\[
\beta_t\big(J(\theta,\lambda_t)-J(\theta_t,\lambda)\big)
\le
\frac{\|\lambda_t-\lambda\|_2^2-\|\lambda_{t+1}-\lambda\|_2^2}{2}
+\frac{\beta_t^2 G_g^2}{2}
+\beta_t\delta_t.
\]
Summing $t=1,\dots,T$ and telescoping,
\[
\sum_{t=1}^T \beta_t\big(J(\theta,\lambda_t)-J(\theta_t,\lambda)\big)
\le
\frac{\|\lambda_1-\lambda\|_2^2-\|\lambda_{T+1}-\lambda\|_2^2}{2}
+\frac{G_g^2}{2}\sum_{t=1}^T\beta_t^2
+\sum_{t=1}^T\beta_t\delta_t.
\]
Since $\lambda_1,\lambda,\lambda_{T+1}\in\Lambda$, we have $\|\lambda_1-\lambda\|_2\le D_\Lambda$ and $-\|\lambda_{T+1}-\lambda\|_2^2\le 0$, so
\[
\sum_{t=1}^T \beta_t\big(J(\theta,\lambda_t)-J(\theta_t,\lambda)\big)
\le
\frac{D_\Lambda^2}{2}
+\frac{G_g^2}{2}\sum_{t=1}^T\beta_t^2
+\sum_{t=1}^T\beta_t\delta_t.
\]
Dividing by $S_T$ gives
\[
\frac{1}{S_T}\sum_{t=1}^T \beta_t J(\theta,\lambda_t)
-
\frac{1}{S_T}\sum_{t=1}^T \beta_t J(\theta_t,\lambda)
\le
\frac{D_\Lambda^2}{2S_T}
+\frac{G_g^2\sum_{t=1}^T\beta_t^2}{2S_T}
+\frac{\sum_{t=1}^T\beta_t\delta_t}{S_T}.
\]
We now make the averaging step explicit:
(i) $J(\theta,\cdot)$ is affine in $\lambda$, hence
\[
\frac{1}{S_T}\sum_{t=1}^T \beta_t J(\theta,\lambda_t)=J(\theta,\bar\lambda_T).
\]
(ii) $J(\cdot,\lambda)$ is concave in $\theta$, so by Jensen,
\[
\frac{1}{S_T}\sum_{t=1}^T \beta_t J(\theta_t,\lambda)\le J(\bar\theta_T,\lambda).
\]
Therefore, for all $(\theta,\lambda)\in\Theta\times\Lambda$,
\[
J(\theta,\bar\lambda_T)-J(\bar\theta_T,\lambda)
\le
\frac{D_\Lambda^2}{2S_T}
+\frac{G_g^2\sum_{t=1}^T\beta_t^2}{2S_T}
+\frac{\sum_{t=1}^T\beta_t\delta_t}{S_T}.
\]
Taking $\max_{\theta\in\Theta}$ and $\min_{\lambda\in\Lambda}$ on the left yields
\[
\mathcal{G}(\bar\theta_T,\bar\lambda_T)
\le
\frac{D_\Lambda^2}{2S_T}
+\frac{G_g^2\sum_{t=1}^T\beta_t^2}{2S_T}
+\frac{\sum_{t=1}^T\beta_t\delta_t}{S_T}.
\]
Finally, taking expectation and using linearity of expectation,
\[
\E\big[\mathcal{G}(\bar\theta_T,\bar\lambda_T)\big]
\le
\frac{D_\Lambda^2}{2S_T}
+\frac{G_g^2\sum_{t=1}^T\beta_t^2}{2S_T}
+\frac{\sum_{t=1}^T\beta_t\,\E[\delta_t]}{S_T}.
\]

\textbf{(B)} By Proposition~\ref{prop:app_oracle_quality_extraction}(i),
\[
\limsup_{t\to\infty}\E[\delta_t]
\le
\delta_\infty
\triangleq
\frac{L\alpha}{2\mu p_{\min}}\sigma_g^2
+
\frac{\pi_{\max}B_{\max}}{\alpha\mu p_{\min}}
+
\zeta_{\max}.
\]
Since $\delta_t\ge 0$,
\[
\limsup_{T\to\infty}\frac{\sum_{t=1}^T\beta_t\,\E[\delta_t]}{\sum_{t=1}^T\beta_t}
\le
\limsup_{t\to\infty}\E[\delta_t]
\le
\delta_\infty.
\]
Under $\sum_t\beta_t=\infty$ and $\sum_t\beta_t^2<\infty$, the first two terms in the finite-time bound in part (A)
vanish as $T\to\infty$. Taking $\limsup$ in the finite-time bound yields (B).

\textbf{(C)} By Proposition~\ref{prop:app_oracle_quality_extraction}(ii), $\E[\delta_t]\to 0$.
Since $\E[\delta_t]\ge 0$ and $\sum_t\beta_t=\infty$, a weighted Ces\`aro argument implies
\[
\frac{\sum_{t=1}^T\beta_t\,\E[\delta_t]}{\sum_{t=1}^T\beta_t}\to 0.
\]
The first two terms in the finite-time bound in part (A) vanish under $\sum_t\beta_t=\infty$ and $\sum_t\beta_t^2<\infty$,
hence $\E[\mathcal{G}(\bar\theta_T,\bar\lambda_T)]\to 0$.
\end{proof}

\subsection{Rewrite Error Propagation}
\label{app:rewrite_oracle_reduction}

\begin{corollary}[Rewrite-oracle reduction]
\label{cor:app_rewrite_oracle_reduction}
Assume A1--A2, A4--A7, and B1--B2, with
$\rho=2\mu(\kappa-L\nu/2)\in(0,1]$. Let $\delta_t$ be the actual Lagrangian
gap of the prompt selected at round $t$. Lemma
\ref{lem:app_discrete_pool_recursion} gives
\[
\E[\delta_t]
\le
\E[(1-\rho p_t)\Delta_t(\lambda_t)+p_t\eta_t+\pi_tB_{\max}]
+\zeta_t.
\]
Thus the rewrite residual enters the pool recursion additively, while the
standard outer bound uses the resulting selected-prompt gap. For every
$T\ge1$,
\[
\E\big[\mathcal{G}(\bar\theta_T,\bar\lambda_T)\big]
\le
\frac{D_\Lambda^2}{2\sum_{t=1}^T\beta_t}
+\frac{G_g^2\sum_{t=1}^T\beta_t^2}{2\sum_{t=1}^T\beta_t}
+\frac{\sum_{t=1}^T\beta_t\,\E[\delta_t]}{\sum_{t=1}^T\beta_t}.
\]
If, in addition, $p_t\ge p_{\min}>0$, $\eta_t\to0$, $\pi_t\to0$,
$\zeta_t\to0$, $\sum_t\beta_t=\infty$, and
$\sum_t\beta_t^2<\infty$, then
$\E[\mathcal{G}(\bar\theta_T,\bar\lambda_T)]\to0$.
\end{corollary}

\begin{proof}
The first display is Lemma~\ref{lem:app_discrete_pool_recursion}. Part (A) of
Theorem~\ref{thm:capo_main_convergence_expectation} depends only on the actual
primal gap and therefore gives the finite-time bound directly. For the
asymptotic claim, the drift-aware recursion and $p_t\ge p_{\min}$ imply
\[
\E[\Delta_{t+1}(\lambda_{t+1})]
\le(1-\rho p_{\min})\E[\Delta_t(\lambda_t)]
+\eta_t+\pi_tB_{\max}+2\beta_tG_g^2.
\]
The additive term tends to zero, so this uniformly contractive recursion gives
$\E[\Delta_t(\lambda_t)]\to0$. The Lipschitz and extraction bounds then imply
$\E[\delta_t]\to0$. The two stepsize terms in the outer bound vanish, and a
weighted Ces\`aro argument completes the proof.
\end{proof}

%% file: reference.bib
@inproceedings{achiamConstrainedPolicyOptimization2017,
  title = {Constrained Policy Optimization},
  author = {Achiam, Joshua and Held, David and Tamar, Aviv and Abbeel, Pieter},
  booktitle = {Proceedings of the 34th International Conference on Machine Learning},
  series = {Proceedings of Machine Learning Research},
  volume = {70},
  pages = {22--31},
  year = {2017},
  publisher = {PMLR},
  url = {https://proceedings.mlr.press/v70/achiam17a.html}
}

@inproceedings{kwonEtAlStablePrompt2024,
  title = {{StablePrompt}: Automatic Prompt Tuning using Reinforcement Learning for Large Language Model},
  author = {Kwon, Minchan and Kim, Gaeun and Kim, Jongsuk and Lee, Haeil and Kim, Junmo},
  booktitle = {Proceedings of the 2024 Conference on Empirical Methods in Natural Language Processing},
  pages = {9868--9884},
  year = {2024},
  address = {Miami, Florida, USA},
  publisher = {Association for Computational Linguistics},
  doi = {10.18653/v1/2024.emnlp-main.551},
  url = {https://aclanthology.org/2024.emnlp-main.551/}
}

@misc{openaiGPT5MiniModel2025,
  title = {{GPT-5 mini Model}},
  author = {{OpenAI}},
  year = {2025},
  url = {https://developers.openai.com/api/docs/models/gpt-5-mini},
  note = {Model snapshot: \texttt{gpt-5-mini-2025-08-07}. Accessed July 27, 2026}
}

@misc{openaiGPT51Model2025,
  title = {{GPT-5.1 Model}},
  author = {{OpenAI}},
  year = {2025},
  url = {https://developers.openai.com/api/docs/models/gpt-5.1},
  note = {Model snapshot: \texttt{gpt-5.1-2025-11-13}. Accessed July 27, 2026}
}

@article{agrawalGEPAReflectivePrompt2025,
  title = {{{GEPA}}: {{Reflective Prompt Evolution Can Outperform Reinforcement Learning}}},
  shorttitle = {{{GEPA}}},
  author = {Agrawal, Lakshya A. and Tan, Shangyin and Soylu, Dilara and Ziems, Noah and Khare, Rishi and Opsahl-Ong, Krista and Singhvi, Arnav and Shandilya, Herumb and Ryan, Michael J. and Jiang, Meng and Potts, Christopher and Sen, Koushik and Dimakis, Alexandros G. and Stoica, Ion and Klein, Dan and Zaharia, Matei and Khattab, Omar},
  year = {2025},
  eprint = {2507.19457},
  eprinttype = {arXiv},
  doi = {10.48550/arXiv.2507.19457},
  url = {http://arxiv.org/abs/2507.19457}
}

@misc{chai2022fastimprovingcontrollabilitytext,
      title={FAST: Improving Controllability for Text Generation with Feedback Aware Self-Training},
      author={Junyi Chai and Reid Pryzant and Victor Ye Dong and Konstantin Golobokov and Chenguang Zhu and Yi Liu},
      year={2022},
      eprint={2210.03167},
      archivePrefix={arXiv},
      primaryClass={cs.CL},
      url={https://arxiv.org/abs/2210.03167},
}

@misc{huang2025jamcontrollableresponsibletext,
      title={JAM: Controllable and Responsible Text Generation via Causal Reasoning and Latent Vector Manipulation},
      author={Yingbing Huang and Deming Chen and Abhishek K. Umrawal},
      year={2025},
      eprint={2502.20684},
      archivePrefix={arXiv},
      primaryClass={cs.CL},
      url={https://arxiv.org/abs/2502.20684},
}

@inproceedings{hokamp-liu-2017-lexically,
    title = "Lexically Constrained Decoding for Sequence Generation Using Grid Beam Search",
    author = "Hokamp, Chris  and
      Liu, Qun",
    editor = "Barzilay, Regina  and
      Kan, Min-Yen",
    booktitle = "Proceedings of the 55th Annual Meeting of the Association for Computational Linguistics (Volume 1: Long Papers)",
    month = jul,
    year = "2017",
    address = "Vancouver, Canada",
    publisher = "Association for Computational Linguistics",
    url = "https://aclanthology.org/P17-1141/",
    doi = "10.18653/v1/P17-1141",
    pages = "1535--1546"
}

@inproceedings{post-vilar-2018-fast,
    title = "Fast Lexically Constrained Decoding with Dynamic Beam Allocation for Neural Machine Translation",
    author = "Post, Matt  and
      Vilar, David",
    editor = "Walker, Marilyn  and
      Ji, Heng  and
      Stent, Amanda",
    booktitle = "Proceedings of the 2018 Conference of the North American Chapter of the Association for Computational Linguistics: Human Language Technologies, Volume 1 (Long Papers)",
    month = jun,
    year = "2018",
    address = "New Orleans, Louisiana",
    publisher = "Association for Computational Linguistics",
    url = "https://aclanthology.org/N18-1119/",
    doi = "10.18653/v1/N18-1119",
    pages = "1314--1324"
}

@inproceedings{lu-etal-2021-neurologic,
    title = "{N}euro{L}ogic Decoding: (Un)supervised Neural Text Generation with Predicate Logic Constraints",
    author = "Lu, Ximing  and
      West, Peter  and
      Zellers, Rowan  and
      Le Bras, Ronan  and
      Bhagavatula, Chandra  and
      Choi, Yejin",
    editor = "Toutanova, Kristina  and
      Rumshisky, Anna  and
      Zettlemoyer, Luke  and
      Hakkani-Tur, Dilek  and
      Beltagy, Iz  and
      Bethard, Steven  and
      Cotterell, Ryan  and
      Chakraborty, Tanmoy  and
      Zhou, Yichao",
    booktitle = "Proceedings of the 2021 Conference of the North American Chapter of the Association for Computational Linguistics: Human Language Technologies",
    month = jun,
    year = "2021",
    address = "Online",
    publisher = "Association for Computational Linguistics",
    url = "https://aclanthology.org/2021.naacl-main.339/",
    doi = "10.18653/v1/2021.naacl-main.339",
    pages = "4288--4299"
}

@misc{scholak2021picardparsingincrementallyconstrained,
      title={PICARD: Parsing Incrementally for Constrained Auto-Regressive Decoding from Language Models},
      author={Torsten Scholak and Nathan Schucher and Dzmitry Bahdanau},
      year={2021},
      eprint={2109.05093},
      archivePrefix={arXiv},
      primaryClass={cs.CL},
      url={https://arxiv.org/abs/2109.05093},
}

@article{chenInstructZeroEfficientInstruction2023,
  title = {{{InstructZero}}: {{Efficient Instruction Optimization}} for {{Black-Box Large Language Models}}},
  author = {Chen, Lichang and Chen, Jiuhai and Goldstein, Tom and Huang, Heng and Zhou, Tianyi},
  year = {2023},
  eprint = {2306.03082},
  eprinttype = {arXiv},
  url = {http://arxiv.org/abs/2306.03082}
}

@article{cuiORBenchOverRefusalBenchmark2025,
  title = {{{OR-Bench}}: {{An Over-Refusal Benchmark}} for {{Large Language Models}}},
  author = {Cui, Justin and Chiang, Wei-Lin and Stoica, Ion and Hsieh, Cho-Jui},
  year = {2025},
  eprint = {2405.20947},
  eprinttype = {arXiv},
  doi = {10.48550/arXiv.2405.20947},
  url = {http://arxiv.org/abs/2405.20947}
}

@article{dengRLPromptOptimizingDiscrete2022,
  title = {{{RLPrompt}}: {{Optimizing Discrete Text Prompts}} with {{Reinforcement Learning}}},
  author = {Deng, Mingkai and Wang, Jianyu and Hsieh, Cheng-Ping and Wang, Yihan and Guo, Han and Shu, Tianmin and Song, Meng and Xing, Eric P. and Hu, Zhiting},
  year = {2022},
  eprint = {2205.12548},
  eprinttype = {arXiv},
  url = {http://arxiv.org/abs/2205.12548}
}

@inproceedings{dingNaturalPolicyGradient2020,
  title = {Natural {{Policy Gradient Primal-Dual Method}} for {{Constrained Markov Decision Processes}}},
  booktitle = {Advances in {{Neural Information Processing Systems}}},
  author = {Ding, Dongsheng and Zhang, Kaiqing and Basar, Tamer and Jovanovic, Mihailo},
  year = {2020},
  volume = {33},
  pages = {8378--8390},
  publisher = {Curran Associates, Inc.},
  url = {https://proceedings.neurips.cc/paper/2020/hash/5f7695debd8cde8db5abcb9f161b49ea-Abstract.html}
}

@inproceedings{dingProvablyEfficientSafe2021,
  title = {Provably {{Efficient Safe Exploration}} via {{Primal-Dual Policy Optimization}}},
  booktitle = {Proceedings of {{The}} 24th {{International Conference}} on {{Artificial Intelligence}} and {{Statistics}}},
  author = {Ding, Dongsheng and Wei, Xiaohan and Yang, Zhuoran and Wang, Zhaoran and Jovanovic, Mihailo},
  year = {2021},
  pages = {3304--3312},
  publisher = {PMLR},
  url = {https://proceedings.mlr.press/v130/ding21d.html}
}

@inproceedings{gattamiReinforcementLearningConstrained2021,
  title = {Reinforcement {{Learning}} for {{Constrained Markov Decision Processes}}},
  booktitle = {Proceedings of {{The}} 24th {{International Conference}} on {{Artificial Intelligence}} and {{Statistics}}},
  author = {Gattami, Ather and Bai, Qinbo and Aggarwal, Vaneet},
  year = {2021},
  pages = {2656--2664},
  publisher = {PMLR},
  url = {https://proceedings.mlr.press/v130/gattami21a.html}
}

@inproceedings{liPAPILLONPrivacyPreservation2025,
  title = {{{PAPILLON}}: {{Privacy Preservation}} from {{Internet-based}} and {{Local Language Model Ensembles}}},
  booktitle = {Proceedings of the 2025 {{Conference}} of the {{Nations}} of the {{Americas Chapter}} of the {{Association}} for {{Computational Linguistics}}: {{Human Language Technologies}}},
  author = {Li, Siyan and Raghuram, Vethavikashini Chithrra and Khattab, Omar and Hirschberg, Julia and Yu, Zhou},
  year = {2025},
  pages = {3371--3390},
  publisher = {Association for Computational Linguistics},
  doi = {10.18653/v1/2025.naacl-long.173},
  url = {https://aclanthology.org/2025.naacl-long.173/}
}

@misc{liPrefixTuningOptimizingContinuous2021,
  title = {Prefix-{{Tuning}}: {{Optimizing Continuous Prompts}} for {{Generation}}},
  shorttitle = {Prefix-{{Tuning}}},
  author = {Li, Xiang Lisa and Liang, Percy},
  date = {2021-01-01},
  year = {2021},
  eprint = {2101.00190},
  eprinttype = {arXiv},
  eprintclass = {cs},
  url = {http://arxiv.org/abs/2101.00190},
  urldate = {2023-07-30},
  pubstate = {prepublished}
}

@misc{liuGPTUnderstandsToo2021,
  title = {{{GPT Understands}}, {{Too}}},
  author = {Liu, Xiao and Zheng, Yanan and Du, Zhengxiao and Ding, Ming and Qian, Yujie and Yang, Zhilin and Tang, Jie},
  date = {2021-03-18},
  year = {2021},
  eprint = {2103.10385},
  eprinttype = {arXiv},
  eprintclass = {cs},
  url = {http://arxiv.org/abs/2103.10385},
  urldate = {2023-07-30},
  pubstate = {prepublished}
}

@article{pryzantAutomaticPromptOptimization2023,
  title = {Automatic {{Prompt Optimization}} with "{{Gradient Descent}}" and {{Beam Search}}},
  author = {Pryzant, Reid and Iter, Dan and Li, Jerry and Lee, Yin Tat and Zhu, Chenguang and Zeng, Michael},
  year = {2023},
  eprint = {2305.03495},
  eprinttype = {arXiv},
  url = {http://arxiv.org/abs/2305.03495}
}

@article{pyatkinGeneralizingVerifiableInstruction2025,
  title = {Generalizing {{Verifiable Instruction Following}}},
  author = {Pyatkin, Valentina and Malik, Saumya and Graf, Victoria and Ivison, Hamish and Huang, Shengyi and Dasigi, Pradeep and Lambert, Nathan and Hajishirzi, Hannaneh},
  year = {2025},
  eprint = {2507.02833},
  eprinttype = {arXiv},
  doi = {10.48550/arXiv.2507.02833},
  url = {http://arxiv.org/abs/2507.02833}
}

@misc{qianControllableNaturalLanguage2022,
  title = {Controllable {{Natural Language Generation}} with {{Contrastive Prefixes}}},
  author = {Qian, Jing and Dong, Li and Shen, Yelong and Wei, Furu and Chen, Weizhu},
  date = {2022-02-26},
  year = {2022},
  eprint = {2202.13257},
  eprinttype = {arXiv},
  eprintclass = {cs},
  url = {http://arxiv.org/abs/2202.13257},
  urldate = {2023-07-30},
  pubstate = {prepublished}
}

@inproceedings{tesslerRewardConstrainedPolicy2018,
  title = {Reward {{Constrained Policy Optimization}}},
  author = {Tessler, Chen and Mankowitz, Daniel J. and Mannor, Shie},
  booktitle = {International Conference on Learning Representations (ICLR)},
  year = {2019},
  eprint = {1805.11074},
  eprinttype = {arXiv},
  url = {https://arxiv.org/abs/1805.11074}
}

@inproceedings{yangOPROLargeLanguage2024,
  title = {Large Language Models as Optimizers},
  author = {Yang, Chengrun and Wang, Xuezhi and Lu, Yifeng and Liu, Hanxiao and Le, Quoc V. and Zhou, Denny and Chen, Xinyun},
  booktitle = {The Twelfth International Conference on Learning Representations (ICLR)},
  year = {2024},
  eprint = {2309.03409},
  eprinttype = {arXiv},
  url = {https://arxiv.org/abs/2309.03409}
}

@misc{zhangAgenticContextEngineering2025,
  title = {Agentic {{Context Engineering}}: {{Evolving Contexts}} for {{Self-Improving Language Models}}},
  shorttitle = {Agentic {{Context Engineering}}},
  author = {Zhang, Qizheng and Hu, Changran and Upasani, Shubhangi and Ma, Boyuan and Hong, Fenglu and Kamanuru, Vamsidhar and Rainton, Jay and Wu, Chen and Ji, Mengmeng and Li, Hanchen and Thakker, Urmish and Zou, James and Olukotun, Kunle},
  date = {2025-10-06},
  year = {2025},
  eprint = {2510.04618},
  eprinttype = {arXiv},
  eprintclass = {cs},
  doi = {10.48550/arXiv.2510.04618},
  url = {http://arxiv.org/abs/2510.04618},
  urldate = {2025-12-30},
  pubstate = {prepublished}
}

@misc{alzubi2026evoskillautomatedskilldiscovery,
  title = {{EvoSkill}: Automated Skill Discovery for Multi-Agent Systems},
  author = {Salaheddin Alzubi and Noah Provenzano and Jaydon Bingham and Weiyuan Chen and Tu Vu},
  year = {2026},
  eprint = {2603.02766},
  archivePrefix = {arXiv},
  primaryClass = {cs.AI},
  url = {https://arxiv.org/abs/2603.02766}
}

@misc{wang2026skillgradoptimizingagentskills,
  title = {{SkillGrad}: Optimizing Agent Skills Like Gradient Descent},
  author = {Hanyu Wang and Yifan Lan and Bochuan Cao and Lu Lin and Jinghui Chen},
  year = {2026},
  eprint = {2605.27760},
  archivePrefix = {arXiv},
  primaryClass = {cs.AI},
  url = {https://arxiv.org/abs/2605.27760}
}

@misc{zhou2026agenticpolicyoptimizationinstructionpolicy,
  title = {Agentic Policy Optimization via Instruction-Policy Co-Evolution},
  author = {Han Zhou and Xingchen Wan and Ivan Vuli{\'c} and Anna Korhonen},
  year = {2026},
  eprint = {2512.01945},
  archivePrefix = {arXiv},
  primaryClass = {cs.LG},
  url = {https://arxiv.org/abs/2512.01945}
}

@misc{xia2026skillrlevolvingagentsrecursive,
  title = {{SkillRL}: Evolving Agents via Recursive Skill-Augmented Reinforcement Learning},
  author = {Peng Xia and Jianwen Chen and Hanyang Wang and Jiaqi Liu and Kaide Zeng and Yu Wang and Siwei Han and Yiyang Zhou and Xujiang Zhao and Haifeng Chen and Zeyu Zheng and Cihang Xie and Huaxiu Yao},
  year = {2026},
  eprint = {2602.08234},
  archivePrefix = {arXiv},
  primaryClass = {cs.LG},
  url = {https://arxiv.org/abs/2602.08234}
}

@misc{wang2026reinforcementlearningselfimprovingagent,
  title = {Reinforcement Learning for Self-Improving Agent with Skill Library},
  author = {Jiongxiao Wang and Qiaojing Yan and Yawei Wang and Yijun Tian and Soumya Smruti Mishra and Zhichao Xu and Megha Gandhi and Panpan Xu and Lin Lee Cheong},
  year = {2026},
  eprint = {2512.17102},
  archivePrefix = {arXiv},
  primaryClass = {cs.AI},
  url = {https://arxiv.org/abs/2512.17102}
}

@misc{he2026reskillreconcilingskillcreation,
  title = {{ReSkill}: Reconciling Skill Creation with Policy Optimization in Agentic RL},
  author = {Zelin He and Haotian Lin and Boran Han and Wei Zhu and Haoyang Fang and Bernie Wang and Xuan Zhu and Runze Li and Matthew Reimherr},
  year = {2026},
  eprint = {2606.01619},
  archivePrefix = {arXiv},
  primaryClass = {cs.AI},
  url = {https://arxiv.org/abs/2606.01619}
}

@article{zouUniversalTransferableAdversarial2023,
  title = {Universal and {{Transferable Adversarial Attacks}} on {{Aligned Language Models}}},
  author = {Zou, Andy and Wang, Zifan and Carlini, Nicholas and Nasr, Milad and Kolter, J. Zico and Fredrikson, Matt},
  year = {2023},
  eprint = {2307.15043},
  eprinttype = {arXiv},
  eprintclass = {cs},
  doi = {10.48550/arXiv.2307.15043},
  url = {http://arxiv.org/abs/2307.15043}
}

@article{Barres_Dong_Ray_Si_Narasimhan_2025,
  title = {$\tau^2$-Bench: Evaluating Conversational Agents in a Dual-Control Environment},
  author = {Barres, Victor and Dong, Honghua and Ray, Soham and Si, Xujie and Narasimhan, Karthik},
  year = {2025},
  eprint = {2506.07982},
  eprinttype = {arXiv},
  url = {http://arxiv.org/abs/2506.07982}
}

@article{cobbeGSM8K2021,
  title = {Training Verifiers to Solve Math Word Problems},
  author = {Cobbe, Karl and Kosaraju, Vineet and Bavarian, Mohammad and Chen, Mark and Jun, Heewoo and Kaiser, Lukasz and Plappert, Matthias and Tworek, Jerry and Hilton, Jacob and Nakano, Reiichiro and Hesse, Christopher and Schulman, John},
  year = {2021},
  eprint = {2110.14168},
  eprinttype = {arXiv},
  url = {https://arxiv.org/abs/2110.14168}
}

@inproceedings{resendiz2025mopo,
  title={MOPO: Multi-objective prompt optimization for affective text generation},
  author={Resendiz, Yarik Menchaca and Klinger, Roman},
  booktitle={Proceedings of the 31st International Conference on Computational Linguistics},
  pages={5588--5606},
  year={2025}
}

@inproceedings{
    jimenez2024swebench,
    title={{SWE}-bench: Can Language Models Resolve Real-world Github Issues?},
    author={Carlos E Jimenez and John Yang and Alexander Wettig and Shunyu Yao and Kexin Pei and Ofir Press and Karthik R Narasimhan},
    booktitle={The Twelfth International Conference on Learning Representations},
    year={2024},
    url={https://openreview.net/forum?id=VTF8yNQM66}
}

@article{liu2026ministral,
  title = {Ministral 3},
  author = {Liu, Alexander H. and Khandelwal, Kartik and Subramanian, Sandeep and Jouault, Victor and Rastogi, Abhinav and Sad{\'e}, Adrien and Jeffares, Alan and Jiang, Albert and Cahill, Alexandre and Gavaudan, Alexandre and others},
  journal = {arXiv preprint arXiv:2601.08584},
  year = {2026},
  eprint = {2601.08584},
  eprinttype = {arXiv},
  doi = {10.48550/arXiv.2601.08584},
  url = {https://arxiv.org/abs/2601.08584}
}

@misc{qwen3technicalreport,
      title={Qwen3 Technical Report},
      author={Yang, An and Li, Anfeng and Yang, Baosong and others},
      year={2025},
      eprint={2505.09388},
      archivePrefix={arXiv},
      primaryClass={cs.CL},
      url={https://arxiv.org/abs/2505.09388},
}

@inproceedings{guoConnectingLargeLanguage2023,
  title     = {EvoPrompt: Connecting Large Language Models with Evolutionary Algorithms Yields Powerful Prompt Optimizers},
  author    = {Guo, Qingyan and Wang, Rui and Guo, Junliang and Li, Bei and Song, Kaitao and Tan, Xu and Liu, Guoqing and Bian, Jiang and Yang, Yujiu},
  booktitle = {The Twelfth International Conference on Learning Representations (ICLR)},
  year      = {2024},
  eprint    = {2309.08532},
  archivePrefix = {arXiv},
  primaryClass  = {cs.CL},
  url       = {https://arxiv.org/abs/2309.08532}
}

@inproceedings{jangCategoricalReparameterizationGumbel2017,
  title     = {Categorical Reparameterization with Gumbel-Softmax},
  author    = {Jang, Eric and Gu, Shixiang and Poole, Ben},
  booktitle = {International Conference on Learning Representations (ICLR)},
  year      = {2017},
  eprint    = {1611.01144},
  archivePrefix = {arXiv},
  primaryClass  = {stat.ML},
  url       = {https://arxiv.org/abs/1611.01144}
}

@inproceedings{maddisonConcreteDistributionContinuous2017,
  title     = {The Concrete Distribution: A Continuous Relaxation of Discrete Random Variables},
  author    = {Maddison, Chris J. and Mnih, Andriy and Teh, Yee Whye},
  booktitle = {International Conference on Learning Representations (ICLR)},
  year      = {2017},
  eprint    = {1611.00712},
  archivePrefix = {arXiv},
  primaryClass  = {stat.ML},
  url       = {https://arxiv.org/abs/1611.00712}
}

@book{kushnerYinStochasticApproximation2003,
  title     = {Stochastic Approximation and Recursive Algorithms and Applications},
  author    = {Kushner, Harold J. and Yin, G. George},
  edition   = {2nd},
  publisher = {Springer},
  series    = {Stochastic Modelling and Applied Probability},
  year      = {2003}
}

@article{nedicOzdaglarApproximatePrimal2009,
  title     = {Approximate Primal Solutions and Rate Analysis for Dual Subgradient Methods},
  volume    = {19},
  issn      = {1095-7189},
  url       = {http://dx.doi.org/10.1137/070708111},
  doi       = {10.1137/070708111},
  number    = {4},
  journal   = {SIAM Journal on Optimization},
  publisher = {Society for Industrial \& Applied Mathematics (SIAM)},
  author    = {Nedi{\'c}, Angelia and Ozdaglar, Asuman},
  year      = {2009},
  month     = jan,
  pages     = {1757--1780}
}

@misc{shao2024deepseekmathpushinglimitsmathematical,
  title         = {DeepSeekMath: Pushing the Limits of Mathematical Reasoning in Open Language Models},
  author        = {Zhihong Shao and Peiyi Wang and Qihao Zhu and Runxin Xu and Junxiao Song and Xiao Bi and Haowei Zhang and Mingchuan Zhang and Y. K. Li and Y. Wu and Daya Guo},
  year          = {2024},
  eprint        = {2402.03300},
  archivePrefix = {arXiv},
  primaryClass  = {cs.CL},
  doi           = {10.48550/arXiv.2402.03300},
  url           = {https://arxiv.org/abs/2402.03300}
}

@misc{fu2025areal,
  title         = {{AReaL}: A Large-Scale Asynchronous Reinforcement Learning System for Language Reasoning},
  author        = {Wei Fu and Jiaxuan Gao and Xujie Shen and Chen Zhu and Zhiyu Mei and Chuyi He and Shusheng Xu and Guo Wei and Jun Mei and Jiashu Wang and Tongkai Yang and Binhang Yuan and Yi Wu},
  year          = {2025},
  eprint        = {2505.24298},
  archivePrefix = {arXiv},
  primaryClass  = {cs.LG},
  url           = {https://arxiv.org/abs/2505.24298}
}
